\documentclass[11pt]{article}
\usepackage{lmodern}

\usepackage[utf8]{inputenc} 
\usepackage[T1]{fontenc}    
\usepackage{url}            
\usepackage{booktabs}       
\usepackage{nicefrac}       
\usepackage{microtype}      
\usepackage{xcolor}         
\usepackage{amsmath, amsthm} 
\usepackage{amsfonts,amssymb}
\usepackage{stmaryrd}

\usepackage{nicefrac}       

\usepackage{multirow}

\usepackage{thmtools, thm-restate}
\usepackage[usenames,dvipsnames]{pstricks}

\usepackage{graphicx}

\usepackage{booktabs}
\usepackage{algorithm,algpseudocode}

\usepackage{enumitem}
\usepackage{xcolor}

\definecolor{dgreen}{rgb}{0.00,0.49,0.00}
\definecolor{dblue}{rgb}{0,0.08,0.75}
\definecolor{dred}{rgb}{0.75,0.0,0.0}
\usepackage{hyperref}
\hypersetup{
    colorlinks,
    citecolor=dgreen,
    urlcolor=dblue,
    linkcolor=dblue
}

\usepackage[font=footnotesize]{caption}
\usepackage[nameinlink,capitalize]{cleveref}
\usepackage{pifont}
\newcommand{\xmark}{\ding{55}}

\usepackage{algorithm, algorithmicx}

\usepackage{url}            

\let\mathsf\relax    
\DeclareRobustCommand{\mathsf}[1]{\text{\normalfont\sffamily#1}}

\newcommand{\X}{{\mathcal X}}
\newcommand{\Y}{{\mathcal Y}}

\newcommand{\EE}{\mathbb{E}}

\renewcommand{\H}{\mathcal{H}}

\newcommand{\hh}{{\H}}

\newcommand{\R}{\mathbb{R}}

\newcommand{\N}{\mathbb{N}}

\newcommand{\la}{\lambda}

\newcommand{\eqals}[1]{\begin{align*}#1\end{align*}}

\newcommand{\eqal}[1]{\begin{align}#1\end{align}}

\renewcommand{\eqals}[1]{\eqal{#1}}

\newcommand{\msf}[1]{\mathsf{#1}}
\newcommand{\mbf}[1]{\mathbf{#1}}

\newcommand{\diag}{\ensuremath{\text{\rm diag}}}

\newcommand{\tr}{\ensuremath{\text{\rm Tr}}}

\renewcommand{\vec}{\ensuremath{\text{\rm vec}}}

\newcommand{\argmin}{\operatornamewithlimits{argmin}}

\newcommand{\vol}{\operatorname{vol}}

\renewcommand{\paragraph}[1]{\vspace{1em}\noindent{\bfseries #1}.}

\declaretheorem[name=Theorem,refname=Thm.]{theorem}
\declaretheorem[name=Lemma,sibling=theorem]{lemma}
\declaretheorem[name=Proposition,refname=Prop.,sibling=theorem]{proposition}
\declaretheorem[name=Remark]{remark}
\declaretheorem[name=Corollary,refname=Cor.,sibling=theorem]{corollary}
\declaretheorem[name=Definition,refname=Def.]{definition}

\declaretheorem[name=Assumption,refname=Asm.]{assumption}

\declaretheorem[name=Example]{example}

\crefname{assumption}{Assumption}{Assumptions}
\crefname{equation}{}{}
\crefname{figure}{Fig.}{Fig.}
\crefname{table}{Table}{Tables}
\crefname{section}{Sec.}{Sec.}
\crefname{theorem}{Thm.}{Thm.}
\crefname{lemma}{Lemma}{Lemmas}
\crefname{corollary}{Cor.}{Cor.}
\crefname{example}{Example}{Examples}
\crefname{remark}{Remark}{Remarks}
\crefname{algorithm}{Alg.}{Algorightms}

\crefname{appendix}{Appendix}{Appendices}
\crefname{subappendix}{Appendix}{Appendices}
\crefname{subsubappendix}{Appendix}{Appendices}

\newcommand{\psd}{\mathbb{S}_+}

\newcommand{\eps}{\varepsilon}

\newcommand{\codecomment}{\mathbin{/\mkern-4mu/}}

\newcommand{\mm}{\msf{M}}

\newcommand{\pp}[2]{f({#1}\,;\, {#2})}

\providecommand{\scal}[2]{\left\langle{#1},{#2}\right\rangle}

\providecommand{\nor}[1]{\left\|{#1}\right\|}

\newcommand{\polylog}{\operatorname{polylog}}

\newcommand{\bd}{\begin{definition}}
\newcommand{\ed}{\end{definition}}
\newlist{enumdef}{enumerate}{1} 
\setlist[enumdef]{label=\upshape(\alph*),ref=\upshape\thedefinition(\alph*)}
\crefalias{enumdefi}{theorem} 

\newcommand{\bi}{\begin{itemize}}
\newcommand{\ei}{\end{itemize}}

\newcommand{\ba}{\begin{assumption}}
\newcommand{\ea}{\end{asssumption}}
\newlist{enumasm}{enumerate}{1} 
\setlist[enumasm]{label=\upshape(\alph*),ref=\upshape\theass(\alph*)}
\crefalias{enumasmi}{ass} 

\newcommand{\br}{\begin{remark}}
\newcommand{\er}{\end{remark}}
\newlist{enumrem}{enumerate}{1} 
\setlist[enumrem]{label=\upshape(\alph*),ref=\upshape\theremark(\alph*)}
\crefalias{enumremi}{theorem} 

\newcommand{\bp}{\begin{proposition}}
\newcommand{\ep}{\end{proposition}}
\newlist{enumprop}{enumerate}{1} 
\setlist[enumprop]{label=\upshape(\alph*),ref=\upshape\theproposition(\alph*)}
\crefalias{enumpropi}{proposition} 

\newcommand{\blm}{\begin{lemma}}
\newcommand{\elm}{\end{lemma}}
\newlist{enumlm}{enumerate}{1} 
\setlist[enumlm]{label=\upshape(\alph*),ref=\upshape\thelemma(\alph*)}
\crefalias{enumlmi}{lemma} 

\newcommand{\bt}{\begin{theorem}}
\newcommand{\et}{\end{theorem}}
\newlist{enumthm}{enumerate}{1} 
\setlist[enumthm]{label=\upshape(\alph*),ref=\upshape\thetheorem(\alph*)}
\crefalias{enumthmi}{theorem} 

\newcommand{\bcor}{\begin{corollary}}
\newcommand{\ecor}{\end{corollary}}
\newlist{enumcor}{enumerate}{1} 
\setlist[enumcor]{label=\upshape(\alph*),ref=\upshape\thecorollary(\alph*)}
\crefalias{enumcori}{corollary} 

\newcommand{\bex}{\begin{example}}
\newcommand{\eex}{\end{example}}
\newlist{enumex}{enumerate}{1} 
\setlist[enumex]{label=\upshape(\alph*),ref=\upshape\theexample(\alph*)}
\crefalias{enumexi}{theorem} 

\crefname{page}{page}{page}

\usepackage{etoolbox}

\AfterEndEnvironment{restatable}{\noindent\ignorespacesafterend}
\AfterEndEnvironment{theorem}{\noindent\ignorespacesafterend}
\AfterEndEnvironment{remark}{\noindent\ignorespacesafterend}
\AfterEndEnvironment{example}{\noindent\ignorespacesafterend}
\AfterEndEnvironment{assumption}{\noindent\ignorespacesafterend}
\AfterEndEnvironment{lemma}{\noindent\ignorespacesafterend}
\AfterEndEnvironment{proof}{\noindent\ignorespacesafterend}
\AfterEndEnvironment{corollary}{\noindent\ignorespacesafterend}
\AfterEndEnvironment{proposition}{\noindent\ignorespacesafterend}
\AfterEndEnvironment{definition}{\noindent\ignorespacesafterend}

\newcommand{\citep}{\cite}

\title{PSD Representations for Effective Probability Models}

\author{%
  Alessandro Rudi${}^\star$ $\qquad$ Carlo Ciliberto${}^{\dagger}$\\
  $ $ \\
  ${}^{\star}$ Inria, École normale supérieure, CNRS, PSL Research University, Paris, France\\
  ${}^{\dagger}$ Department of Computer Science, University College London, London, UK \\
  \texttt{alessandro.rudi@inria.fr} $\qquad$ \texttt{c.ciliberto@ucl.ac.uk} \\
}

\date{}
\begin{document}

\maketitle

\begin{abstract}
\noindent Finding a good way to model probability densities is key to probabilistic inference. An ideal model should be able to concisely approximate any probability while being also compatible with two main operations: multiplications of two models (product rule) and marginalization with respect to a subset of the random variables (sum rule). In this work, we show that a recently proposed class of positive semi-definite (PSD) models for non-negative functions is particularly suited to this end. In particular, we characterize both approximation and generalization capabilities of PSD models, showing that they enjoy strong theoretical guarantees. Moreover, we show that we can perform efficiently both sum and product rule in closed form via matrix operations, enjoying the same versatility of mixture models. Our results open the way to applications of PSD models to density estimation, decision theory and inference.
\end{abstract}

\section{Introduction}
\label{sec:intro}

Modeling probability distributions is a key task for many applications in machine learning \cite{murphy2012machine,bishop2006}. To this end, several strategies have been proposed in the literature, such as adopting mixture models (e.g. Gaussian mixtures) \cite{bishop2006}, exponential models \cite{sriperumbudur2017density}, implicit generative models \cite{goodfellow2014generative,kingma2013auto} or kernel (conditional) mean embeddings \cite{muandet2016kernel}. An ideal probabilistic model should have two main features: $i)$ efficiently perform key operations for probabilistic inference, such as sum rule (i.e. marginalization) and product rule \cite{bishop2006} and, $ii)$ concisely approximate a large class of probabilities. Finding models that satisfy these two conditions is challenging and current methods tend to tackle only one of the two. Exponential and implicit generative models have typically strong approximation properties
(see e.g. \cite{sriperumbudur2017density,singh2018nonparametric}) but cannot easily perform operations such as marginalization. On the contrary, mixture models are designed to efficient integrate and multiply probabilities, but tend to require a large number of components to approximate complicated distributions.

In principle, mixture models would offer an appealing strategy to model probability densities since they allow for efficient computations when performing key operations such as sum and product rule. However, these advantages in terms of computations, come as a disadvantage in terms of expressiveness: even though mixture models are universal approximators (namely they can approximate arbitrarily well any probability density), they require a significant number $n$ of observations and of components to do that. Indeed it is known that models that are non-negative mixtures of non-negative components lead to learning rates that suffer from the curse of dimensionality. For example, when approximating probabilities on $\R^d$, kernel density estimation (KDE) \cite{parzen1962estimation} with non-negative components has rates as slow as $n^{-2/(4+d)}$ (see, e.g. \cite[page 100]{wand1994kernel}), and cannot improve with the regularity (smoothness) of the target probability (see e.g. \cite[Thm. 16.1]{devroye2012combinatorial} for impossibility results in the one dimensional case). 

In the past decades, this limitation has been overcome by removing either the non-negativity of the weights of the mixture (leading to RBF networks \cite{shawe2000introduction}\cite{tsybakov2008introduction}), or the non-negativity of the components used in the mixture (leading to KDE with oscillating kernel \cite{tsybakov2008introduction}), or both. On the positive side this allows to achieve a learning rate of $n^{-\beta/(2\beta + d)}$, for $\beta$-times differentiable densities on $\R^d$. Note that such rate is {\em minimax optimal} \cite{goldenshluger2014adaptive} and overcomes the curse of dimensionality when $\beta \geq d$. Additionally the resulting model is very {\em concise}, since only $m = O(n^{d/\beta})$ centers are necessary to achieve the optimal rate \cite{turner2021statistical}. However, on the negative side, {\em the resulting model is not a probability}, since it may attain negative values, and so it cannot be used where a proper probability is required.

In this paper, we show that the positive semidefinite (PSD) models, recently proposed in \cite{marteau2020non}, offer a way out of this dichotomy. By construction PSD models generalize mixture models by allowing also for negative weights, while still guaranteeing that the resulting function is non-negative everywhere. Here, we prove that they get the best of both worlds: expressivity (optimal learning and approximation rates with concise models) and flexibility (exact and efficient sum and product rule).

\begin{table}[t]
\resizebox{\textwidth}{!}{  
    \centering
    \begin{tabular}{rcccccc}
          & \multirow{2}{*}{\bfseries Non-negative} &  {\bfseries Sum} & {\bfseries Product} & {\bfseries Concise} & {\bfseries Optimal} & {\bfseries Efficient} \\
          & & {\bfseries Rule} & {\bfseries Rule} & {\bfseries Approximation} & {\bfseries Learning} & {\bfseries Sampling} \\
         \midrule\\[-0.5em]
          {\bfseries Linear Models} & {\color{dred}\xmark} &  {\color{dgreen}\checkmark} & {\color{dgreen}\checkmark} & {\color{dgreen}\checkmark} & {\color{dgreen}\checkmark} & {\color{dred}\xmark} \\[0.5em]
          {\bfseries Mean Embeddings}& {\color{dred}\xmark} &  {\color{dgreen}\checkmark} & {\color{dgreen}\checkmark} & {\color{dgreen}\checkmark} & {\color{dgreen}\checkmark} & {\color{dred}\xmark}  \\[0.5em]
         {\bfseries Mixture Models} & {\color{dgreen}\checkmark} &  {\color{dgreen}\checkmark} & {\color{dgreen}\checkmark} & {\color{dred}\xmark} & {\color{dred}\xmark} & {\color{dgreen}\checkmark} \\[0.5em]
         {\bfseries Exponential Models} & {\color{dgreen}\checkmark} &  {\color{dred}\xmark} & {\color{dgreen}\checkmark} & {\color{dgreen}\checkmark} & {\color{dgreen}\checkmark} & {\color{dred}\xmark} \\[0.5em]
{\bfseries PSD Models}& {\color{dgreen}\checkmark} &  {\color{dgreen}\checkmark} & {\color{dgreen}\checkmark} & {\color{dgreen}\checkmark} & {\color{dgreen}\checkmark} & {\color{dgreen}\checkmark}\\[-0.1em]
& (see \cite{marteau2020non}) & (\cref{prop:marginalization}) &  (\cref{prop:multiplication}) & (\cref{thm:approximation}) & (\cref{thm:learning}) & (see \cite{marteau2021sampling}) 
\\[1em]
    \end{tabular}
}

    \caption{Summary of the main desirable properties for a probability model.}
    \label{tab:summary}
\end{table}

\paragraph{Contributions} 
The main contributions of this paper are:
\begin{enumerate}[label={\em \roman*)}]
\item Showing that PSD models can perform exact sum and product rules in terms of efficient matrix operations (\cref{sec:operations}).
\item Characterizing the approximation and learning properties of PSD models with respect to a large family of probability densities (\cref{sec:approximation}).
\item Providing a ``compression'' method to control the number of components of a PSD model (introduced in \cref{sec:compression} and analyzed in \cref{sec:effect-compression}). 
\item Discussing a number of possible applications of PSD models (\cref{sec:applications}).
\end{enumerate}

\paragraph{Summary of the results} \cref{tab:summary} summarizes all the desirable properties that a probability model should have and which of them are satisfied by state-of-the-art estimators. Among these we have: {\em non-negativity}, the model should be point-wise non-negative; {\em sum and product rule}, the model should allow for efficient computation of such operations between probabilities; {\em concise approximation}, a model with the optimal number of components $m = O(\varepsilon^{-d/\beta})$ is enough to approximate with error $\varepsilon$ a non-negative function in a wide family of smooth $\beta$-times differentiable functions \cite{turner2021statistical}; {\em optimal learning}, the model achieves optimal learning rates of order $n^{-\beta/(2\beta + d)}$  when learning the unknown target probability from i.i.d. samples \cite{tsybakov2008introduction}; {\em efficient sampling} the model allows to extract i.i.d. samples without incurring in the curse of dimensionality in terms of computational complexity. We note that for PSD models, this last property has been recently investigated in \cite{marteau2021sampling}. We consider under the umbrella of {\em linear models} all the models which corresponds to a mixture $\sum_{j=1}^m w_j f_j(x)$ of functions $f_j:\X\to\R$, with either $w_j \geq0, ~ \forall j=1,\dots,m$ or $f_j\geq0, ~ \forall j=1,\dots,m$ or both. We denote by {\em mixture models} the mixture models defined as above, where $w_j \geq 0, f_j \geq 0, ~ \forall j=1,\dots,m$. By {\em mean embeddings}, we denote the kernel mean embedding estimator framework from \cite{smola2007hilbert} see also \cite{muandet2016kernel}). With {\em exponential models} we refer to models of the form $\exp(\sum_{j=1}^m w_j f_j(x))$ with $w_j \in \R, f_j:X \to \R$ \cite{sriperumbudur2017density}. With {\em PSD models} we refer to the framework introduced in \cite{marteau2020non} and studied in this paper. We note that this work contributes in showing that the latter are the only probability models to satisfy all requirements (we recall that non-negativity and efficient sampling have been shown respectively in \cite{marteau2020non} and \cite{marteau2021sampling}).

\paragraph{Notation}
We denote by $\R^d_{++}$ the space vectors in $\R^d$ with positive entries,  $\R^{n \times d}$ the space of $n \times d$ matrices, $\psd^n=\psd(\R^n)$ the space of positive semidefinite $n \times n$ matrices. Given a vector $\eta\in\R^d$, we denote $\diag(\eta)\in\R^{d \times d}$ the diagonal matrix associated to $\eta$. We denote by $A \circ B$ and $A\otimes B$ respectively the entry-wise and Kronecker product between two matrices $A$ and $B$. We denote by $\|A\|_{F}, \|A\|, \det(A), \vec(A)$ and $A^\top$ respectively the Frobenius norm, the operator norm (i.e.  maximum singular value), the determinant, the (column-wise) vectorization of a matrix and the (conjugate) transpose of $A$. With some abuse of notation, where clear from context we write element-wise products and division of vectors $u,v\in\R^{d}$ as $uv, u/v$. Given two matrices $X\in\R^{n\times d_1}, Y\in\R^{n \times d_2}$ with same number of rows, we denote by $[X,Y]\in\R^{n \times (d_1 + d_2)}$ their concatenation row-wise. The term $\mathbf{1}_n\in\R^n$ denotes the vector with all entries equal to $1$.

\section{PSD Models}
Following \cite{marteau2020non}, in this work we consider the family of positive semi-definite (PSD) models, namely non-negative functions parametrized by a feature map $\phi:\X\to\hh$ from an input space $\X$ to a suitable feature space $\hh$ (a separable Hilbert space e.g. $\R^q$) and a linear operator $\mm\in\psd(\hh)$, of the form 
\eqal{\label{eq:psd-models-general}
    \pp{x}{\mm,\phi} = \phi(x)^\top \mm~\phi(x).
}
PSD models offer a general way to parametrize non-negative functions (since $\mm$ is positive semidefinite, $\pp{x}{\mm,\phi}\geq0$ for any $x\in\X$) and enjoy several additional appealing properties discussed in the following. 
%
%
In this work, we will focus on a special family of models of the form \cref{eq:psd-models-general} to parametrize probability densities over $\X = \R^d$. In particular, we will consider the case where: $i)$ $\phi = \phi_\eta:\R^d\to\hh_\eta$ is a feature map associated to the Gaussian kernel \cite{scholkopf2002} $k_\eta(x,x') = \phi_\eta(x)^\top\phi_\eta(x) = e^{-(x-x')^\top\diag(\eta)(x-x')}$, with $\eta\in\R_{++}^d$ and, $ii)$ the operator $\mm$ lives in the span of $\phi(x_1),\dots, \phi(x_n)$ for a given set of points $(x_i)_{i=1}^n$, namely there exists $A\in\psd^n$ such that $\mm = \sum_{ij} A_{ij} \phi(x_i) \phi(x_j)^\top$. We define a {\em Gaussian PSD model} by specializing the definition in \cref{eq:psd-models-general} as
\eqal{\label{eq:psd-models-gaussian}
    \pp{x}{A, X, \eta} = \sum_{i,j=1}^n A_{i,j} k_\eta(x_i,x)k_\eta(x_j,x), \qquad \forall \, x \in \R^d
}
in terms of the coefficient matrix $A\in\psd^n$, the base points matrix $X\in\R^{n\times d}$, whose $i$-th row corresponds to the point $x_i$ for each $i=1,\dots,n$ and kernel parameter $\eta$. In the following, given two base point matrices $X\in\R^{n\times d}$ and $X'\in\R^{m \times d}$, we denote by $K_{X,X',\eta} \in \R^{n \times m}$ the kernel matrix with entries $(K_{X,X',\eta})_{ij} = k_{\eta}(x_i,x_j')$ where $x_i, x'_j$ are the $i$-th and $j$-th rows of $X, X'$ respectively. When clear from context, in the following we will refer to Gaussian PSD models as PSD models.

\begin{remark}[PSD models generalize Mixture models]\label{rem:psd-models-generalized-mixtures}
Mixture models (a mixture of Gaussian distributions) are a special case of PSD models. Let $A = \diag(a)$ be a diagonal matrix of $n$ positive weights $a\in\R_{++}^n$. We have $\pp{x}{A,X,\eta/2} = \sum_{i=1}^n A_{ii} k_{\eta/2}(x_i,x)^2 = \sum_{i=1}^n a_i k_{\eta}(x_i,x)$.
\end{remark}

\begin{remark}[PSD models allow negative weights]\label{rem:negative-weights}
From \cref{eq:psd-models-gaussian}, we immediately see that PSD models generalize mixture models by allowing also for negative weights: e.g., $\pp{\cdot}{A, X, \eta}$ with $A = (1, -\frac{1}{2}; -\frac{1}{2}, \frac{1}{4}) \in \psd^2, \eta = 1, X = (x_1;x_2)$ and $x_1 = 0, x_2 = 2$, corresponding to $\pp{x}{A, X, \eta} = e^{-2x^2} + \frac{1}{4}e^{-2(x-2)^2} - \frac{1}{e}e^{-2(x-1)^2}$, i.e. a mixture of Gaussians with also negative weights.
\end{remark}

\subsection{Operations with PSD models}
\label{sec:operations}
In \cref{sec:approximation} we will show that PSD models can approximate a wide class of probability densities, significantly outperforming mixture models. Here we show that this improvement does not come at the expenses of computations. In particular, we show that PSD models enjoy the same flexibility of mixture models: $i)$ they are closed with respect to key operations such as marginalization and multiplication and $ii)$ these operations can be performed efficiently in terms of matrix sums/products. The derivation of the results reported in the following is provided in \cref{app:operations}. They follow from well-known properties of the Gaussian function. 

\paragraph{Evaluation}
Evaluating a PSD model in a point $x_0\in\X$ corresponds to $\pp{x=x_0}{A,X,\eta} = K_{X,x_0,\eta}^\top A K_{X,x_0,\eta}$. Moreover, partially evaluating a PSD in one variable yields
\eqal{\label{eq:partial-evaluation}
    \pp{x,y=y_0}{A,[X,Y],(\eta_1,\eta_2)} ~=~ \pp{x}{B,X,\eta_1} \quad \textrm{with}\quad B = A \circ (K_{Y,y_0,\eta_2}K_{Y,y_0,\eta_2}^\top).
}
Note that $\pp{x}{B,X,\eta_1}$ is still a PSD model since $B$ is positive semidefinite.

\paragraph{Sum Rule (Marginalization and Integration)} 
The integral of a PSD model can be computed as
\eqal{\label{eq:integration}
\int \pp{x}{A,X,\eta}~ dx = c_{2\eta}~ \tr(A ~K_{X,X,\frac{\eta}{2}}) \qquad \textrm{with} \qquad c_\eta = \int k_{\eta}(0,x)~dx,
}
where $c_\eta = \pi^{d/2}\det(\diag(\eta))^{-1/2}$. This is particularly useful to model probabiliy densities with PSD models. Let $Z = \int \pp{x}{A,X,\eta} dx$, then the function $\pp{x}{A/Z,X,\eta} = \frac{1}{Z}\pp{x}{A,X,\eta}$ is a probability density. Integrating only one variable of a PSD model we obtain the sum rule.
\begin{restatable}[Sum Rule -- Marginalization]{proposition}{PMarginalization}\label{prop:marginalization}
Let $X\in\R^{n\times d}$, $Y\in\R^{n\times d'}$, $A\in\psd(\R^n)$ and $\eta\in\R_{++}^{d}, \eta'\in\R_{++}^{d'}$. Then, the following integral is a PSD model
\eqal{\label{eq:marginalization}
     \int ~\pp{x,y}{A,[X,Y],(\eta,\eta')} ~dx  = \pp{y}{B,Y,\eta'}, \qquad \textrm{with} \qquad B = c_{2\eta} ~A \circ K_{X,X,\frac{\eta}{2}},
}
\end{restatable}
The result above shows that we can efficiently marginalize a PSD model with respect to one variable by means of an entry-wise multiplication between two matrices.

\begin{remark}[Integration and marginalization on the hypercube]\label{rem:integration-on-hypercube}
The integrals in \cref{eq:integration,eq:marginalization} can be performed also on when $\X$ is a hypercube $H = \prod_{t=1}^d [a_t,b_t]$ rather than the entire space $\R^d$. This leads to a closed form, where the matrix $K_{X,X,\frac{\eta}{2}}$ is replaced by a suitable $K_{X,X,\frac{\eta}{2},H}$ that can be computed with same number of operations (the full form of such matrix is reported in \cref{app:operations}).  
\end{remark}

\paragraph{Product Rule (Multiplication)}
Multiplying two probabilities is key to several applications in probabilistic inference \cite{bishop2006}. The family of PSD models is closed with respect to this operation. 
\begin{restatable}[Multiplication]{proposition}{PMultiplication}\label{prop:multiplication}
Let $X\in\R^{n\times d_1}$, $Y\in\R^{n\times d_2}$, $Y'\in\R^{m\times d_2}$, $Z\in\R^{m\times d_3}$, $A\in\psd^n$, $B\in\psd^{m}$ and $\eta_1\in\R_{++}^{d_1}$, $\eta_2,\eta_2'\in\R_{++}^{d_2}$, $\eta_3\in\R_{++}^{d_3}$. Then
\eqal{\label{eq:multiplication}
    \pp{x,y}{A,[X,Y],(\eta_1,\eta_2)}
    \pp{y,z}{B,[Y',Z],(\eta_2',\eta_3)} = 
    \pp{x,y,z}{C,W,\eta},
} 
is a PSD model, where $C = (A \otimes B) \circ \left(\vec(K_{Y',Y,\widetilde{\eta}_2})\vec(K_{Y',Y,\widetilde{\eta}_2})^\top\right)$, with $\tilde\eta_2 = \tfrac{\eta_2\eta_2'}{\eta_2+\eta_2'}$, base matrix $W = [X\otimes \mathbf{1}_m, ~ (Y\tfrac{\eta_2}{\eta_2 + \eta_2'})\otimes\mbf{1}_m +  \mathbf{1}_n\otimes (Y'\tfrac{\eta_2'}{\eta_2 + \eta_2'}),~\mathbf{1}_n\otimes Z]$ and $\eta = \big(\eta_1,\eta_2+\eta_2',\eta_3\big)$.
\end{restatable}
We note that, despite the heavy notation, multiplying two PSD is performed via simple operations such as tensor and entry-wise product between matrices. In particular, we note that $X\otimes\mathbf{1}_m\in\R^{nm\times d_1}$ and $\mathbf{1}_n\otimes Z\in\R^{nm\times d_3}$ correspond respectively to the $nm\times d_1$-matrix containing $m$ copies of each row of $X$ and the $nm\times d_3$-matrix containing $n$ copies of $Z$. Finally, $(Y\eta)\otimes\mathbf{1}_m +  \mathbf{1}_n\otimes (Y\eta')$ is the $nm\times d_2$ base point matrix containing all possible sums $(\eta y_i + \eta' y_j')_{i,j=1}^{n,m}$ of points from $Y$ and $Y'$.

\paragraph{Reduction}
As observed above, when performing operations with PSD models, the resulting base point matrix might be of the form $X \otimes \mathbf{1}_m$ (e.g. if we couple the multiplication in \cref{eq:multiplication} with marginalization as in a Markov transition, see the corollary below). In these cases we can reduce the PSD model to have only $n$ base points (rather than $nm$), as follows
\eqal{\label{eq:reduction}
    \pp{x}{A, X\otimes\mathbf{1}_m, \eta} = \pp{x}{B,X,\eta} \qquad \textrm{with} \qquad B = (I_m \otimes \mathbf{1}_n^\top) A (I_m \otimes \mathbf{1}_n),
}
where $A\in\psd^{nm}$ and $I_m\in\R^{m\times m}$ is the $m\times m$ identity matrix. The reduction operation is useful to avoid the dimensionality of the PSD model grow unnecessarily. This is for instance the case of a Markov transition.

\begin{restatable}[Markov Transition]{corollary}{CMarkovTransition}\label{cor:markov}
Let $X\in\R^{n\times d_1}$, $Y\in\R^{n\times d_2}$, $Y'\in\R^{m\times d_2}$, $A\in\psd^n$, $B\in\psd^{m}$ and $\eta_1\in\R_{++}^{d_1}$, $\eta_2,\eta_2'\in\R_{++}^{d_2}$. Then
\eqal{
    \int \pp{x,y}{A,[X,Y],(\eta_1,\eta_2)}\pp{y}{B,Y',\eta_2'}~dy = \pp{x}{C,X,\eta_1},
}
with $C\in\psd^n$ obtained by applying in order, \cref{prop:multiplication}, \cref{prop:marginalization} and reduction \cref{eq:reduction}. 
\end{restatable}
We remark that the result of a Markov transition retains the same base point matrix $X$ and parameters $\eta$ of the transition kernel. This is thanks to the reduction operation in \cref{eq:reduction}, which avoids the resulting matrix $C$ to be $nm\times nm$. This fact is particularly useful in applications that require multiple Markov transitions, such as in hidden Markov models (see also \cref{sec:applications}).

\subsection{Compression of a PSD model}
\label{sec:compression}

From \cref{prop:multiplication} we note that the multiplication operation can rapidly yield a large number of base points and thus incur in high computational complexity in some settings. It might therefore be useful to have a method to reduce the number of base points while retaining essentially the same model. To this purpose, here we propose a dimensionality reduction strategy. In particular, given a set of points $\tilde{x}_1, \dots, \tilde{x}_m \in \R^d$, we leverage the representation of the PSD model in terms of reproducing kernel Hilbert spaces \cite{aronszajn1950theory} to use powerful sketching techniques as {\em Nystr\"om projection} (see, e.g., \cite{williams2001using}), to project the PSD model on a new PSD model now based only on the new points. Given $A \in \psd^n$, $X \in \R^{n\times d}$, $\eta \in \R^d_{++}$. Let $\tilde{X} \in \R^{m \times d}$ be the base point matrix whose $j$-rows corresponds to the point $\tilde{x}_j$. The {\em compression} of $\pp{\cdot}{A,X,\eta}$ corresponds to
\eqal{\label{eq:compression}
      \pp{x}{\tilde{A},\tilde{X},\eta} \qquad \textrm{with} \qquad \tilde{A} = BAB^\top \in \psd^m, \quad B = K_{\widetilde{X},\widetilde{X},\eta}^{-1}K_{\widetilde{X},X,\eta} \in \R^{m \times n},
}
which it is still a PSD model since the matrix $BAB^\top \in \psd^m$. We not that even with a rather simple strategy to choose the new base points $\tilde{x}_1,\dots,\tilde{x}_m$ -- such as uniform sampling -- compression is an effective tool to reduce the computational complexity of costly operations. In particular, in \cref{sec:effect-compression} we show that a compressed model with only $O(t \polylog(1/\eps))$ centers (instead of $t^2$) can $\eps$-approximate the product of two PSD models with $t$ points each.

\section{Representation power of PSD models}\label{sec:approximation}

In this section we study the theoretical properties of Gaussian PSD models. We start by showing that they admit concise approximations of the target density (in the sense discussed in the introduction to this paper and of \cref{tab:summary}). We then proceed to studying the setting in which we aim to learn an unknown probability from i.i.d. samples. We conclude the section by characterizing the approximation properties of the compression operation introduced in \cref{sec:compression}.

\subsection{Approximation properties of Gaussian PSD models}
\label{sec:approximation-results}

We start the section recalling that Gaussian PSD models are universal approximators for probability densities. In particular, the following result restates \citep[Thm. 2]{marteau2020non} for the case of probabilities.

\begin{proposition}[Universal consistency -- Thm. 2 in \cite{marteau2020non}]
\label{prop:universality}
The Gaussian PSD family is a universal approximator for probabilities that admit a density.
\end{proposition}
The result above is not surprising since Gaussian PSD models generalize classical Gaussian mixtures (see \cref{rem:psd-models-generalized-mixtures}), which are known to be universal \cite{bishop2006}. We now introduce a mild assumption that will enable us to complement \cref{prop:universality} with approximation and learning results. In the rest of the section we assume that $\X = (-1,1)^d$ (or more generally an open bounded subset of $\R^d$ with Lipschitz boundary). Here $L^\infty(\X)$ and $L^2(\X)$ denote respectively the space of essentially bounded and  square-integrable functions over $\X$, while $W^\beta_2(\X)$ denotes the {\em Sobolev space} of functions whose weak derivatives up to order $\beta$ are square-integrable on $\X$ (see \cite{adams2003sobolev} or \cref{sect:notation} for more details). 
\begin{assumption}\label{asm:psd-model}
Let $\beta > 0, q \in \N$. There exists $f_1,\dots, f_q \in W^\beta_2(\X) \cap L^\infty(\X)$, such that the density $p:\X\to\R$ satisfies
\eqals{
    p(x) = \sum_{j=1}^q f_j(x)^2, \qquad \forall x \in \X.
    }
\end{assumption}
The assumption above is quite general and satisfied by a wide family of probabilities, as discussed in the following proposition. The proof is reported in \cref{proof:prop:when-asm-psd-holds}

\begin{restatable}
[Generality of \cref{asm:psd-model}]{proposition}{PWhenAsmPSDHolds}\label{prop:when-asm-psd-holds}
The assumption above is satisfied by
\begin{enumerate}[label=(\alph*),itemsep=-0.1em,topsep=0em]
    \item\label{a} any probability density $p$ that $\beta$-times differentiable and strictly positive on $[-1,1]^d$. 
    \item\label{b} any exponential model $p(x) = e^{-v(x)}$ with $v \in W^\beta_2(\X) \cap L^\infty(\X)$,
    \item any mixture model of Gaussians or, more generally, of exponential models from (b),
    \item any $p$ that is $\beta+2$-times differentiable on $[-1,1]^d$, with a finite set of zeros, all in $(-1,1)^d$, and with positive definite Hessian in each zero. E.g. $p(x) \propto x^2 e^{-x^2}$.
\end{enumerate}
Moreover when $p$ is $\beta$-times differentiable over $[-1,1]^d$, then it belongs to $W^\beta_2(\X) \cap L^\infty(\X)$.
\end{restatable}
We note that in principle the class ${\cal C}_{\beta,d} = W^\beta_2(\X) \cap L^\infty(\X)$ is larger than the Nikolskii class usually considered in density estimation \cite{goldenshluger2014adaptive} when $\beta \in \N$, but \cref{asm:psd-model} imposes a restriction on it. However, \cref{prop:when-asm-psd-holds} shows that the restriction imposed by \cref{asm:psd-model} is quite mild, since it the resulting space includes all probabilities that are $\beta$-times differentiable and strictly positive (it also contains probabilities that are not strictly positive but have some zeros, see \cite{rudi2020finding}). We can now proceed to the main result of this work, which characterizes the approximation capabilities of PSD models. 

\begin{restatable}[Conciseness of PSD Approximation]{theorem}{TApproximation}\label{thm:approximation}
Let $p$ satisfy \cref{asm:psd-model}. Let $\eps > 0$. There exists a Gaussian PSD model of dimension $m \in \N$, i.e., $\hat{p}_m(x) = \pp{x}{A_m,X_m,\eta_m}$, with $A_m \in \psd^m$ and $X_m \in \R^{m\times d}$ and $\eta_m \in \R^{d}_{++}$, such that
\eqal{
\|p - \hat{p}_m\|_{L^2(\X)}  \leq \eps, \quad \textrm{with} \quad m = O(\eps^{-d/\beta} (\log \tfrac{1}{\eps})^{d/2}). 
}
\end{restatable}
The proof of \cref{thm:approximation} is reported in \cref{proof:thm:approximation}.
%
According to the result, the number of base points needed for a PSD model to approximate a density up to precision $\eps$ depends on its smoothness (the smoother the better) and matches the bound $m = \eps^{-d/\beta}$ that is optimal for function interpolation \cite{novak2006deterministic}, corresponding to models allowing for negative weights, and is also optimal for convex combinations of oscillating kernels \cite{turner2021statistical}.








\subsection{Learning a density with PSD models}\label{sec:learning-psd-main}
In this section we study the capabilities of PSD models to estimate a density from $n$ samples. Let $\X = (-1, 1)^d$ and let $p$ be a probability on $\X$. Denote by $x_1,\dots, x_n$ the samples independently and identically distributed according to $p$, with $n \in \N$. We consider a Gaussian PSD estimator $\hat{p}_{n,m} = \pp{x}{\hat{A}, \tilde{X}, \eta}$ that is built on top of $m$ additional points $\tilde{x}_1,\dots,\tilde{x}_m$, sampled independenly and uniformly at random in $\X$. In particular, $\eta \in \R^d_{++}, \tilde{X} \in \R^{m \times d}$ is the base point matrix whose $j$-th row corresponds to the point $\tilde{x}_j$ and $\hat{A} \in \psd^m$ is trained as follows
\eqal{\label{eq:learning-hatA}
\hat{A} = \argmin_{A \in \psd^m} \, \int_\X \pp{x}{A,\tilde{X},\eta}^2 dx - \tfrac{2}{n} \sum_{i=1}^n \pp{x_i}{A,\tilde{X},\eta} ~+~ \la \|K^{1/2} A K^{1/2}\|^2_F,
}
where $K = K_{\tilde{X},\tilde{X},\eta}$. Note that the functional is constituted by two parts. The first two elements are an empirical version of $\|f - p\|^2_{L^2(\X)}$ modulo a constant independent of $f$ (and so not affecting the optimization problem), since $x_i$ are identically distributed according to $p$ and so $\frac{1}{n} \sum_{i=1}^n f(x_i) \approx \int f(x) p(x) dx$. The last term is a regularizer and corresponds to $\|K^{1/2} A K^{1/2}\|^2_F = \tr(A K A K)$, i.e. the Frobenius norm of $\mm = \sum_{ij=1}^m A_{ij} \phi_\eta(\tilde{x}_i) \phi_\eta(\tilde{x}_j)$. The problem in \cref{eq:learning-hatA} corresponds to a quadratic problem with a semidefinite constraint and can be solved using techniques such as Newton method \cite{rudi2020finding} or first order dual methods  \cite{marteau2020non}. 
We are now ready to state our result.

\begin{restatable}{theorem}{TLearning}\label{thm:learning}
Let $n, m \in \N, \la > 0, \eta \in \R^d_{++}$ and $p$ be a density satisfying \cref{asm:psd-model}. With the definitions above, let $\hat{p}_{n,m}$ be the model $\hat{p}_{n,m}(x) = \pp{x}{\hat{A}, \tilde{X}, \eta}$, with $\hat{A}$ the minimizer of \cref{eq:learning-hatA}. Let $\eta = n^{\frac{2}{2\beta+d}}~{\bf 1}_d$ and $\la = n^{-\frac{2\beta+2d}{2\beta + d}}$. When $m \geq C' n^\frac{d}{2\beta+d} (\log n)^{d} \log(C'' n (\log n))$, the following holds with probability at least $1-\delta$,
\eqals{
\|p - \hat{p}_{n,m}\|_{L^2(\X)} \leq C n^{-\frac{\beta}{2\beta + d}}(\log n)^{d/2},
}
where constant $C$ depends only on $\beta, d$ and $p$ and the constants $C',C''$ depend only on $\beta, d$.
\end{restatable}
%
%
The proof of \cref{thm:learning} is reported in \cref{proof:thm:learning}. The theorem guarantees that under \cref{asm:psd-model}, Gaussian PSD models can achieve the rate $O(n^{-\beta/(2\beta+d)})$ -- that is optimal for the $\beta$-times differentiable densities -- while admitting a concise representation. Indeed, it needs a number $m = O(n^{d/(2\beta+d)})$ of base points, matching the optimal rate in \cite{turner2021statistical}. When $\beta \geq d$, a model with $m = O(n^{1/3})$ centers achieves optimal learning rates.

\subsection{The Effect of compression}\label{sec:effect-compression}
We have seen in the previous section that  Gaussian PSD models achieve the optimal learning rates, with concise models. However, we have seen in the operations section that multiplying two PSD models of $m$ centers leads to a PSD model with $m^2$ centers. Here we study the effect of compression, to show that it is possible to obtain an $\eps$-approximation of the product via a compressed model with $O(m \polylog(1/\eps))$ centers.  
In the following theorem we analyze the effect in terms of the $L^\infty$ distance on a domain $[-1,1]^d$, induced by the compression, when using points taken independently and uniformly at random from the same domain. 

Let $A \in \psd^n$, $X \in \R^{n\times d}$, $\eta \in \R^d_{++}$, we want to study the compressibility of the PSD model $p(x) = \pp{x}{A,X,\eta}$.  Let $\tilde{X} \in \R^{m \times d}$ be the base point matrix whose $j$-rows corresponds to the point $\tilde{x}_j$ with $\tilde{x}_1, \dots, \tilde{x}_m$ be sampled independently and uniformly at random from $[-1,1]^d$. Denote by $\tilde{p}_m(x)$ the PSD model $\tilde{p}_m(x) = \pp{x}{\tilde{A},\tilde{X},\eta}$ where $\tilde{A}$ is the compression of $A$ via \cref{eq:compression}. We have the following theorem.

\begin{restatable}[Compression of Gaussian PSD models]{theorem}{TApproximationNystrom}\label{thm:compression}
Let $\delta \in (0,1]$, $\eta_+ = \max(1, \max_{i=1,\dots,d} \eta_i)$. When $m$ satisfies 
\eqals{
m \geq O\left(\big(\eta_+^{1/2}\log \tfrac{\|A\|n}{\eps}\big)^d \log\tfrac{1}{\delta}\right),
}
then the following holds with probability at least $1-\delta$,
\eqals{
|p(x) - \tilde{p}_m(x)| ~\leq~ \eps^2 + \eps\sqrt{p(x)},\qquad\forall x \in [-1,1]^d,
}
\end{restatable}
%
The proof of the theorem above is in \cref{proof:thm:compression}. To understand its relevance, let $\hat{p}_1$ be a PSD model trained via \cref{eq:learning-hatA} on $n$ points sampled from $p_1$ and $\hat{p}_2$ trained from $n$ points sampled from $p_2$, where both $p_1, p_2$ satisfy \cref{asm:psd-model} for the same $\beta$ and $m, \la, \eta$ are chosen as \cref{thm:learning}, in particular $m=n^{d/(2\beta+d)}$ and $\eta = \eta_+ {\bf 1}_d$, $\eta_+ = n^{2/(2\beta+d)}$. Consider the model $\hat{p} = \hat{p}_1 \cdot \hat{p}_2$. By construction $\hat{p}$ has $m^2 = n^{2d/(2\beta+d)}$ centers, since it is the pointwise product of $\hat{p}_1$, $\hat{p}_2$ (see \cref{prop:multiplication}) and approximates $p_1 \cdot p_2$ with error $\eps = n^{-\beta/(2\beta+d)} \polylog(n)$, since both $\hat{p}_1, \hat{p}_2$ are $\eps$-approximators of $p_1,p_2$. Instead, by compressing $\hat{p}$, we obtain an estimator $\bar{p}$, that according to \cref{thm:compression}, achieves error $\eps$ with a number of center 
\eqals{
m' = O(\eta_+^{d/2} \polylog(1/\eps)) = O(n^{d/(2\beta + d)} \polylog(1/\eps)) = O(m \polylog(1/\eps)).
}
Then $\bar{p}$ approximates $p_1 \cdot p_2$ at the optimal rate $n^{-\beta/(2\beta+d)}$, but with a number of centers $m'$ that is only $O(m \polylog(n))$, instead of $m^2$. This means that $\bar{p}$ is essentially as good as if we learned it from $n$ samples taken directly from $p_1 \cdot p_2$. This renders compression a suitable method to reduce the computational complexity of costly inference operations as the product rule.

\section{Applications}
\label{sec:applications}
 
PSD models are a strong candidate in a variety of probabilistic settings. On the one hand, they are computationally amenable to performing key operations such as sum and product rules, similarly to mixture models (\cref{sec:operations}). On the other hand, they are remarkably flexible and can approximate/learn (coincisely) a wide family of target probability densities (\cref{sec:approximation}). Building on these properties,  in this section we consider different possible applications of PSD models in practice.

\subsection{PSD Models for Decision Theory}
\label{sec:decision-theory}

Decision theory problems (see e.g. \cite{bishop2006} and references therein) can be formulated as a minimization
\eqal{\label{eq:general-stochastic-optimization}
    \theta_* = \argmin_{\theta\in\Theta}~L(\theta) = \EE_{x\sim p}~ \ell(\theta,x),
}
where $\ell$ is a loss function, $\Theta$ is the space of target parameters (decisions) and $p$ is the underlying data distribution. When we can sample directly from $p$ -- e.g. in supervised or unsupervised learning settings -- we can apply methods such as stochastic gradient descent to efficently solve \cref{eq:general-stochastic-optimization}. However, in many applications, sampling from $p$ is challenging or computationally unfeasible. This is for instance the case when $p$ has been obtained via inference (e.g. it is the $t$-th estimate in a hidden Markov model, see \cref{sec:hmm}) or it is fully known but has a highly complex form (e.g. the dynamics of a physical system).  In contexts where sampling cannot be performed efficiently, it is advisable to consider alternative approaches. Here we propose a strategy to tackle \cref{eq:general-stochastic-optimization} when $p$ can be modeled (or well-approximated) by a PSD model. Our method hinges on the following result.

\begin{restatable}{proposition}{PMeanVariance}\label{prop:mean-variance-characteristic}
Let $p(x) = \pp{x}{A,X,\eta}$ with $X\in\R^{n \times d}$, $A\in\psd^n$, $\eta\in\R_{++}^d$. Let $g:\R^d\to\R$ and define $c_{g,\eta}(z) = \int g(x)e^{-\eta\nor{x-z}^2}~dx$ for any $z\in\R^d$. Then
\eqal{
\mathbb{E}_{x \sim p} ~ g(x) ~~=~~ \tr(\,(A \circ K_{X,X, \eta/2})\,G\,) \qquad \textrm{with} \qquad G_{ij} = c_{g, 2\eta}\big(\tfrac{x_i+x_j}{2}\big).
}
\end{restatable}
Thanks to \cref{prop:mean-variance-characteristic} we can readily compute several quantities related to a PSD model such as its mean $\EE_p[x] = X^\top b$ with $b = (A \circ K_{X, \eta/2}) {\bf 1}_n$, its covariance or its characteristic function (see \cref{app:operations} for the explicit formulas and derivations).
However, the result above is particularly useful to tackle the minimization in \cref{eq:general-stochastic-optimization}. In particular, since $\nabla L(\theta) = \EE_{x\sim p}\nabla_\theta\ell(\theta,x)$, we can use \cref{prop:mean-variance-characteristic} to directly compute the gradient of the objective function: it is sufficient to know how to evaluate (or approximate) the integral $c_{\nabla\ell_\theta,\eta}(z) = \int \nabla_\theta\ell(\theta,x)e^{-\eta\nor{x-z}^2}~dx$ for any $\theta\in\Theta$ and $z\in\R^d$. Then, we can use first order optimization methods, such as gradient descent, to efficiently solve \cref{eq:general-stochastic-optimization}. Remarkably, this approach works well also when we approximate $p$ with a PSD model $\hat p$. If $\ell$ is convex, since $\hat p$ is non-negative, the resulting $\hat L(\theta) = \EE_{x\sim\hat p}~\ell(\theta,x)$ is still a convex functional (see also the discussion on structured prediction in \cref{sec:conditional-psd}). This is not the case if we use more general estimators of $p$ that do not preserve non-negativity.

\subsection{PSD Models for Estimating Conditional Probabilities}
\label{sec:conditional-psd}
In supervised learning settings, one is typically interested in solving decision problems of the form $\min_{\theta\in\Theta}~\EE_{(x,y)\sim p}~ \ell(h_\theta(x),y)$ where $p$ is a probability over the joint input-output space $\X\times\Y$ and $h_\theta:\X\to\Y$ is a function parameterized by $\theta$. It is well-known (see e.g. \cite{steinwart2008}) that the ideal solution of this problem is the $\theta_*$ such that for any $x\in\X$ the function $h_{\theta_*}(x) = \argmin_{z\in\Y}~\EE_{y\sim p(\cdot|x)}~ \ell(z,y)$ is the minimizer with respect to $z\in\Y$ of the conditional expectation of $\ell(z,y)$ given $x$. This leads to target functions that capture specific properties of $p$, such as moments. For instance, when $\ell$ is the squared loss,  $h_{\theta_*}(x)=\EE_{y\sim p(\cdot|x)} y$ corresponds to the conditional expectation of $y$ given $x$, while for $\ell$ the absolute value loss, $h_{\theta_*}$ recovers the conditional median of $p$. 

In several applications, associating an input $x$ to a single quantity $h_\theta(x)$ in output is not necessarily ideal. For instance, when $p(y|x)$ is multi-modal, estimating the mean or median might not yield useful predictions for the given task. Moreover, estimators of the form $h_\theta$ require access to the full input $x$ to return a prediction, and therefore cannot be used when some features are missing (e.g. due to data corruption). In these contexts, an alternative viable strategy is to directly model the conditional probability. When using PSD models, conditional estimation can be performed in two steps, by first modeling the joint distribution $p(y,x) = \pp{y,x}{A,[Y,X],(\eta,\eta')}$ (e.g. by learning it as suggested \cref{sec:approximation}) and then use the operations in \cref{sec:operations} to condition it with respect to $x_0\in\X$ as
\eqal{\label{eq:conditional-psd}
    p(y|x_0) = \frac{p(y,x_0)}{p(x_0)} = \pp{y}{B,Y,\eta} \quad \textrm{with} \quad B = \frac{A \circ (K_{X,x_0,\eta'}K_{X,x_0,\eta'}^\top)}{c_{2\eta}\tr\big(A \circ (K_{X,x_0,\eta'}K_{X,x_0,\eta'}^\top) K_{Y,Y,\eta}\big)}.
}
In case of data corruption, it is sufficient to first marginalize $p(y,x)$ on the missing variables and then apply \cref{eq:conditional-psd}. Below we discuss a few applications of the conditional estimator.

\paragraph{Conditional Expectation}
Conditional mean embeddings \cite{song2013kernel} are a well-established tool to efficiently compute the conditional expectation $\EE_{y\sim p(\cdot|x_0)}~ g(y)$ of a function $g:\Y\to\R$. However, although they enjoy good approximation properties \cite{muandet2016kernel}, they to not guarantee the resulting estimator to take only non-negative values. In contrast, when $p$ is a PSD model (or an approximation), we can apply \cref{prop:mean-variance-characteristic} to $p(\cdot|x_0)$ in \cref{eq:conditional-psd} and evaluate the conditional expectation of any $g$ for which we know how to compute (or approximate) the integral $c_{g,\eta}(z) = \int g(y)e^{-\eta\nor{y-z}^2}~dy$. In particular we have $\EE_{y\sim p(\cdot|x)}~ g(y) = \tr((B\circ K_{Y,Y,\eta/2})G)$ with $B$ as in \cref{eq:conditional-psd} and $G$ the matrix with entries $G_{ij} = c_{g,2\eta}(\tfrac{y_i+y_j}{2})$. Remarkably, differently from conditional mean embeddings estimators, this strategy allows us to compute the conditional expectations also of functions $g$ not in $\hh_\eta$. 

\paragraph{Structured Prediction}
Structured prediction identifies supervised learning problems where the output space $\Y$ has complex structures so that it is challenging to find good parametrizations for $h_\theta:\X\to\Y$ \cite{bakir2007predicting,nowozin2011}. In \cite{ciliberto2020general}, a strategy was proposed to tackle these settings by first learning an approximation $\psi_\theta(z,x)\approx\EE_{y\sim p(\cdot|x)}~\ell(z,y)$ and then model $h_\theta(x) = \argmin_{z\in\Y} \psi_\theta(z,x)$. However, the resulting function $\psi_\theta(\cdot,x)$ is not guaranteed to be convex, even when $\ell$ is convex. In contrast, by combining the conditional PSD estimator in \cref{eq:conditional-psd} with the reasoning in \cref{sec:decision-theory}, we have a strategy that overcomes this issue: when $p(\cdot|x)$ is a PSD model approximating $p$, we can compute its gradient $\EE_{y\sim p(\cdot|x)}~\nabla_z\ell(z,y)$ as mentioned in \cref{sec:decision-theory} using \cref{prop:mean-variance-characteristic}. Moreover, if $\ell(\cdot,y)$ is convex, the term $\EE_{y\sim p(\cdot|x)}~\ell(\cdot,y)$ is also convex, and we can use methods such as gradient descent to find $h(x)$ exactly.

\paragraph{Mode Estimation}
When the output distribution $p(y|x)$ is multimodal, having access to an explicit form for the conditional density can be useful to estimate its modes. This problem is typically non-convex, yet, when the output $y$ belongs to a small dimensional space (e.g. in classification or scalar-valued regression settings), efficient approximations exist (e.g. bisection). 


\subsection{Inference on Hidden Markov Models}
\label{sec:hmm}

We consider the problem of performing inference on hidden Markov models (HMM) using PSD models. Let $(x_t)_{t\in\N}$ and $(y_t)_{t\in\N}$ denote two sequences of states and observations respectively. For each $t\geq1$, we denote by $x_{0:t} = x_0,\dots,x_t$ and $y_{1:t} = y_1,\dots,y_t$ and we assume that $p(x_{t}|x_{0:t-1},y_{1:t-1}) = p(x_{t}|x_{t-1}) = \tau(x_{t},x_{t-1})$ and $p(y_{t}|x_{0:t},y_{1:t-1}) = p(y_{t}|x_{t}) = \omega(y_{t},x_{t})$ with $\tau:\X\times\X\to\R_+$ and $\omega:\Y\times\X\to\R_+$ respectively the transition and observation functions. 

Our goal is to infer the distribution of possible states $x_{t}$ at time $t$, given all the observations $y_{1:t}$ and a probability $p(x_0)$ on the possible initial states. We focus on this goal for simplicity, but other forms of inferences are possible (e.g. estimating $x_{m+t}$, namely $m$ steps into the future or the past). We assume that the transition and observation functions can be well approximated by PSD models $\hat\tau$ and $\hat\omega$ (e.g. by learning them or known a-priori for the problem's dynamics). Then, given a PSD model estimate $\hat p(x_0)$ of the initial state $p(x_0)$, we can recursively define the sequence of estimates
\eqals{\label{eq:hmm-approximate}
    \hat p(x_{t}|y_{t:1}) = \frac{\hat\tau(y_{t},x_{t})\int \omega(x_{t},x_{t-1})\hat p(x_{t}|y_{1:t-1})~dx_{t-1}}{\int\hat\tau(y_{t},x_{t}) \omega(x_{t},x_{t-1})\hat p(x_{t}|y_{1:t-1})~dx_{t}d x_{t-1}}.
}
Note that when $\hat\tau, \hat \omega, \hat p(x_0)$ correspond to the real transition, observation and initial state probability, the formula above yields the exact distribution $p(x_t|y_{1:t})$ over the states at time $t$ (this follows directly by subsequent applications of Bayes' rule. See also e.g. \cite{bishop2006}). If $\hat\tau, \hat \omega, \hat p(x_0)$ are PSD models, then each of the $\hat p(x_t|y_{t:1})$ is a PSD model recursively defined only in terms of the previous estimate and the operations introduced in \cref{sec:operations}. In particular, we have the following result.

\begin{algorithm}[t]
   \caption{PSD Hidden Markov Model \label{alg:psd-hmm}}
\begin{algorithmic}
\small
    \State {\bfseries Input:} Transition $\hat \tau(x_+,x) = \pp{x_+,x}{B,[X_+,X],(\eta_+,\eta)}$, initial $\hat p(x_0) = \pp{x_0}{A_0,X_0,\eta_0}$ and observation $\hat \omega(y,x) = \pp{y,x}{C,[Y,X'],(\eta_{obs},\eta')}$ distributions. $\tilde X$ as in \cref{prop:psd-hmm}. $\tilde\eta = \tfrac{\eta(\eta'+\eta_+)}{\eta+\eta'+\eta_+}$. 
    
    \vspace{0.2em}
    \State $\tilde X' =  (X\frac{\eta}{\eta + \eta'+\eta_+})\otimes \mathbf{1}_{nm} + \mathbf{1}_n\otimes  (\tilde X \frac{\eta'+\eta_+}{\eta + \eta'+\eta_+}),$ \qquad and~~ $\tilde\eta' = \tfrac{\eta'\eta_+}{\eta'+\eta_+}$
    
    \vspace{0.5em}
    
    \State {\bfseries For} any new observation $y_t$:
    \vspace{0.25em}
    \State\quad $C_t = C \circ (K_{Y,y_t,\eta_{obs}}^\top K_{Y,y_t,\eta_{obs}})$ \qquad\qquad\qquad\qquad\qquad $\codecomment$ Partial evaluation~ $\hat\omega_t(x_+) = \hat\omega(y=y_t,x_+)$
    \State\quad $B_t = (B \otimes A_{t-1}) \circ \big(\vec(K_{ \tilde X, X,\tilde\eta})\vec(K_{ \tilde X,X,\tilde\eta})^\top\big)$ \hspace{2em} $\codecomment{}$ Product~ $\hat\beta_t(x_+,x) = \hat\tau(x_+,x)\hat p(x|y_{1:t-1})$~
    \State\quad $D_t = (I_n\otimes\mathbf{1}_{nm}^\top)(B_t \circ K_{\tilde{X}',\tilde{X}',\frac{\tilde\eta}{2}} )(I_n\otimes\mathbf{1}_{nm})$ \qquad\quad~~ $\codecomment$ Marginalization~ $\hat\beta_t(x_+) = \int \beta_t(x_+,x)~dx$ 
    \State\quad $E_t = (C_t \otimes D_t)\circ \big(\vec(K_{X',X_+,\frac{\tilde\eta'}{2}})\vec(K_{X',X_+,\frac{\tilde\eta'}{2}})^\top\big)$ \quad~~~\, $\codecomment$ Product~ $\hat\pi_t(x_+)=\hat\omega_t(x_+)\hat\beta_t(x_+)$
    \State\quad $A_t = E_t/c_t$, with $c_t = c_{2(\eta'+\eta_+)}\tr(E_t K_{\tilde{X},\tilde X,\frac{\eta'+\eta_+}{2}})$ \hspace{1em} $\codecomment$ Normalization~ $\hat p(x_+|y_{1:t}) = \tfrac{\hat\pi_t(x_+)}{\int \hat\pi(x_+) ~dx}$
    
    \vspace{0.5em}
    \State{\bfseries Return}~~ $\pp{x_t}{A_t,\tilde X,\eta'+\eta_+}$
\end{algorithmic}
\end{algorithm}
\begin{restatable}[PSD Hidden Markov Models (HMM)]{proposition}{PHMM}\label{prop:psd-hmm}
Let $X_0\in\R^{n_0
\times d}$, $X_+,X\in\R^{n \times d}$, $X'\in\R^{m \times d}$, $Y\in\R^{m \times d'}$, $A_0\in\psd^{n_0}, A\in\psd^n$, $B\in\psd^m$ and $\eta_0,\eta,\eta',\eta_+\in\R_{++}^d$, $\eta_{obs}\in\R_{++}^{d'}$. Let 
\eqals{
\hat \tau(x_+,x) = \pp{x_+,x}{B,[X_+,X],(\eta_+,\eta)}, \qquad \hat \omega(y,x)  = \pp{y,x}{C,[Y,X'],(\eta_{obs},\eta')},
}
be approximate transition and observation functions. Then, given the initial state probability $\hat p(x_0) = \pp{x_0}{A_0,X_0,\eta_0}$, for any $t\geq1$, the estimate $\hat p$ in \cref{eq:hmm-approximate} is a PSD model of the form
\eqal{\label{eq:hmm-psd-model}
    \hat p(x_{t}|y_{t:1}) = \pp{x_{t}}{A_{t},\tilde X,\eta' + \eta_+ },
}
where $\tilde X =  (X'\tfrac{\eta'}{\eta'+\eta_+})\otimes \mathbf{1}_n + \mathbf{1}_m\otimes (X_+\tfrac{\eta_+}{\eta'+\eta_+})$ and $A_{t}$ is recursively obtained from $A_{t-1}$ as in \cref{alg:psd-hmm}.
\end{restatable}
\begin{remark}[Sum-product Algorithm]
\label{rem:sum-product-algorithm}
Eq. \cref{eq:hmm-approximate} is an instance of the so-called sum-product algorithm, a standard inference method for graphical models \cite{bishop2006} (of which HMMs are a special case). The application of the sum-product algorithm relies mainly on sum and product rules for probabilities (as is the case for HMMs in \cref{eq:hmm-approximate}). Hence, according to \cref{sec:operations}, it is highly compatible with PSD models.  
\end{remark}

\section{Discussion}

In this work we have shown that PSD models are a strong candidate in practical application related to probabilistic inference. They satisfy both requirements for an ideal probabilistic model: $i)$ they perform exact sum and product rule in terms of efficient matrix operations; $ii)$ we proved that they can concisely approximate a wide range of probabilities.

\paragraph{Future Directions}
We identify three main directions for future work: $i)$ when performing inference on large graphical models (see \cref{rem:sum-product-algorithm}) the multiplication of PSD models might lead to an inflation in the number of base points. Building on our compression strategy, we plan to further investigate low-rank approximations to mitigate this issue. $ii)$ An interesting problem is to understand how to efficiently sample from a PSD model. A first answer to this open question was recently given in \cite{marteau2021sampling}. $iii)$ The current paper has a purely theoretical and algorithmic focus. In the future, we plan to investigate the empirical behavior of PSD models on the applications introduced in \cref{sec:applications}. Related to this, we plan to develop a library for operations with PSD models and make it available to the community.

\paragraph{Acknowledgments} We thanks Gaspard Beugnot for proof reading the paper.
A.R. acknowleges the support of the French government under management of Agence Nationale de la Recherche as part of the “Investissements d’avenir” program, reference ANR-19-P3IA-0001 (PRAIRIE 3IA Institute) and the support of the European Research Council (grant REAL 947908). C.C. acknowledges the support of the Royal Society (grant SPREM RGS$\backslash$R1$\backslash$201149) and Amazon.com Inc. (Amazon Research Award -- ARA 2020)

{
\bibliographystyle{plain}
\bibliography{biblio}
}

\newpage

\renewcommand{\theHsection}{A\arabic{section}}



\numberwithin{equation}{section}
\numberwithin{lemma}{section}
\numberwithin{proposition}{section}
\numberwithin{theorem}{section}
\numberwithin{corollary}{section}
\numberwithin{definition}{section}
\numberwithin{algorithm}{section}
\numberwithin{remark}{section}

\appendix

{\Huge\textbf{Appendix}}

\vspace{2em}
The appendix is organized as follows:

\begin{itemize}
    \item \cref{sect:notation} introduces notation and some key definitions and results that will be useful to prove the results in this work.
    
    \item \cref{sec:operators-definition} provides basic notation and definitions for working with linear operators between reproducing kernel Hilbert spaces.
    
    \item \cref{app:compression} discusses in detail the effects of compression introduced in \cref{sec:compression} in the main paper. In particular we study the approximation error incurred by a compressed model as a function of the number of base points used.
    
    \item \cref{app:approximation} reports the proofs of the results in \cref{sec:approximation-results} regarding the approximation properties of PSD models. 
    
    \item \cref{app:learning} reports the proof of \cref{thm:learning} characterizing the learning capabilities of PSD models.
    
    \item \cref{app:operations} provides the derivations for the PSD models operations discussed in \cref{sec:operations} as well as some results directly related, namely \cref{prop:mean-variance-characteristic} and \cref{prop:psd-hmm}.
    
\end{itemize}


\section{Notation and definitions}\label{sect:notation}

We introduce basic notation and review results that will be useful in the following. 

\paragraph{Multi-index notation} Let $\alpha \in \N_0^d$, $x \in \R^d$ and $f$ be an infinitely differentiable function on $\R^d$, we introduce the following notation
$$|\alpha| = \sum_{j=1}^d \alpha_i, \quad \alpha! = \prod_{j=1}^d \alpha_j!, \quad x^\alpha = \prod_{j=1}^d x_j^{\alpha_j}, \quad \partial^\alpha f = \frac{\partial^{|\alpha|} f}{\partial x_1^{\alpha_1}\cdots\partial x_d^{\alpha_d}}.$$
We introduce also the notation $D^\alpha$ that corresponds to the multivariate distributional derivative of order $\alpha$ and such that 
$$D^\alpha f = \partial^\alpha f$$
for functions that are differentiable at least $|\alpha|$ times \cite{adams2003sobolev}.

\paragraph{Fourier Transform} 
Given two functions $f,g:\R^d \to \R$ on some set $\R^d$, we denote by $f \cdot g$ the function corresponding to {\em pointwise product} of $f, g$, i.e.,
$$(f \cdot g)(x) = f(x)g(x), \quad \forall x \in \R^d.$$
Let $f, g \in L^1(\R^d)$ we denote the {\em convolution} by $f \star g$ 
$$(f \star g)(x) = \int_{\R^d} f(y) g(x-y) dy.$$
We now recall some basic properties, that will be used in the rest of the appendix.
\bp[Basic properties of the Fourier transform \cite{wendland2004scattered}, Chapter 5.2.]\label{prop:fourier}
$ $

\begin{enumprop}
\item\label{prop:fourier:L2} There exists a linear isometry ${\cal F}: L^2(\R^d) \to L^2(\R^d)$ satisfying 
$${\cal F}[f] = \int_{\R^d} e^{-2 \pi i \,\omega^\top x} \,f(x)\, dx \quad  \forall f \in L^1(\R^d) \cap L^2(\R^d),$$
where $i = \sqrt{-1}$. The isometry is uniquely determined by the property in the equation above. 
\item\label{prop:fourier:plancherel} Let $f \in L^2(\R^d)$, then $\|{\cal F}[f]\|_{L^2(\R^d)} = \|f\|_{L^2(\R^d)}$.
\item\label{prop:fourier:scale} Let $f \in L^2(\R^d), r > 0$ and define $f_r(x) = f(\frac{x}{r}), \forall x \in \R^d$, then ${\cal F}[f_r](\omega) = r^d {\cal F}[f](r\omega)$. 
\item\label{prop:fourier:product} Let $f, g \in L^1(\R^d)$, then ${\cal F}[f \cdot g] =  {\cal F}[f] \star {\cal F}[g]$.
\item\label{prop:fourier:derivative} Let $\alpha \in \N_0^d$,  $f, D^\alpha f \in L^2(\R^d)$, then
${\cal F}[D^\alpha f](\omega) = (2\pi i)^{|\alpha|} \omega^\alpha {\cal F}[f](\omega)$, $\forall \omega \in \R^d$.
\item\label{prop:fourier:Linfty-omega} Let $f \in L^1(\R^d) \cap L^2(\R^d)$, then $\|{\cal F}[f]\|_{L^\infty(\R^d)} \leq \|f\|_{L^1(\R^d)}$.
\item\label{prop:fourier:Linfty-x} Let $f \in L^\infty(\R^d) \cap L^2(\R^d)$, then $\|f\|_{L^\infty(\R^d)} \leq \|{\cal F}[f]\|_{L^1(\R^d)}$.
\end{enumprop}
\ep

\paragraph{Reproducing kernel Hilbert spaces for translation invariant kernels.} We now list some important facts about reproducing kernel Hilbert spaces in the case of translation invariant kernels on $\R^d$. For this paragraph, we refer to \cite{steinwart2008,wendland2004scattered}. For the general treatment of positive kernels and Reproducing kernel Hilbert spaces, see \cite{aronszajn1950theory,steinwart2008}.
Let $v:\R^d \to \R$ such that its Fourier transform ${\cal F}[v] \in L^1(\R^d)$ and satisfies ${\cal F}[v](\omega) \geq 0$ for all $\omega \in \R^d$. Then, the following hold.
\begin{enumerate}[label=(\alph*)]
\item The function $k:\R^d \times \R^d \to \R$ defined as $k(x,x') = v(x-x')$ for any $x,x' \in \R^d$ is a positive kernel and is called {\em translation invariant kernel}.
\item The {\em reproducing kernel Hilbert space} (RKHS) $\hh$ and its norm $\|\cdot\|_{\hh}$ are characterized by 
\eqal{\label{eq:tr-inv-rkhs-def}
\hh = \{f \in L^2(\R^d) ~|~ \|f\|_{\hh} < \infty \}, \quad \|f\|^2_{\hh} = \int_{\R^d} \frac{|{\cal F}[f](\omega)|^2}{{\cal F}[v](\omega)}d\omega,
}
\item $\hh$ is a separable Hilbert space, whose inner product $\scal{\cdot}{\cdot}_{\hh}$ is characterized by
$$\scal{f}{g}_{\hh} = \int_{\R^d} \frac{{\cal F}[f](\omega)\overline{{\cal F}[g](\omega)}}{{\cal F}[v](\omega)} d\omega.$$
In the rest of the paper, when clear from the context we will simplify the notation of the inner product, by using $f^\top g$ for $f,g \in \hh$, instead of the more cumbersome $\scal{f}{g}_\hh$.
\item The feature map $\phi:\R^d \to \hh$ is defined as $\phi(x) = k(x-\cdot) \in \hh$ for any $x \in \R^d$.
\item The functions in $\hh$ have the {\em reproducing property}, i.e.,
\eqals{\label{eq:reproducing-property}
f(x) = \scal{f}{\phi(x)}_\hh, \quad \forall f \in \hh, x \in \R^d,
}
in particular $k(x',x) = \scal{\phi(x')}{\phi(x)}_\hh$ for any $x',x \in \R^d$.
\end{enumerate}

We now introduce an important example of translation invariant kernel and the associated RKHS, that will be useful in our analysis.

\begin{example}[Gaussian Reproducing kernel Hilbert space]\label{ex:gaussian-rkhs}
Let $\eta \in \R^d_{++}$ and $k_\eta(x,x') = e^{-(x-x')^\top\diag(\eta)(x-x')}$, for $x,x' \in \R^d$ be the Gaussian kernel with precision $\eta$. The function $k_\eta$ is a translation invariant kernel, since $k_\eta(x,x') = v(x-x')$ with $v(z) = e^{-\|D^{1/2}z\|^2}, D = \diag(\eta)$ and ${\cal F}[v](\omega) = c_\eta e^{-\pi^2 \|D^{-1/2}\omega\|^2}$, $c_\eta = \pi^{d/2} \det(D)^{-1/2}$, for $\omega \in \R^d$ is in $L^1(\R^d)$ and satisfies ${\cal F}[v](\omega) \geq 0$ for all $\omega \in \R^d$. The associated reproducing kernel Hilbert space $\hh_\eta$ is defined according to \cref{eq:tr-inv-rkhs-def}, with norm
\eqal{
\|f\|^2_{\hh_\eta} ~~=~~ \frac{1}{c_\eta} \int_{\R^d} ~|{\cal F}[f](\omega)|^2 ~e^{\pi^2 \|D^{-1/2}\omega\|^2} ~d\omega, \qquad \forall f \in L^2(\R^d).
}
The inner product and the feature map $\phi_\eta$ are defined as in the discussion above.
\end{example}

\subsection{Sobolev spaces}
Let $\beta \in \N, p \in [1,\infty]$ and let $\Omega \subseteq \R^d$ be an open set. The set $L^p(\Omega)$ denotes the set of $p$-integrable functions on $\Omega$ for $p \in [1,\infty)$ and that of the essentially bounded on $\Omega$ when $p = \infty$.
The set $W^\beta_p(\Omega)$ denotes the Sobolev space, i.e., the set of measurable functions with their distributional derivatives up to $\beta$-th order belonging to $L^p(\Omega)$,
\eqal{\label{eq:norm-sobolev-derivative}
W^\beta_p(\Omega) = \{f \in L^p(\Omega) ~|~\|f\|_{W^\beta_p(\Omega)} < \infty\}, \quad \|f\|^p_{W^\beta_p(\Omega)} = \sum_{|\alpha| \leq \beta} \|D^\alpha f\|^p_{L^p(\Omega)},
}
where $D^\alpha$ denotes the distributional derivative. In the case of $p = \infty$, 
\eqals{
\|f\|_{W^\beta_\infty(\Omega)} = \max_{|\alpha| \leq \beta} \|D^\alpha f\|_{L^\infty(\Omega)}
}

We now recall some basic results about Sobolev spaces that are useful for the proofs in this paper.
First we start by recalling the restriction properties of Sobolev spaces. Let $\Omega \subseteq \Omega' \subseteq \R^d$ be two open sets. Let $\beta \in \N$ and $p \in [1,\infty]$. By definition of the Sobolev norm above we have
$$\|g|_\Omega\|_{W^s_p(\Omega)} \leq \|g\|_{W^s_p(\Omega')},$$
and so $g|_\Omega \in W^s_p(\Omega)$ for any $g \in W^s_p(\Omega')$. Now we recall the extension properties of Sobolev spaces.

\bp[Stein total extension theorem, 5.24 in \cite{adams2003sobolev} or \cite{rogers2004degree}]\label{prop:extension-sobolev}
Let $\Omega$ be a bounded open subset of $\R^d$ with locally Lipschitz boundary \cite{adams2003sobolev}. For any measurable function $h:\Omega \to \R$, there exists a function $\tilde{h}:\R^d \to \R$, such that $\tilde{h}|_\Omega = h$ almost everywhere on $\Omega$ and for any $\beta \geq 0$ and $p \in [1,\infty]$, the condition $\|h\|_{W^\beta_p(\Omega)} < \infty$ implies $\|\tilde{h}\|_{W^\beta_p(\R^d)} \leq C_{\Omega, \beta, p} \|h\|_{W^\beta_p(\Omega)}$ with $C_{\Omega, \beta, p} < \infty$ and not depending on $h, \tilde{h}$, but only on $\Omega, \beta, p$. 
\ep

\begin{corollary}\label{cor:extension-intersection}
Let $\X \subset \R^d$ be a non-empty open set with Lipschitz boundary. Let $\beta \in \N, p \in [1,\infty]$. Then for any function $f \in W^\beta_p(\X) \cap L^\infty(\X)$ there exists an extension $\tilde{f}$ on $\R^d$, i.e. a function $\tilde{f} \in W^\beta_p(\R^d) \cap L^\infty(\R^d)$ such that 
\eqals{
\tilde{f} = f|_\X ~\textrm{a.e. on } \X, \quad \|\tilde{f}\|_{L^\infty(\R^d)} \leq C \|f\|_{L^\infty(\X)}, \quad \|\tilde{f}\|_{W^\beta_p(\R^d)} \leq C' \|f\|_{L^\infty(\X)}.
}
The constant $C$ depends only on $\X, d$, and the constant $C'$ only on $\X,\beta,d,p$ 
\end{corollary}

\bp\label{prop:inclusion-Cb-in-Wb-Linfy}
Let $m \in \N$. Let $\X$ be an open bounded set with Lipschitz boundary. Let $f$ be a function that is $m$ times differentiable on $\overline{\X}$, the closure of $\X$. Then there exists a function $\tilde{f} \in W^m_p(\X) \cap L^\infty(\X)$ for any $p \in [1,\infty]$, such that $\tilde{f} = f$ on $\X$.
\ep
\begin{proof}
A function $f$ that is $m$-times differentiable on the closure of $X$ belongs also to $W^{m}_\infty(\X)$ since each derivative up to order $m$ is continuous and the set $\X$ is bounded. Then $f$ satisfies also $f \in W^m_p(\X) \cap L^\infty(\X)$ since $W^m_\infty(\X) \subset L^\infty(\X)$ and $W^m_\infty(\X) \subset W^\beta_p(\X)$, by construction, for bounded $\X$ and any $p \in [1,\infty]$.
\end{proof}

The following proposition provides a useful characterization of the space $W^\beta_2(\R^d)$
\begin{proposition}[Characterization of the Sobolev space $W^k_2(\R^d)$, \cite{wendland2004scattered}]\label{prop:sobolev}
Let $k \in \N$. The norm of the Sobolev space $\|\cdot\|_{W^k_2(\R^d)}$ is equivalent to the following norm
\eqals{
\|f\|'^{\,2}_{W^k_2(\R^d)} = \int_{\R^d} ~|{\cal F}[f](\omega)|^2 ~(1+\|\omega\|^2)^{k} ~d\omega, \quad \forall f \in L^2(\R^d)
}
and satisfies
\eqal{\label{eq:sobolev-norm-char}
\tfrac{1}{(2\pi)^{2k}} \|f\|^2_{W^k_2(\R^d)} \leq \|f\|'_{W^k_2(\R^d)}~\leq~ 2^{2k} \|f\|^2_{W^k_2(\R^d)}, \quad \forall f \in L^2(\R^d)
}
Moreover, when $k > d/2$, then $W^k_2(\R^d)$ is a reproducing kernel Hilbert space.
\end{proposition}
\begin{proof}
Consider first the seminorm $|g|^2_{W^t_2(\R^d)} = \sum_{|\alpha| \leq t} \|D^\alpha g\|^2_{L^2(\R^d)}$. We have that $\|g\|^2_{W^k_2(\R^d)} = \sum_{t=0}^k |g|^2_{W^t_2(\R^d)}$. Now let $0 \leq t \leq k$. By using the properties of the Fourier transform (in particular, the Plancherel theorem and the transform of a distributional derivative \cref{prop:fourier}) we have that 
$$
|g|^2_{W^t_2(\R^d)} = \sum_{|\alpha| = t} \|D^\alpha g\|^2_{L^2(\R^d)}  = \sum_{|\alpha| = t} \|(2\pi i)^\alpha \omega^\alpha {\cal F}[g](\omega)\|^2_{L^2(\R^d)} = \int \sum_{|\alpha| = t} (2\pi \omega)^{2\alpha} |{\cal F}[g](\omega)|^2 d\omega.
$$
Now note that, by the multinomial theorem,
$\|2\pi \omega\|^{2t} = (2\pi \omega^2_1 + \dots +  2\pi \omega^2_d)^k = \sum_{|\alpha|=k} \binom{k}{\alpha} (2\pi\omega)^\alpha$. Since $1 \leq \binom{t}{\alpha} \leq 2^t$ for any $\alpha \in \N_0^d, |\alpha| = t$, then $2^{-t}\|2\pi \omega\|^{2t} \leq \sum_{|\alpha| = t} (2\pi \omega)^{2\alpha} \leq \|2\pi \omega\|^{2t}$, so 
\eqal{\label{eq:seminorm-sobolev}
|g|^2_{W^t_2(\R^d)} ~\leq~  (2\pi)^{2t} \int \|\omega\|^{2t}\, |{\cal F}[g](\omega)|^2 d\omega ~\leq~ 2^t\, |g|^2_{W^t_2(\R^d)}.
}
Since, $(1+\|\omega\|^2)^k = \sum_{t=0}^k \binom{k}{t} \|\omega\|^{2k}$ and so $\sum_{t=0}^k \|\omega\|^{2t} \leq (1+\|\omega\|^2)^k \leq 2^{k} \sum_{t=0}^k \|\omega\|^{2t}$, then
$$
\|g\|_{W^k_2(\R^d)} = \sum_{t=0}^k |g|^2_{W^t_2(\R^d)}  \leq (2\pi)^{2k} \sum_{t=0}^k \int \|\omega\|^{2t}\, |{\cal F}[g](\omega)|^2 d\omega \leq (2\pi)^{2k} \int (1+\|\omega\|^2)^{k} |{\cal F}[g](\omega)|^2,
$$
moreover
$$
\int (1+\|\omega\|^2)^{k} |{\cal F}[g](\omega)|^2 \leq 2^k \sum_{t=0}^k (2\pi)^{2t} \int \|\omega\|^{2t} |{\cal F}[g](\omega)|^2 \leq 2^{2k} \sum_{t=0}^k |g|^2_{W^k_2(\R^d)} = 2^{2k} \|g\|_{W^k_2(\R^d)}. 
$$
To conclude, we recall that when $k > d/2$ the space $W^k_2(\R^d)$ endowed with the equivalent norm $\|\cdot\|'_{W^k_2(\R^d)}$ is a reproducing kernel Hilbert space \cite{wendland2004scattered}.
\end{proof}

\section{Useful linear operators in RKHS}\label{sec:operators-definition}

Consider the space ${\cal G}_\eta = \hh_\eta \otimes \hh_\eta = \{v \otimes v' ~|~ v,v' \in \hh_\eta\}$ endowed with the inner product $\scal{u \otimes u'}{v \otimes v'}_{{\cal G}_\eta} = \scal{u }{v}_{\hh_\eta}\scal{u'}{v'}_{\hh_\eta}$ for any $u,u',v,v' \in \hh_\eta$. Denote by $\vec$ the unitary map that maps the Hilbert-Schmidt operators on $\hh_\eta$ in vectors in ${\cal G}_\eta$. In particular, for any $u,v \in \hh_\eta$, we have $\vec(uv^\top) = v \otimes u$, moreover for any $\mm,\mm' :\hh_\eta \to \hh_\eta$ with finite Hilbert-Schmidt norm
\eqal{
\scal{\vec(\mm)}{\vec(\mm')}_{{\cal G}_\eta} = \tr(\mm^* \mm'), \quad \scal{\vec(\mm)}{v \otimes u}_{{\cal G}_\eta} = v^\top \mm u.
}
Now denote by $\psi_\eta$ the feature map $\psi_\eta(x) = \phi_\eta(x) \otimes \phi_\eta(x)$ for any $x \in \R^d$. 
We define the operator $Q \in \psd({\cal G}_\eta)$ and the vectors $\hat{v}, v \in {\cal G}_\eta$ as follows
\eqal{
Q = \int_\X \psi_\eta(x)\psi_\eta(x)^\top dx, \quad \hat{v} = \frac{1}{n} \sum_{i=1}^n \psi_\eta(x_i), \quad v = \int_\X \psi_\eta(x) p(x) dx.
}
Define the operator $S:{\cal G}_\eta \to L^2(\X)$ as
\eqal{
S f &= \scal{\psi_\eta(\cdot)}{f}_{{\cal G}_\eta} \in L^2(\X), \quad \forall f \in {\cal G}_\eta\\
S^* \alpha &= \int \alpha(x) \psi_\eta(x) dx \in {\cal G}_\eta, \quad \forall \alpha \in L^2(\X).
}
Note, in particular, that $Q$ and $v$ are characterized by
\eqal{
Q  = S^*S, \quad v = S^*p.
}
Given $\tilde{x}_1,\dots, \tilde{x}_m \in \R^d$ define the operator $\tilde{Z}:\hh_\eta \to \R^m$ as  $\tilde{Z} = (\phi_\eta(\tilde{x}_1)^\top, \dots, \phi_\eta(\tilde{x}_m)^\top)$, in particular we have
\eqal{\label{eq:tildeZ-definition}
\begin{split}
\tilde{Z}u &= (\phi_\eta(\tilde{x}_1)^\top u, \dots, \phi_\eta(\tilde{x}_m)^\top u), \forall u \in \hh_\eta\\
\tilde{Z}^*\alpha &= \sum_{i=1}^m \phi_\eta(\tilde{x}_i) \alpha_i, \forall \alpha \in \R^m,
\end{split}
}
In particular, note that for any $A \in \R^{m \times m}$ 
\eqal{\label{eq:ZAZ}
\tilde{Z}^*A\tilde{Z} = \sum_{i,j=1}^m A_{i,j} \phi_\eta(\tilde{x}_i) \phi_\eta(\tilde{x}_j)^\top, \quad
\tilde{Z}\tilde{Z}^* = K_{\tilde{X},\tilde{X}, \eta}.
}
Given $\tilde{Z}:\hh_\eta \to \R^m$, define the associated projection operator $\tilde{P}: \hh_\eta \to \hh_\eta$ on the range of the adjoint $\tilde{Z}^*$. In particular, note that 
\eqal{\label{eq:tildeP-definition}
\tilde{P} = \tilde{Z}^* K_{\tilde{X},\tilde{X},\eta}^{-1} \tilde{Z}
}
indeed, since $\tilde{Z}\tilde{Z}^* = K_{\tilde{X},\tilde{X},\eta}$ and it is invertible for any $\eta \in \R^d_{++}$, then 
\eqals{
\tilde{P}^2 = \tilde{Z}^* K_{\tilde{X},\tilde{X},\eta}^{-1} \tilde{Z}\tilde{Z}^* K_{\tilde{X},\tilde{X},\eta}^{-1} \tilde{Z}^* = \tilde{Z}^* K_{\tilde{X},\tilde{X},\eta}^{-1} K_{\tilde{X},\tilde{X},\eta} K_{\tilde{X},\tilde{X},\eta}^{-1} \tilde{Z} = \tilde{Z}^* K_{\tilde{X},\tilde{X},\eta}^{-1}\tilde{Z} = \tilde{P}.
}
and
\eqals{
\tilde{P}\tilde{Z}^* =  \tilde{Z}^* K_{\tilde{X},\tilde{X},\eta}^{-1}\tilde{Z}\tilde{Z}^* = \tilde{Z}^* K_{\tilde{X},\tilde{X},\eta}^{-1} K_{\tilde{X},\tilde{X},\eta} = \tilde{Z}^*.
}
and analogously $\tilde{Z}\tilde{P} = \tilde{Z}$. This implies also that $\tilde{P}\phi_\eta(\tilde{x}_i) = \phi_\eta(\tilde{x}_i)$ for any $i=1,\dots,m$.

\section{Compression of a PSD model}
\label{app:compression}
Let $\X \subset \R^d$ be an open set with Lipschitz boundary, contained in the hypercube $[-R, R]^d$ with $R > 0$. Given  $\tilde{x}_1,\dots,\tilde{x}_m \in \X$ be $m$ points in $[-R, R]^d$. Define the base point matrix $\tilde{X} \in \R^{m\times d}$ to be the matrix whose $j$-th row is the point $\tilde{x}_j$. The following result holds. We introduce the so called {\em fill distance} \cite{wendland2004scattered}
\eqal{\label{eq:fill-distance}
h = \max_{x \in [-R, R]^d} \min_{z \in \tilde{X}} \|x-z\|,
}
In the next lemma we specialize Theorem 4.5 of \cite{rieger2010sampling}, to obtain explicit constants in terms of $R$ and of our $\eta$. In particular, we identify the scale parameter $\sigma = \min(R, 1/\sqrt{\max_i \eta_i})$. This is interesting since it shows the effect of the precision $\eta$ of the kernel (if it was a Gaussian probability, it variance would scale exactly as $1/\sqrt{\eta}$).
\begin{lemma}[Norm of functions with scattered zeros]\label{lm:scattered-zeros}
Let $T = (-R,R)^d$ and $\eta \in \R^d_{++}$. Let $u \in \hh_\eta$ satisfying $u(\tilde{x}_1) = \dots = u(\tilde{x}_m) = 0$. There exists three constants $c, C, C'$ depending only on $d$ (and in particular, independent from $R, \eta, u, \tilde{x}_1, \dots, \tilde{x}_m$), such that, when $h \leq \sigma/{C'}$, then,
\eqal{
\|u\|_{L^\infty(T)} ~~\leq~~ C q_\eta ~ e^{- \frac{c\,\sigma}{h} \log \frac{c\,\sigma}{h}} ~\|u\|_{\hh_\eta},
}
with $q_\eta = \det(\frac{1}{\eta_+}\diag(\eta))^{-1/4}$ and $\sigma = \min(R, \frac{1}{\sqrt{\eta_+}})$ and $\eta_+ = \max_{i=1,\dots, d} \eta_i$.
\end{lemma}
\begin{proof}
By Theorem 4.3 of \cite{rieger2010sampling} there exists two constants $B_d, B'_d$ depending only on $d$ (and independent from $R, \tilde{x}_j$), such that for any $k \in \N, k > d/2 + 1$ and $u \in W^k_2(T)$ satisfying $u(\tilde{x}_1) = \dots = u(\tilde{x}_m) = 0$, the following holds,
\eqal{\label{eq:scattered-sobolev}
\|u\|_{L^\infty(T)} ~~\leq~~ \tfrac{B_{d}^k k^{k-d/2}}{k!}~h^{k-d/2}~|u|_{W^k_2(T)},
}
when $k h \leq R/B'_d$. Here the seminorm $|u|_{W^k_2(T)}$, by using the multinomial notation recalled in \cref{sect:notation}, corresponds to $|u|^2_{W^k_2(T)} = \sum_{|\alpha| = k} \|D^\alpha u\|^2_{L^2(T)}$. Our goal is to apply the result above to $f \in \hh_\eta$. First we recall that $\hh_\eta \subset W^k_2(\R^d)$ for any $k$ \cite{rieger2010sampling}. Then, since $T \subset \R^d$, we have
$|f|^2_{W^k_2(T)} = \sum_{|\alpha| = k} \|D^\alpha f\|^2_{L^2(T)} \leq \sum_{|\alpha| = k} \|D^\alpha f\|^2_{L^2(\R^d)} = |f|^2_{W^k_2(\R^d)}$. Then, using \cref{eq:seminorm-sobolev} we have $|f|^2_{W^k_2(\R^d)} \leq (2\pi)^{2k} \int \|\omega\|^{2k} |{\cal F}[f](\omega)|^2 d\omega$. Now, denote by $D$ the matrix $D = \diag(\eta)$ and $c_\eta = \pi^{d/2} \det(D)^{-1/2}$. By the characterization of the norm $\|\cdot\|_{\hh_\eta}$ in terms of the Fourier transform reported in \cref{ex:gaussian-rkhs}, we have
\eqals{
\|f\|^2_{W^k_2(\R^d)} &\leq~ (2\pi)^{2k}\int \|\omega\|^{2k} |{\cal F}[f](\omega)|^2  d\omega \\
& =~ \int \|\omega\|^{2k} c_\eta(2\pi)^{2k}  e^{-\pi^2 \|D^{-1/2} \omega\|^2} \, \frac{1}{c_\eta}e^{\pi^2 \|D^{-1/2} \omega\|^2}|{\cal F}[f](\omega)|^2 d\omega\\
& \leq~ c_\eta (2\pi)^{2k} \sup_{t \in \R^d} \|t\|^{2k} e^{-\pi^2 \|D^{-1/2} t\|^2} ~ \int \frac{1}{c_\eta} e^{\pi^2 \|D^{-1/2} \omega\|^2} |{\cal F}[f](\omega)|^2   d\omega\\
& = \|f\|^2_{\hh_\eta}  c_\eta (2\pi)^{2k} \sup_{t \in \R^d} \|t\|^{2k} e^{-\pi^2 \|D^{-1/2} t\|^2}.
}
Now, since $\sup_{z \in \R^d} \|z\|^{2k} e^{-\|z\|^2} = \sup_{r\geq 0} r^{2k} e^{-r^2} = k^k e^{-k} \leq k!$ and $\|D\| = \max_{i} \eta_i = \eta_+$,
\eqals{
\sup_{t \in \R^d} \|t\|^{2k} e^{-\pi^2 \|D^{-\frac{1}{2}} t\|^2} = \sup_{z \in \R^d} \|\tfrac{1}{\pi}D^{\frac{1}{2}}z\|^{2k} e^{-\|z\|^2} & \leq \tfrac{\|D\|^k}{\pi^{2k}} \sup_{z \in \R^d} \|z\|^{2k} e^{-\|z\|^2} = \tfrac{\eta_+^k k!}{\pi^{2k}}.
}
Then,
\eqal{\label{eq:Wk-norm-Heta}
\|f\|_{W^k_2(\R^d)} ~\leq~ \|f\|_{\hh_\eta} \,c_\eta^{1/2}(4\eta_+)^{k/2}\,\sqrt{k!} .
}
By plugging the bound above in \cref{eq:scattered-sobolev}, we obtain that when $k, h_{\tilde{X},R}$ satisfy $k h_{\tilde{X},R} \leq R B'_d$ then, 
\eqals{
\|f\|_{L^\infty(T)} &~\leq~ C \frac{(C_2 k h)^{k-d/2}}{\sqrt{k!}}~ \|f\|_{\hh_\eta},
}
where $C_2 = 2\sqrt{\eta_+} B_d$, and $C = (4B^2_d \pi)^{d/4} \det(\eta/\eta_+)^{-1/4}$.
Let now $C_3 = 1/\max(B'_d, 2 B_d)$. Assume that $h \leq \frac{C_3}{2+d}\min(R, 1/\sqrt{\eta_+})$ and set $k = \lfloor s \rfloor$ and $s = \frac{C_3}{h}\min(R, 1/\sqrt{\eta_+})$. 
Note first, that with this choice of $h$ and $k$ we satisfy $h k \leq R/B'_d$ so we can apply Theorem 4.3 of \cite{rieger2010sampling}. Moreover, by construction $d/2 + 1 \leq \frac{s}{2} \leq s-1 \leq \lfloor s \rfloor \leq s$. Then $C_2 \lfloor s \rfloor h \leq 1$ and $\lfloor s \rfloor-d/2 \geq 0$, so $(C_2 h \lfloor s \rfloor)^{\lfloor s \rfloor - d/2} \leq 1$. Moreover $\frac{1}{\sqrt{k!}} \leq e^{-\frac{k}{2}\log \frac{k}{2}}$ so we have
\eqals{
\|f\|_{L^\infty(T)} &~\leq~ C (C_2 h \lfloor s \rfloor)^{\lfloor s \rfloor - d/2} e^{- \frac{1}{2}\lfloor s \rfloor \log \frac{1}{2}\lfloor s \rfloor} ~ \|f\|_{\hh_\eta}\\
& ~\leq~ C e^{- \frac{s-1}{2} \log \frac{s-1}{2}} ~ \|f\|_{\hh_\eta}  ~\leq~ C e^{- \frac{s}{4} \log \frac{s}{4}} ~ \|f\|_{\hh_\eta}.
}
The final result is obtained by writing $s/4 = c \sigma /h$ with $\sigma = \min(R, 1/\sqrt{\eta_+})$ and $c = C_3/4$ and by writing the assumption on $h$ as $h \leq \sigma/C'$ with $C' = (2+d)/C_3$.
\end{proof}

\begin{lemma}[Lemma 3, page 28 \cite{pagliana2020interpolation}]\label{lm:sup-phi-to-Linfty}
Let $\X \subset \R^d$ with non-zero volume. Let $\hh$ be a reproducing kernel Hilbert space on $\X$, associated to a continuous uniformly bounded feature map $\phi:\X \to \hh$. Let $A:\hh \to \hh$ be a bounded linear operator. Then, 
\eqals{
\sup_{x \in \X} \|A\phi(x)\|_\hh \leq \sup_{\|f\|_{\hh}\leq 1} \|A^*f\|_{C(\X)}.
}
In particular, if $\X \subset \R^d$ is a non-empty open set, then $\sup_{x \in \X} \|A\phi(x)\|_\hh \leq \sup_{\|f\|_{\hh} \leq 1} \|A^*f\|_{L^\infty(\X)}$.
\end{lemma}
\begin{proof}
We recall the variational characterization of the norm $\|\cdot\|_\hh$ in terms of the inner product $\scal{\cdot}{\cdot}_\hh$ as $\|v\|_\hh = \sup_{\|f\| \leq 1} \scal{f}{v}_\hh$. We have the following
\eqals{
\sup_{x \in \X} \|A\phi(x)\|_\hh &= \sup_{x \in \X, \|f\|\leq 1} \scal{f}{A\phi(x)}_\hh \leq \sup_{x \in \X, \|f\|\leq 1} |\scal{A^*f}{\phi(x)}_\hh| \\
& = \sup_{\|f\|\leq 1} \sup_{x \in \X} |(A^*f)(x)| = \sup_{\|f\|\leq 1} \|A^*f\|_{C(\X)}.
}
Finally, note that when $\X \subset \R^d$ is a non-empty open set $\|A^*f\|_{C(\X)} = \|A^*f\|_{L^\infty(\X)}$,
since $A^*f \in \hh$ and all the functions in $\hh$ are continuous and bounded due to the continuity of $\phi$.
\end{proof}

\begin{theorem}[Approximation properties of the projection]\label{thm:(I-P)phi}
Let $R > 0, \eta \in \R^d_{++}, m \in \N$. Let $\X \subseteq T = (-R, R)^d$ be a non-empty open set and let $\tilde{x}_1,\dots,\tilde{x}_m$ be a set of distinct points. Let $h > 0$ be the fill distance associated to the points w.r.t $T$ (defined in \cref{eq:fill-distance}). Let $\tilde{P}:\hh_\eta \to \hh_\eta$ be the associated projection operator (see definition in \cref{eq:tildeP-definition}). There exists three constants $c, C, C'$, such that, when $h \leq \sigma/{C'}$,
\eqals{
\sup_{x \in \X} \|(I-\tilde{P})\phi_\eta(x)\|_{\hh_\eta} \leq C q_\eta ~ e^{- \frac{c\,\sigma}{h} \log \frac{c\,\sigma}{h}}.
}
Here $q_\eta = \det(\frac{1}{\eta_+}\diag(\eta))^{-1/4}$ and $\sigma = \min(R, \frac{1}{\sqrt{\eta_+}})$, $\eta_+ = \max_i \eta_i$. The constants $c, C', C''$ depend only on $d$ and, in particular, are independent from $R, \eta,\tilde{x}_1, \dots, \tilde{x}_m$.
\end{theorem}
\begin{proof}
We first recall some basic properties of the projection operator $\tilde{P}:\hh_\eta \to \hh_\eta$ on the span of $\phi_\eta(\tilde{x}_1),\dots,\phi_\eta(\tilde{x}_m)$, defined in \cref{eq:tildeP-definition}. By construction $\tilde{P}\phi_\eta(\tilde{x}_i) = \phi_\eta(\tilde{x}_i)$ is of rank $m$ for any $i=1,\dots,m$. Now note that for any $f \in \hh_\eta$, the function $(\tilde{P}f)(\tilde{x}_i) = f(\tilde{x}_i)$, indeed, by the reproducing property of $\hh_\eta$
\eqals{
(\tilde{P}f)(\tilde{x}_i)  = \scal{\tilde{P}f}{\phi_\eta(\tilde{x}_i)}_{\hh_\eta} = \scal{f}{\tilde{P}\phi_\eta(\tilde{x}_i)}_{\hh_\eta} = \scal{f}{\phi_\eta(\tilde{x}_i)}_{\hh_\eta} = f(\tilde{x}_i).
}
Then $(f - \tilde{P}f)(\tilde{x}_i) = 0$ for any $i=1,\dots,m$. 
By \cref{lm:scattered-zeros}, we know that there exist three constants $c, C, C'$ depending only on $d$ such that when $h \leq \sigma/C'$ we have that the following holds
$\|f\|_{L^\infty(T)} \leq  C q_\eta ~ e^{- \frac{c\,\sigma}{h} \log \frac{c\,\sigma}{h}}$,
for any $f \in \hh_\eta$ such that $f(\tilde{x}_1) = \dots = f(\tilde{x}_m) = 0$. Since, for any $f \in \hh_\eta$, we have that $f - \tilde{P}f$ belongs to $\hh_\eta$ and satisfies such property, we can apply \cref{lm:scattered-zeros} with $u = (I-\tilde{P})f$, obtaining, under the same assumption on $h$,
\eqals{
\|(I - \tilde{P})f\|_{L^\infty(T)} \leq  C q_\eta ~ e^{- \frac{c\,\sigma}{h} \log \frac{c\,\sigma}{h}} \|f\|_{\hh_\eta}, \quad \forall f \in \hh_\eta,
}
where we used the fact that $\|(I - \tilde{P})f\|_{\hh_\eta} \leq \|I - \tilde{P}\|\|f\|_{\hh_\eta}$ and $\|I - \tilde{P}\| \leq 1$, since $P$ is a projection operator and so also $I-P$ satisfies this property. The final result is obtained by applying \cref{lm:sup-phi-to-Linfty} with $A = I - \tilde{P}$, from which we have
\eqal{\sup_{x \in \X} \|(I-\tilde{P})\phi(x)\|_\hh \leq \sup_{\|f\|\leq 1} \|(I-\tilde{P})f\|_{L^\infty(T)} &\leq \sup_{\|f\|\leq 1} C q_\eta ~ e^{- \frac{c\,\sigma}{h} \log \frac{c\,\sigma}{h}} \|f\|_{\hh_\eta}\\
& = C q_\eta ~ e^{- \frac{c\,\sigma}{h} \log \frac{c\,\sigma}{h}}.
}
\end{proof}

\begin{theorem}[Compression of a PSD model]\label{thm:compr-fill-dist}
Let $\eta \in \R^d_{++}$ and let $\mm \in \psd(\hh_\eta)$. Let $\X$ be an open bounded subset with Lipschitz boundary of the cube $[-R,R]^d$, $R > 0$. Let $\tilde{x}_1,\dots,\tilde{x}_m \in \X$ and $\tilde{X}$ be the base point matrix whose $j$-rows are the points $\tilde{x}_j$ with $j=1,\dots,m$. Consider the model $p = \pp{\cdot}{\mm, \phi_\eta}$ and the the compressed model $\tilde{p} = \pp{\cdot}{A_m, \tilde{X}, \eta}$ with
\eqals{
A_m ~=~ K_{\tilde{X},\tilde{X},\eta}^{-1} \,\tilde{Z} \mm \tilde{Z}^*\, K_{\tilde{X},\tilde{X},\eta}^{-1},
}
where $\tilde{Z}:\hh_\eta \to \R^m$ is defined in \cref{eq:tildeZ-definition} in terms of $\tilde{X}_m$. Let $h$ be the fill distance (defined in \cref{eq:fill-distance}) associated to the points $\tilde{x}_1,\dots, \tilde{x}_m$. The there exist three constants $c, C, C'$ depending only on $d$ such that, when $h \leq \sigma/C'$, with $\sigma = \min(R, 1/\sqrt{\eta_+}), \eta_+ = \max_{i=1,\dots,d} \eta_i, q_\eta = \det(\frac{1}{\eta_+}\diag(\eta))^{-1/4}$, then
\eqals{
|p(x) - \tilde{p}(x)| ~\leq~ 2C q_\eta \sqrt{\|\mm\| p(x)}\, e^{-\frac{c\,\sigma}{h}\log\frac{c\,\sigma}{h}}  \,+\, C^2 q_\eta^2\|\mm\|\,e^{-\frac{2c\,\sigma}{h}\log\frac{c\,\sigma}{h}}, \quad \forall x \in \X.
}
\end{theorem}
\begin{proof}
 Consider the projection operator $\tilde{P}:\hh_\eta \to \hh_\eta$ associated to the points $\tilde{x}_1,\dots,\tilde{x}_m$, defined in \cref{eq:tildeP-definition}. Note that the adjoint $\tilde{Z}^*$ has range equal to $\operatorname{span}\{\phi_{\eta}(\tilde{x}_1),\dots,\phi_{\eta}(\tilde{x}_1)\}$ and that, by construction, $\tilde{Z}\, \mm \tilde{Z}^* \in \psd^m$ and so $A_m \in \psd^m$. 

\paragraph{Step 1. Error induced by a projection}
By the reproducing property $k_{\eta}(x,x') = \phi_{\eta}(x)^\top \phi_{\eta}(x')$ (see \cref{ex:gaussian-rkhs} and \cref{eq:reproducing-property}) and the fact that $\tilde{P} = \tilde{Z}^*K_{\tilde{X},\tilde{X},\eta}^{-1} \tilde{Z}$, (see \cref{eq:tildeP-definition}), then, for any $x \in \R^d$ 
\eqals{
\pp{x}{A_m, \tilde{X}_m, \eta} &= \sum_{i,j=1}^m (A_m)_{i,j} k_{\eta}(x,\tilde{x}_i)k_{\eta}(x,\tilde{x}_j) \\ 
& = \phi_{\eta}(x)^\top\big(\sum_{i,j=1}^m (A_m)_{i,j} \phi_{\eta}(\tilde{x}_i)\phi_{\eta}(\tilde{x}_j)^\top\big)\phi_{\eta}(x) \\
& = \phi_{\eta}(x)^\top \tilde{Z}^*A_m \tilde{Z}\phi_{\eta}(x) \\
& = \phi_{\eta}(x)^\top\, \tilde{Z}^*K_{\tilde{X},\tilde{X},\eta}^{-1} \tilde{Z} \, \mm \, \tilde{Z}^* K_{\tilde{X},\tilde{X},\eta}^{-1} \tilde{Z}\, \phi_{\eta}(x) \\
& = \phi_{\eta}(x)^\top\, \tilde{P}\mm \tilde{P} \phi_{\eta}(x)\\
& = \pp{x}{\tilde{P}\mm \tilde{P}, \phi_{\eta}}.
}
This implies that, for all $x \in \R^d$ the following holds
\eqals{
\pp{x}{A_m, \tilde{X}_m, \eta} - \pp{x}{\mm, \phi_{\eta}} &= \pp{x}{\tilde{P}\mm\tilde{P}, \phi_{\eta}} - \pp{x}{\mm, \phi_{\eta}} \\
& = \phi_{\eta}(x)^\top(\tilde{P}\mm\tilde{P} - \mm)\phi_{\eta}(x).
}

\paragraph{Step 2. Bounding $|\phi_{\eta}(x)^\top(\tilde{P}\mm\tilde{P} - \mm)\phi_{\eta}(x)|$}
Now, consider that 
\eqals{
\tilde{P}\mm\tilde{P} - \mm = (I-\tilde{P})\mm(I-\tilde{P}) - \mm(I-\tilde{P}) - (I-\tilde{P})\mm.
}
Since $|a^\top A B A a| \leq \|Aa\|^2_\hh \|B\|$ and $|a^\top A B a| \leq \|A a\|_{\hh} \|B^{1/2}\| \|B^{1/2} a\|_{\hh}$, for any $a$ in a Hilbert space $\hh$ and for $A, B$ bounded symmetric linear operators with $B \in \psd(\hh)$, by bounding the terms of the equation above, we have for any $x \in \R^d$,
\eqals{
|\phi_{\eta}(x)^\top(\tilde{P}\mm\tilde{P}  - \mm)\phi_{\eta}(x)| & \leq 2\|(I-\tilde{P})\phi_\eta(x)\|_{\hh_\eta}\|\mm\|^{1/2}\|\mm^{1/2}\phi_\eta(x)\|_{\hh_\eta} \\
& \qquad \qquad + \|(I-\tilde{P})\phi_{\eta}(x)\|^2_{\hh_\eta}\|\mm\| \\
& = 2c_\mm^{1/2}\pp{x}{\mm, \phi_\eta}^{1/2}  \tilde{u}(x)  + c_\mm \tilde{u}(x)^2, 
}
where $c_\mm = \|\mm\|$ and we denoted by $\tilde{u}(x)$ the quantity $\tilde{u}(x) = \|(I-\tilde{P})\phi_{\eta}(x)\|_{\hh_\eta}$ and we noted that $\|\mm^{1/2}\phi_\eta(x)\|^2_{\hh_\eta} = \phi_\eta(x)^\top \mm\phi_\eta(x) = \pp{x}{\mm, \phi_\eta}$. 

\paragraph{Step 3. Bounding $\tilde{u}$}
Now, by \cref{thm:(I-P)phi} we have that when the fill distance $h$ (defined in \cref{eq:fill-distance}) satisfies $h \leq \sigma/C'$ with $\sigma = \min(R,1/\sqrt{\tau})$, then
\eqals{
\|\tilde{u}\|_{L^\infty(\X)} = \sup_{x \in \X} \|(I-\tilde{P})\phi_{\eta}(x)\|_{\hh_{\eta}} \leq C e^{-\frac{c\,\sigma}{h}\log\frac{c\,\sigma}{h}}.
}
with $c, C, C'$ depending only on $d$. 
\end{proof}

\subsection{Proof of \cref{thm:compression}}\label{proof:thm:compression}

\cref{thm:compression} is a corollary of the next theorem, considering that $1/\sigma \leq (1+\eta_+)^{1/2}$ and moreover $\det(\frac{1}{\eta_+} \diag(\eta)) = \prod_{j=1}^d \eta_i/\eta_+ \leq 1$, $\tr(A K_{X,X,\eta}) \leq \|A\|\tr(K_{X,X,\eta})$ since both $A, K_{X,X,\eta} \in \psd^n$, and by construction $\tr(K_{X,X,\eta}) = n$, then $c_{A,\eta} \leq \|A\| n$. 

\begin{theorem}
Let $\eps \in (0, 1/e)$. Let $A \in \psd^n$, $X \in \R^{n\times d}, \eta \in \R^d_{++}$. Let $\tilde{x}_1, \dots, \tilde{x}_m$ be sampled independently and uniformly at random from $[-1,1]^d$. Let $\delta \in (0,1]$, $\eta_+ = \max(1, \max_{i=1,\dots,d} \eta_i)$. When $m$ satisfies $m \geq Q'\eta_+^{d/2} (\log \frac{Q\|A\|n}{\eps})^d\,\log(\frac{Q''(1+\eta_+)}{\delta}\log \frac{Q \|A\|n}{\eps})$, then the following holds with probability at least $1-\delta$,
\eqals{
|p(x) - \tilde{p}(x)| ~\leq~ \eps^2 + \eps\sqrt{p(x)},\qquad\forall x \in [-1,1]^d,
}
Here the three constants $Q,Q',Q''$ depend only on $d$.
\end{theorem}
\begin{proof}
First let us rewrite $\pp{\cdot}{A,X,\eta}$ in the equivalent form $\pp{\cdot}{\mm, \phi_\eta}$ with $\mm \in \psd(\hh_\eta)$ defined as $\mm = \sum_{ij=1}^n A_{ij} \phi_\eta(x_i)\phi_\eta(x_j)$. In particular, by the cyclicity of the trace 
\eqal{
\tr(\mm) = \sum_{ij=1}^n A_{i,j} \phi_\eta(x_i)^\top \phi_\eta(x_j) = \sum_{ij=1}^n A_{i,j} k_\eta(x_i, x_j) = \tr(A K_{X,X,\eta}).
}
The proof of this theorem is an application of the approximation result in \cref{thm:compr-fill-dist} to the model $\pp{\cdot}{\mm, \phi_\eta}$ where we use as compression points, the points $\tilde{x}_1, \dots, \tilde{x}_m$ sampled independently and uniformly at random from $[-1,1]^d$.

The result of the theorem depends on the fill distance $h$, defined in \cref{eq:fill-distance}, and associated to the points $\tilde{x}_1, \dots, \tilde{x}_m$. Let $c, C, C'$ be the constants depending only on $d$ from \cref{thm:compr-fill-dist}. To apply \cref{thm:compr-fill-dist} we have to guarantee that $h \leq \sigma/C'$ with $\sigma = \min(1,1/\sqrt{\eta_+})$, in particular, choosing $h$ such that $h \leq \min(c, 1/C') \sigma/(e \log(2Cc_{A,\eta}/\eps))$ guarantees that $h  \leq \sigma/C'$ and, by applying the theorem, we have for all $x \in [-1,1]^d$
\eqals{
|p(x) - \tilde{p}(x)| &\leq 2C q_\eta \sqrt{\|\mm\| p(x)}\, e^{-\frac{c\,\sigma}{h}\log\frac{c\,\sigma}{h}}  \,+\, C q_\eta^2\|\mm\|\,e^{-\frac{2c\,\sigma}{h}\log\frac{c\,\sigma}{h}},\\
& \leq 2C c_{A,\eta}\sqrt{p(x)}\, e^{-\frac{c\,\sigma}{h}\log\frac{c\,\sigma}{h}}  \,+\, C^2 c_{A,\eta}^2\,e^{-\frac{2c\,\sigma}{h}\log\frac{c\,\sigma}{h}},
}
with $q_\eta = \det(\frac{1}{\eta_+}\diag(\eta))^{-1/4}$
where in the last step we used the fact that $\|\mm\| \leq \tr(\mm) = \tr(A K_{X,X,\eta})$ and so $\|\mm\| q_\eta^2 \leq \tr(A K_{X,X,\eta}) q_\eta^2 = c_{A,\eta}^2$. Note now, that by the choice we made for $h$, we have $\log(c \sigma/h) \geq 1$ and so that $e^{-\frac{c\sigma}{h} \log \frac{c\sigma}{h}} \leq \eps/(2Cc_{A,\eta})$. This implies 
\eqals{
|p(x) - \tilde{p}(x)| ~\leq~ \eps\sqrt{p(x)}  + \eps^2.
}
The final result is obtained by controlling the number of points $m$ such that $h$ satisfy the required bound in high probability. By, e.g. Lemma~12, page 19 of \cite{vacher2021dimension} and the fact that $[-1,1]^d$ is a convex set, we have that there exists two constants $C_1, C_2$ depending only on $d$ such that $h \leq C_1 m^{-1/d} (\log(C_2 m / \delta))^{1/d}$, with probability at least $1-\delta$. In particular $m$ satisfying
\eqal{\label{eq:choice-m}
m \geq \left(\frac{e C_1}{\min(c,1/C')} \frac{1}{\sigma}\log \frac{2C c_{A,\eps}}{\eps} \right)^d \log \frac{C_2 m}{\delta}
}
guarantees that $C_1 m^{-1/d} (\log \frac{C_2 m}{\delta})^{1/d} \leq \frac{\sigma \min(c, 1/C')}{e} \log \frac{2Cc_{A,\eta}}{\eps}$. Note that, given $A \geq e, B \geq e$, the inequality $m \geq B\log(A m)$ is satisfied by $m \geq 2B \log 2 A B$, indeed $\log A \geq \log \log A$ and $\log 2B \geq \log \log 2B$ and so, when $m = m_0 = 2B \log 2 A B$ we have
\eqal{ \label{eq:mlogm}
B \log (A m_0) &= B \log (2A B\log(2A B)) \\
& = B \log A + B\log 2B + B\log\log A + B\log\log 2B \\
&\leq 2B \log A + 2B \log 2B = 2B \log 2AB = m_0,
}
and moreover $m - B\log(Am)$ is increasing for $m \geq B$.
Then, to satisfy \cref{eq:choice-m} we choose $m \geq 2B \log (2 A B)$ with $B = \left(\frac{e C_1}{\min(c,1/C')}  \frac{1}{\sigma}\log \frac{2C c_{A,\eta}}{\eps} \right)^d$ and $A = \frac{C_2}{\delta}$, in particular
\eqals{
m = Q\left(\frac{1}{\sigma}\log \tfrac{Q' c_{A,\eta}}{\eps} \right)^d\,\log\left(\tfrac{Q''}{\delta \sigma}\log \tfrac{Q' c_{A,\eta}}{\eps}\right).
}
with $Q = 2d(\frac{e C_1}{\min(c,1/C')})^d$, with $Q' = 2C$, $Q'' = 2C_2 Q^d$.
\end{proof}

\section{Approximation of a probability via a PSD model}\label{app:approximation}

In this section we prove \cref{prop:when-asm-psd-holds} and \cref{thm:approximation}. 

\subsection{Proof of \cref{prop:when-asm-psd-holds}}\label{proof:prop:when-asm-psd-holds}

\begin{lemma}\label{lm:psd-p-strictly-positive}
Let $\beta \in \N$ and $\beta > 0$. Let $\X$ be an open bounded subset of $\R^d$ with Lipschitz boundary. Let $p$ be a strictly positive $\beta$-times differentiable function on $\overline{X}$, the closure of $\X$. Then there exist a function $\tilde{f}$ satisfying $p(x) = f(x)^2$ for all $x \in \X$ and such that $\tilde{f} \in W^\beta_q(\X) \cap L^\infty(\X)$ for all $q \in [1,\infty]$. 
\end{lemma}
\begin{proof}
Let $\tilde{p} \in C^\beta(\R^d)$ be an extension of $p$ to $\R^d$ (see Whitney extension theorem \cite{hormander1990analysis}), i.e. such that $\tilde{p}|_{\overline{X}} = p$.
Let $c = \min_{x \in \overline{\X}} p(x)$ and $C = \max_{x \in \overline{\X}} p(x)$, we have that $c > 0$ since $\X$ is compact and $p$ is continuous.
Note that $g(z) = \sqrt{z}$ is $C^\infty$ on the open interval $(0,+\infty)$. Let $u \in C^\infty(\R)$ be a bump function such that $u(x) \in [0,1]$ for any $x \in \R$, moreover it is identically $0$ on $J = (\infty, c/2] \cup [2C, \infty)$ and identically $1$ on the interval $I = [c,C]$. Then the function $h(z) = u(z) g(z)$ is identically $0$ on $J$, moreover $h(z) = \sqrt{z}$ on $I$ and $h \in C^\infty(\R)$ since $h = 0$ on $J$ and both $u,g \in C^\infty([c/2,2C])$. 
Now, denote by $f$ the function $f(x) = h(\tilde{p}(x))$ for all $x \in \R^d$. Since $p(x) \in I$ and $\tilde{p}(x) = p(x)$ for any $x \in \X$, we have $f(x) = h(\tilde{p}(x)) = h(p(x)) = \sqrt{p(x)}$ for any $x \in \X$. Moreover $h \in C^{\beta}(\R^d)$ since it is the composition of a $C^\beta(\R^d)$ function with a $C^\infty(\R)$ function. The proof is concluded by taking $\tilde{f}$ to be the restriction of $f$ to $\X$ and observing that it belongs to $W^\beta_q(\X) \cap L^\infty(\X)$, for all $q \in [1,\infty]$, as derived in \cref{prop:inclusion-Cb-in-Wb-Linfy}.
\end{proof}

\begin{lemma}[\cite{rudi2020finding} Corollary 2, page 23]\label{lm:psd-p-with-zeros}
Let $\X$ be an open bounded subset of $\R^d$ with Lipschitz boundary. Let $p$ be a probability density that is $\beta+2$-times differentiable on the closure of $\X$, with $\beta > 0$. Assume that the zeros of $p$ are isolated points with strictly positive Hessian and their number is finite. Moreover assume that there are no zeros of $p$ on the boundary.
Then there exist $q \in \N$ and $q$ functions $f_1 \dots f_q$ such that  $f_1 \dots f_q \in W_r^{\beta}(\X) \cap L^\infty(\X)$, for any $r \in [1,\infty]$ and satisfying
\eqals{
p(x) = \sum_{i=1}^q f_i(x)^2, \quad \forall x \in \X.
}
\end{lemma}
\begin{proof}
Let $\tilde{p}$ be the $\beta+2$-times differentiable extension to $\R^d$ of $p$ (via the Withney extension theorem \cite{hormander1990analysis}), i.e. $\tilde p = p$ on the closure of $\X$. We apply \cite{rudi2020finding} Corollary 2, page 23 on $\tilde{p}$, obtaining $q$ functions $f_1, \dots, f_q \in C^{\beta}(\R^d)$ such that $\tilde{p}(x) = \sum_i f_i(x)^2$ for all $x \in \X$. The result is obtained by applying \cref{prop:inclusion-Cb-in-Wb-Linfy} on the restrictions $f_1, \dots, f_q$ on $\X$.
\end{proof}

Now we are ready to prove \cref{prop:when-asm-psd-holds}. We restate here fore convenience. 

\PWhenAsmPSDHolds*

\begin{proof}
Let $\X = (-1,1)^d$. The case (a) is proven in \cref{lm:psd-p-strictly-positive}. For the case (b), let $\tilde{v} \in W^\beta_2(\R^d) \cap L^\infty(\R^d)$ be the extension of $v$ to $\R^d$ (see \cref{cor:extension-intersection}), by Theorem 1, page 8, in \cite{sickel1996composition} the function $e^{\tilde{v}} - 1 \in  W^\beta_2(\R^d) \cap L^\infty(\R^d)$ since $\exp(\cdot) - 1$ is analytic and $0$ in $0$. Let $q = (e^{\tilde{v}} - 1)|_\X$,  $q \in W^\beta_2(\X) \cap L^\infty(\X)$ and so also $g = q+1$ belongs to $ W^\beta_2(\X) \cap L^\infty(\X)$, since $\X$ is a bounded set. Finally note that $g = e^v$ on $\X$ and $\min_{x \in \X} g(x) = \min_{x \in \X} e^{-v(x)} \geq e^{-\|v\|_{L^\infty(\X)}} > 0$, so it satisfies the point (a). The point (c) is a consequence of (b) indeed if $p = \sum_{i=1}^t \alpha_i e^{-v_i}$ and each $v_i$ satisfies (b), then $e^{-v_i} = \sum_{j=1}^{q_i} f_{i,j}^2$ with $f_{i,j} \in W^\beta_2(\R^d) \cap L^\infty(\R^d)$, so $p = \sum_{i=1}^t \sum_{j=1}^{q_i} g_{i,j}^2$, with $g_{i,j} = \sqrt{\alpha_i} f_{i,j} \in W^\beta_2(\R^d) \cap L^\infty(\R^d)$. Finally, (d) is proven in \cref{lm:psd-p-with-zeros}.
\end{proof}

\subsection{Additional results required to prove \cref{thm:approximation}}

We now focus on proving the result in \cref{thm:approximation}. To this end, we first prove some preliminary result that will be useful in the following.

Let $S \subseteq \R^d$. We recall the definition of the function ${\bf 1}_S$, that is ${\bf 1}_S(\omega) = 1$ for any $\omega \in S$ and ${\bf 1}_S(\omega) = 0$ for any $\omega \notin S$. Define moreover, 
\eqal{\label{eq:good-conv}
g(x) = \tfrac{2^{-d/2}}{V_d} \|x\|^{-d} J_{d/2}(2\pi\|x\|)J_{d/2}(4\pi\|x\|)
}
where $J_{d/2}$ is the Bessel function of the first kind of order $d/2$ and $V_d = \int_{\|x\| \leq 1} dx = \frac{\pi^{d/2}}{\Gamma(d/2+1)}$. 
\begin{lemma}\label{lm:good-conv}
The function $g$ defined above satisfies $g \in L^1(\R^d) \cap L^2(\R^d)$ and $\int g(x) dx = 1$. Moreover, for any $\omega \in \R^d$, we have
\eqals{
{\bf 1}_{\{\|\omega\| < 1\}}(\omega) \leq  {\cal F}[g](\omega) \leq {\bf 1}_{\{\|\omega\| \leq 3\}}(\omega).
}
\end{lemma}
\begin{proof}
In this proof we will use the notation about the convolution and the Fourier transform in \cref{prop:fourier}.
Define $b(x) = \|x\|^{-d/2} J_{d/2}(2\pi\|x\|)$ where $J_{d/2}$ is the Bessel function of the first kind of order $d/2$. Note that $b \in L^2(\R^d) \cap L^1(\R^d)$, since there exists a constant $c > 0$ $|J_{d/2}(z)| \leq c \min(z^{d/2},z^{-1/2})$ for any $z \geq 0$ \cite{stein1971introduction}. Moreover note that the Fourier transform of $b$ is ${\cal F}[b](\omega) = {\bf 1}_{\{\|\omega\| < 1\}}$ (see \cite{stein1971introduction}, Thm.~4.15, page 171). Define now $g(x) = \frac{1}{V_d} b(x) b(2x) = \frac{1}{V_d} \|x\|^{-d} J_{d/2}(2\pi\|x\|)J_{d/2}(4\pi\|x\|)$. Note that $g \in L^1(\R^d)$ since 
\eqals{
\|g\|_{L^1(\R^d)} = \|b(\cdot) b(2\cdot)\|_{L^1(\R^d)} \leq \|b(\cdot)\|_{L^2(\R^d)}\|b(2\cdot)\|_{L^2(\R^d)} < \infty,
}
and analogously
\eqals{
\|g\|_{L^2(\R^d)} = \|b(\cdot) b(2\cdot)\|_{L^2(\R^d)} \leq \|b(\cdot)\|_{L^\infty(\R^d)}\|b(2\cdot)\|_{L^2(\R^d)} < \infty.
}
By the properties of the Fourier transform, we have
${\cal F}[g] = \frac{1}{V_d} {\cal F}[b] \star {\cal F}[b(2\cdot)] = \frac{1}{V_d} \int {\bf 1}_{\{\|z\| \leq 1\}}{\bf 1}_{\{\|\omega - z\| \leq 2\}} dz$. Note that for any $\omega \in \R^d$, since ${\bf 1}_{\{\|\omega - z\| \leq 2\}} \leq 1$,
\eqals{
{\cal F}[g](\omega) = \frac{1}{V_d} \int {\bf 1}_{\{\|z\| \leq 1\}}{\bf 1}_{\{\|\omega - z\| \leq 2\}} dz \leq \frac{1}{V_d} \int {\bf 1}_{\{\|z\| \leq 1\}} dz \leq 1.
}
Now, note that when $\|\omega\|,\|z\| \leq 1$, then $\|\omega - z\| \leq \|\omega\|+\|z\| \leq 2$. So we have 
$${\bf 1}_{\{\|\omega\| \leq 1\}}{\bf 1}_{\{\|z\| \leq 1\}}{\bf 1}_{\{\|\omega - z\| \leq 2\}} = {\bf 1}_{\{\|z\| \leq 1\}}{\bf 1}_{\{\|\omega\| \leq 1\}}.$$
Then
\eqals{
{\bf 1}_{\{\|\omega\| \leq 1\}}{\cal F}[g](\omega) &= \frac{1}{V_d} \int {\bf 1}_{\{\|\omega\| \leq 1\}}{\bf 1}_{\{\|z\| \leq 1\}}{\bf 1}_{\{\|\omega - z\| \leq 2\}} dz \\
& = \frac{1}{V_d} {\bf 1}_{\{\|\omega\| \leq 1\}} \int {\bf 1}_{\{\|z\| \leq 1\}} dz = {\bf 1}_{\{\|\omega\| \leq 1\}}.
}
Moreover note that for all $\|\omega\| > 3, \|z\| \leq 1$ we have
$\|\omega - z\| \geq |\|\omega\| - \|z\|| > 2$, then ${\bf 1}_{\{\|\omega\| > 3\}}{\bf 1}_{\{\|z\| \leq 1\}}{\bf 1}_{\{\|\omega - z\| \leq 2\}} = 0$. So for any $\|\omega\| > 3$ 
\eqals{
{\bf 1}_{\{\|\omega\| > 3\}}{\cal F}[g](\omega) = \frac{1}{V_d} \int {\bf 1}_{\{\|\omega\| > 3\}}{\bf 1}_{\{\|z\| \leq 1\}}{\bf 1}_{\{\|\omega - z\| \leq 2\}} dz = 0.
}
To conclude $\int g(x) dx = \int g(x) e^{-2\pi i w^\top 0} dx = {\cal F}[g](0) = 1$.
\end{proof}

\begin{theorem}\label{thm:Meps}
Let $\beta > 0, q \in \N$. 
Let $f_1,\dots,f_q \in W^\beta_2(\R^d) \cap L^\infty(\R^d)$ and denote by $p$ the function $p = \sum_{i=1}^q f_i^2$. Let $\eps \in (0,1]$ and let $\eta \in \R^d_{++}$. Denote by $\eta_0 = \min_{j=1,\dots,d} \eta_j$. Let $\phi_\eta$ be the feature map of the Gaussian kernel with bandwidth $\eta$ and let $\hh_\eta$ be the associated RKHS. Then there exists $\mm_\eps \in \psd(\hh_\eta)$ with $\operatorname{rank}(\mm_\eps) \leq q$, such that 
\eqals{
\|\pp{\cdot}{\mm_\eps,\phi_\eta} - p(\cdot)\|_{L^r(\R^d)} &\leq \eps, \qquad \tr(\mm_\eps) \leq  C |\eta|^{1/2}(1 + \eps^2\exp(\tfrac{C'}{\eta_0} \eps^{-\frac{2}{\beta}})),
}
for all $r \in [1,2]$, where $|\eta| = \det(\diag(\eta))$ and $C, C'$ depend only on $\beta, d, \|f_i\|_{W^\beta_2(\R^d)}, \|f_i\|_{L^\infty(\R^d)}$.
\end{theorem}
\begin{proof}
Let $t > 0$ (to be set later) and let $g$ be defined according to \cref{eq:good-conv}. Define $g_t(x) = t^{-d}g(x/t)$. Given the properties of $g$ in \cref{lm:good-conv}, we have that $\int g_t(x) dx = 1$, $g_t \in L^1(\R^d) \cap L^2(\R^d)$, that ${\cal F}[g_t](\omega) = {\cal F}[g](t \omega)$ and so that $|{\cal F}[g_t](\omega)| = |{\cal F}[g](t\omega)| \leq {\bf 1}_{\{t\|\omega\| \leq 3\}}(\omega)$. Moreover we have that $|1 - {\cal F}[g_t](\omega)| = |1 - {\cal F}[g](t\omega)| \leq {\bf 1}_{\{t\|\omega\| \geq 1\}}(\omega)$.

Now, note that $\int (1+\|\omega\|^2)^{\beta} |{\cal F}[f](\omega)|^2 d\omega \leq 2^{2\beta}\|f\|^2_{W^\beta_2(\R^d)}$, as discussed in \cref{prop:sobolev}. 

\paragraph{Step 1. Bounding $\|f - f \star g_t\|_{L^2(\R^d)}$}\\
Since, we have seen that $|1 - {\cal F}[g_t](\omega)| \leq {\bf 1}_{t\|\omega\| \geq 1}$, then for any $f \in W^\beta_2(\R^d)$ we have
\eqals{
\|f - f \star g_t\|^2_{L^2(\R^d)} &= \|{\cal F}[f] - {\cal F}[f \star g_t]\|^2_{L^2(\R^d)} = \| {\cal F}[f](1  - {\cal F}[g_t])\|^2_{L^2(\R^d)} \\
& = \int |1 - {\cal F}[g](t \omega)|^2 |{\cal F}[f](\omega)|^2 d\omega \leq \int_{t\|\omega\| \geq 1} {\cal F}[f](\omega)^2 d\omega \\
& = \int_{t\|\omega\| \geq 1}(1+\|\omega\|^2)^{-\beta}\, (1+\|\omega\|^2)^{\beta} |{\cal F}[f](\omega)|^2 d\omega \\
&\leq \sup_{t\|\omega\| \geq 1}(1+\|\omega\|^2)^{-\beta} \, \int (1+\|\omega\|^2)^{\beta} |{\cal F}[f](\omega)|^2 d\omega \\
&= 2^{2\beta}\|f\|^2_{W^\beta_2(\R^d)} \frac{t^{2\beta}}{(1+t^2)^\beta} \leq \|f\|^2_{W^\beta_2(\R^d)} (2t)^{2\beta}.
}

\paragraph{Step 2. Bounding $\|f \star g_t\|_{\hh_\eta}$} \\
However, the function $f \star g_t$ belongs to $\hh_\eta$ for any $\eta \in \R^d_{++}$. Indeed, as discussed in \cref{ex:gaussian-rkhs}, we have that $\|u\|_{\hh_\eta}$ is characterized as
$$\|u\|^2_{\hh_\eta} = c_\eta \int |{\cal F}[u](\omega)|^2 e^{\pi^2 \omega^\top \diag(\eta)^{-1} \omega} d\omega,$$
and $u \in \hh_\eta$ iff $\|u\|_{\hh_\eta} < \infty$, with $c_\eta = \pi^{-d/2}\det(\diag(\eta))^{1/2}$. Now, let $\eta_0 = \min_{i=1..d} \eta_i$, since we have seen that $|{\cal F}[g_t](\omega)| \leq {\bf 1}_{t\|\omega\| \leq 3}(\omega)$, then we have that
\eqals{
\|f \star & g_t\|^2_{\hh_\eta} = c_\eta \int |{\cal F}[f](\omega){\cal F}[g](t\omega)|^2 e^{\omega^\top \diag(\eta)^{-1} \omega} d\omega \\
& \leq c_\eta \int |{\cal F}[f](\omega){\cal F}[g](t\omega)|^2 e^{\frac{\pi^2}{\eta_0} \|\omega\|^2} d\omega = c_\eta \int_{t\|\omega\| \leq 3} |{\cal F}[f](\omega)|^2 e^{\frac{\pi^2}{\eta_0} \|\omega\|^2} d\omega \\
& = c_\eta \int_{t\|\omega\| \leq 3} |{\cal F}[f](\omega)|^2 (1+\|\omega\|^2)^\beta \frac{e^{\frac{\pi^2}{\eta_0} \|\omega\|^2}}{(1+\|\omega\|^2)^\beta} d\omega \\
& \leq c_\eta \sup_{t \|\omega\| \leq 3} \tfrac{e^{\frac{\pi^2}{\eta_0} \|\omega\|^2}}{(1+\|\omega\|^2)^\beta} \int |{\cal F}[f](\omega)|^2 (1+\|\omega\|^2)^\beta d\omega \\
&\leq \|f\|^2_{W^\beta_2(\R^d)} c_\eta 2^{2\beta} \sup_{r \leq 3/t} \tfrac{e^{r^2\pi^2/\eta_0}}{(1+r^2)^\beta}
}

\paragraph{Step 3. Bounding $\tr(\mm_\eps)$}\\
Note that the function $\frac{1}{(1+r^2)^{\beta}}\exp(\frac{r^2 \pi^2}{\eta_0})$ has only one critical point in $r$ that is a minimum, then $\sup_{r \leq 3/t} \frac{1}{(1+r^2)^{\beta}}\exp(\frac{r^2 \pi^2}{\eta_0}) \leq \max[1,\, \frac{t^{2\beta}}{(t^2+9)^{\beta}}\exp(\frac{9\pi^2}{\eta_0t^2})] \leq 1 + (t/3)^{2\beta} \exp(\frac{89}{\eta_0 t^2})$. Now let consider the functions $f_{i,t} = f_i \star g_t$ for $i \in \{1,\dots,q\}$ and note that, by the results above $\|f_{i,t} - f_i\|_{L^2(\R^d)} \leq \|f_i\|_{W^\beta_2(\R^d)} (2t)^{\beta}$ and $\|f_{i,t}\|^2_{\hh_\eta} \leq \|f_i\|^2_{W^\beta_2(\R^d)} c_\eta 2^{2\beta}(1 + (t/3)^{2\beta} \exp(\frac{89}{\eta_0 t^2}))$. Since $f_{i,t}$ belong to the reproducing kernel Hilbert space $\hh_\eta$, define the operator $\mm_\eps$ as 
$$\mm_\eps = \sum_{i=1}^q f_{i,t}f_{i,t}^\top.$$
First note that $\mm_\eps \in \psd(\hh_\eta)$, moreover $\operatorname{rank}(\mm_\eps) = q$ and 
\eqals{
\tr(\mm_\eps) = \sum_{i=1}^q \|f_{i,t}\|^2_{\hh_\eta} \leq c_\eta 2^{2\beta}(1 + (t/3)^{2\beta} e^{\frac{89}{\eta_0 t^2}}) \sum_{i=1}^q \|f_i\|^2_{W^\beta_2(\R^d)}.
}

\paragraph{Step 4. Bounding $\|p -  \pp{x}{\mm_\eps,\phi_\eta}\|_{L^1(\R^d)}$}\\
Note that 
$$\pp{x}{\mm_\eps,\phi_\eta} = \phi_\eta(x)^\top \mm_\eps \phi_\eta(x) = \sum_{i=1}^q (f_{i,t}^\top \phi(x))^2 = \sum_{i=1}^q f_{i,t}(x)^2.$$
Then, since $a^2 - b^2 = (a-b)(a+b)$ for any $a,b \in \R$, by applying the H\"older inequality
\eqals{
\|p -  \pp{x}{\mm_\eps,\phi_\eta}\|_{L^1(\R^d)} & = \|\sum_{i=1}^q f_i^2 - f_{i,t}^2\|_{L^1(\R^d)} = \|\sum_{i=1}^q (f_i - f_{i,t})(f_i + f_{i,t})\|_{L^1(\R^d)}\\
& \leq \sum_{i=1}^q \|f_i - f_{i,t}\|_{L^2(\R^d)}(\|f_i\|_{L^2(\R^d)} + \|f_{i,t}\|_{L^2(\R^d)}),
}
finally, by the Young convolution inequality, 
$$\|f_{i,t}\|_{L^2(\R^d)} = \|f_i \star g_t\|_{L^2(\R^d)} \leq \|f_i\|_{L^2(\R^d)} \|g_t\|_{L^1(\R^d)}.$$
By the change of variable $x = z t$, $dx = t^d dz$, we have
\eqals{
\|g_t\|_{L^1(\R^d)} = \int |g_t(tx)| dx =  \int t^{-d}|g(x/t)|dx  = \int |g(z)| dz = \|g\|_{L^1(\R^d)}.
}
then we obtain
\eqals{
\|p -  \pp{x}{\mm_\eps,\phi_\eta}\|_{L^1(\R^d)} ~\leq~ (2t)^\beta ~ (1 + \|g\|_{L^1(\R^d)})\sum_{i=1}^q \|f_i\|_{W^\beta_2(\R^d)}\|f_i\|_{L^2(\R^d)}.
}

\paragraph{Step 5. Bounding $\|p -  \pp{x}{\mm_\eps,\phi_\eta}\|_{L^2(\R^d)}$}\\
With the same reasoning above, we have
\eqals{
\|p -  \pp{x}{\mm_\eps,\phi_\eta}\|_{L^2(\R^d)} & = \|\sum_{i=1}^q f_i^2 - f_{i,t}^2\|_{L^2(\R^d)} = \|\sum_{i=1}^q (f_i - f_{i,t})(f_i + f_{i,t})\|_{L^2(\R^d)}\\
& \leq \sum_{i=1}^q \|f_i - f_{i,t}\|_{L^2(\R^d)}(\|f_i\|_{L^\infty(\R^d)} + \|f_{i,t}\|_{L^\infty(\R^d)})
}
finally, by the Young convolution inequality, 
$$\|f_{i,t}\|_{L^\infty(\R^d)} = \|f_i \star g_t\|_{L^\infty(\R^d)} \leq \|f_i\|_{L^\infty(\R^d)} \|g_t\|_{L^1(\R^d)}.$$
Then,
\eqals{
\|p -  \pp{x}{\mm_\eps,\phi_\eta}\|_{L^2(\R^d)} \leq (2t)^\beta (1+\|g\|_{L^1(\R^d)})\sum_{i=1}^q \|f_i\|_{W^\beta_2(\R^d)}\|f_i\|_{L^\infty(\R^d)}
}

\paragraph{Step 6. Setting $t$ appropriately}\\
Finally, noting that by construction $\|f\|_{L^2(\R^d)} \leq \|f\|_{W^\beta_2(\R^d)}$ and setting
\eqals{
t = \left(\frac{\eps}{C_1}\right)^{\frac{1}{\beta}}, ~~~ C_1 = 2^\beta (1 + \|g\|_{L^1(\R^d)}) \sum_{i=1}^q \|f_i\|_{W^\beta_2(\R^d)} \max(\|f_i\|_{L^\infty(\R^d)}, \|f_i\|_{W^\beta_2(\R^d)})
}
then, $\|p -  \pp{x}{\mm_\eps,\phi_\eta}\|_{L^j(\R^d)} \leq \eps$, with $j=1,2$. By Littlewood's interpolation inequality, $\|\cdot\|_{L^r(\R^d)} \leq \|\cdot\|^{(2/r)-1}_{L^1(\R^d)}\|\cdot\|^{2-(2/r)}_{L^2(\R^d)}$ when $r \in [1,2]$ (see, e.g, Thm. 8.5 pag 316 of \cite{castillo2016introductory}), we have
\eqals{
\|p -  \pp{x}{\mm_\eps,\phi_\eta}\|_{L^r(\R^d)} \leq \eps, \quad \forall r \in [1,2].
}
By setting $C_2 = 2^{2\beta} \sum_{i=1}^q \|f_i\|^2_{W^\beta_2(\R^d)}$, we have
$$\tr(\mm_\eps) \leq c_\eta C_2(1 + e^{\frac{89}{\eta_0 t^2}}(t/3)^{2\beta}) \leq c_\eta C_2 (1 + \tfrac{3^{-2\beta}}{C_1^2}\eps^2e^{\frac{89}{\eta_0}(\frac{C_1}{\eps})^{2/\beta}}) \leq C |\eta|^{1/2}(1 +  \eps^2e^{\frac{C'}{\eta_0 \eps^{2/\beta}}}),$$
where $|\eta| = \det(\diag(\eta))$ and $C' = 89C_1^{2/\beta}$ and $C = \pi^{-d/2} C_2 \max(1, 3^{-2\beta}/C_1^2)$.
\end{proof}

\subsection{Proof of \cref{thm:approximation}}
\label{proof:thm:approximation}

We can now prove \cref{thm:approximation}. We will prove a more general result \cref{thm:approximation-advanced}, from which \cref{thm:approximation} follows when $R = 1$ and $\X = (-1,1)^d$ applied to $\tilde{f}_1,\dots, \tilde{f}_q \in W^\beta_2(\R^d) \cap L^\infty(\R^d)$ that are the extension to $\R^d$ of the functions $f_1,\dots,f_q$ characterizing $p$ via \cref{asm:psd-model}. The details of the extension are in \cref{cor:extension-intersection}

\begin{theorem}\label{thm:approximation-advanced}
Let $R > 0$ and let $\X \subseteq T = (R,R)^d$ be a non-empty open set with Lipschitz boundary.
Let $f_1,\dots, f_q \in W^\beta_2(\R^d) \cap L^\infty(\R^d)$ and let $p = \sum_{i=1}^q \tilde{f}_i^2$.
Then, for any $\eps \in (0, 1/e]$, there exists $m \in \N$, $\eta \in \R^d_{++}$, a base point matrix $\tilde{X}\in \R^{m \times d}$ and a matrix $A \in \psd^m$ such that $\|\pp{\cdot}{A, \tilde{X}, \eta} - p\|_{L^2(\X)} \leq 2\eps$, with
\eqal{
m^{1/d} ~\leq~ C ~+~ C'\log\tfrac{1+R}{\eps} ~+~ C''\,R\eps^{-\frac{1}{\beta}}(\log\tfrac{(1+R)}{\eps})^{\frac{1}{2}} 
}
where $C,C',C''$ depend only on $\X, \beta, d, \|f_j\|_{W^\beta_2(\R^d)}, \|f_j\|_{L^\infty(\R^d)}$ for $j=1,\dots,q$. This implies that there exists a model of dimension $m$ such that $\|\pp{\cdot}{A, \tilde{X}, \eta} - p\|_{L^2(\X)} \leq \eps$, 
\eqal{
m = O\left(R^d\eps^{-d/\beta} (\log\tfrac{1+R}{\eps})^{d/2}\right).
}
\end{theorem}
\begin{proof}
Let $\eps \in (0,1/e]$ and $\eta = \tau {\bf 1}_d \in \R^d$ with $\tau > 0$ and $m \in \N$. Let $\mm_\eps \in \psd(\hh_{\eta})$ be the operator constructed in \cref{thm:Meps}. We consider the compression of the model $p_\eps = \pp{\cdot}{\mm_\eps, \phi_\eta}$ as in \cref{thm:compr-fill-dist}. In particular, let $\tilde{x}_1,\dots, \tilde{x}_m$ be a covering of $T$ with $\ell_2$.
We consider the following model $\tilde{p}_m = \pp{\cdot}{A_m, \tilde{X}_m, \eta}$ where $\tilde{X}_m \in \R^{m \times d}$ is the base point matrix whose $j$-th row  is the point $\tilde{x}_j$ for $j=1,\dots,m$,  and where $A_m \in \psd^m$ is defined as
\eqals{
A_m ~=~ K_{\tilde{X},\tilde{X},\eta}^{-1} \,\tilde{Z} \mm_{\eps} \tilde{Z}^*\, K_{\tilde{X},\tilde{X},\eta}^{-1},
}
where $\tilde{Z}:\hh_\eta \to \R^m$ is defined in \cref{eq:tildeZ-definition} and its adjoint $\tilde{Z}^*$ has range equal to $\operatorname{span}\{\phi_{\eta}(\tilde{x}_1),\dots,\phi_{\eta}(\tilde{x}_1)\}$. Note that, by construction $\tilde{Z}\, \mm_{\eps} \tilde{Z}^* \in \psd^m$ and so $A_m \in \psd^m$.

\paragraph{Step 1. Approximation error decomposition}
We will split the approximation error as follows,
\eqals{
\begin{split}
\|\tilde{p}_m - p\|_{L^2(\X)} &~\leq~ \|\tilde{p}_m - p_\eps \|_{L^2(\X)} +  \|p_\eps - p\|_{L^2(\X)}. 
\end{split}
}
Note that for the second term, by \cref{thm:Meps}, we have
\eqal{\label{eq:normL12-appr}
\|p_\eps - p\|_{L^r(\X)} \leq \|p_\eps - p\|_{L^r(\R^d)} \leq \eps, \qquad \forall~r \in [1,2]
}

\paragraph{Step 2. Error induced by a projection}
Since an $h$-covering of a set has fill distance $h$, by definition of fill distance \cref{eq:fill-distance}, then we will choose the $m$ base points $\tilde{x}_1, \dots, \tilde{x}_m$ to be an $h$-covering of the hypercube $T$. Since the $\ell_2$-ball of diameter $1$ contains a cube of side $1/\sqrt{d}$, it is possible to cover a cube of side $2R$ with $m \leq (1+ 2R\sqrt{d}/h)^d$ balls of diameter $2h$ (and so of radius $h$), see, e.g., Thm. 5.3, page 76 of \cite{cucker2007learning}. 
Now, by \cref{thm:compr-fill-dist} applied to $\mm_\eps$, we have that when the fill distance $h$ (defined in \cref{eq:fill-distance}) satisfies $h \leq \sigma/C'$ with $\sigma = \min(R,1/\sqrt{\tau})$, then
\eqals{
|p_\eps(x) - \tilde{p}_m(x)| ~\leq~ 2C \sqrt{\|\mm_\eps\| p_\eps(x)}\, e^{-\frac{c\,\sigma}{h}\log\frac{c\,\sigma}{h}}  \,+\, C^2\|\mm_\eps\|\,e^{-\frac{2c\,\sigma}{h}\log\frac{c\,\sigma}{h}}, \quad \forall x \in \X
}
with $c, C, C'$ depending only on $d$. Now denoting by $\alpha = 2C \sqrt{\|\mm_\eps\|} e^{-\frac{c\,\sigma}{h}\log\frac{c\,\sigma}{h}}$ and $\beta = C^2\|\mm_\eps\|\,e^{-\frac{2c\,\sigma}{h}\log\frac{c\,\sigma}{h}}$, we have
\eqals{
\|p_\eps - \tilde{p}_m\|_{L^2(\X)} &\leq \|\alpha \sqrt{p_\eps} + \beta\|_{L^2(\X)} \leq \alpha \|\sqrt{p_\eps}\|_{L^2(\X)} + \beta \|{\bf}1\|_{L^2(\X)} \\
& \leq (2R)^{d/2}\beta + \alpha \|p_\eps^{1/2}\|^2_{L^2(\X)}.
}
where we used the fact that $\|{\bf 1}\|^2_{L^2(\X)} \leq \|{\bf 1}\|^2_{L^2(T)} = \int_T dx = (2R)^{d}$.

\paragraph{Step 3. Final bound}
First, note that by \cref{thm:Meps}
\eqals{
\tr(\mm_{\eps}) \leq  C_1 \tau^{d/2}(1 + \eps^2\exp(\tfrac{C_2}{\tau} \eps^{-\frac{2}{\beta}}))
}
where $C_1, C_2$ are independent on $\eps, \tau$ and depend only on $\beta, d, \|f_i\|_{W^\beta_2(\R^d)}, \|f_i\|_{L^\infty(\R^d)}$. By setting $\tau = \frac{C_2\eps^{-2/\beta}}{2\log\frac{1+R}{\eps}}$ we have
\eqals{
\|\mm_\eps\| \leq \tr(\mm_{\eps}) \leq  C_1 \tau^{d/2}(1 + \eps^2\exp(\tfrac{C_2}{\tau} \eps^{-\frac{2}{\beta}})) \leq (1 + R)^2 C_3 \eps^{-d/\beta}.
}
with $C_3 = 2^{-d/2} C_1 C_2^{d/2}$. Then, note that $\|p_\eps^{1/2}\|_{L^2(\X)} = \|p_\eps\|^{1/2}_{L^1(\X)}$, so, using \cref{eq:normL12-appr}, we have
\eqals{
\|p_\eps\|_{L^1(\X)} \leq \|p_\eps - p\|_{L^1(\X)} + \|p\|_{L^1(\X)} \leq 1 + \eps \leq 2.
}
By choosing $h = c \sigma/s$ with $s = \max(C', (1+\frac{d}{2\beta}) \log \frac{1}{\eps} \, + (1+\frac{d}{4}) \log(1+R) \, + \log(C \sqrt{C_3}) + e)$, since $s \geq e$, then $\log s \geq 1$, so
\eqals{
C e^{-\frac{c\,\sigma}{h}\log\frac{c\,\sigma}{h}} = Ce^{-s \log s} \leq Ce^{-s} \leq \tfrac{1}{15\sqrt{C_3}} (1+R)^{-d/4}\eps^{1 + \frac{d}{2\beta}}.
}
Gathering the results from the previous steps, we have
\eqals{
\|\tilde{p}_m - p\|_{L^2(\X)} &\leq \eps  + 4 \|\mm_\eps\|^{1/2} C e^{-\frac{c\,\sigma}{h}\log\frac{c\,\sigma}{h}}  ~+~ \|\mm_\eps\| R^{d/2} C^2 e^{-\frac{c\,\sigma}{h}\log\frac{c\,\sigma}{h}}\\
& \leq \eps + \tfrac{4}{15}(1+R)^{-(d/4)}\eps + \tfrac{1}{225} \big(\tfrac{R}{R+1})^{d/2} \eps^2 \\
& \leq 2\eps.
}
To conclude we recall the fact that $\tilde{x}_1, \dots, \tilde{x}_m$ is a $h$-covering of $T$, guarantees that the number of centers $m$ in the covering satisfies
\eqals{
m \leq (1 + \tfrac{2R\sqrt{d}}{h})^d.
}
Then, since $h \geq c\sigma/(C_4 \log\frac{C_5 \log (1+R)}{\eps})$ with $C_4 = 1 + d/\min(2\beta,4)$ and $C_5 = (C\sqrt{C_3}e)^{1/C_4}$, and since $\sigma = \min(R,1/\sqrt{\tau})$, then $R/\sigma = \max(1, R \sqrt{\tau}) \leq 1 + \sqrt{C_2/2} \eps^{-1/\beta}(\log\frac{1+R}{\eps})^{-1/2}$, so we have
\eqals{
m^{\frac{1}{d}} &\leq 1+ 2R\sqrt{d}/h \leq 1 + C_4 \big(1 +  R\sqrt{d} (\tfrac{C_2}{2})^{1/2}\eps^{-\frac{1}{\beta}}(\log\tfrac{1+R}{\eps})^{-\frac{1}{2}}\big) \log\tfrac{C_5(1+R)}{\eps}\\
& = (1 + C_4 \log C_5) + C_4\log\tfrac{1+R}{\eps} ~+~ R C_4\sqrt{d} (\tfrac{C_2}{2})^{\frac{1}{2}}\eps^{-\frac{1}{\beta}}(\log\tfrac{1+R}{\eps})^{-\frac{1}{2}}\log\tfrac{C_5(1+R)}{\eps} \\
&\leq C_6 ~+~ C_4\log\tfrac{1+R}{\eps} ~+~ RC_7\eps^{-\frac{1}{\beta}}(\log\tfrac{(1+R)}{\eps})^{\frac{1}{2}}
}
with $C_6 = 1 + C_4 \log C_5$, $C_7 = C_4\sqrt{d}(\tfrac{C_2}{2})^{1/2} \log(eC_5)$, since $\log(e C_5) \geq 1$, then \eqals{
\log\tfrac{C_5(1+R)}{\eps} & = \log eC_5 + \log\tfrac{(1+R)}{e\eps} \leq (\log e C_5) (1 + \tfrac{\log\tfrac{(1+R)}{e\eps}}{\log e C_5}) \\
& \leq (\log e C_5) (1 + \log\tfrac{(1+R)}{e\eps}) \leq  (\log e C_5) (1 + \log\tfrac{(1+R)}{\eps}).
}
The constants $C, C', C''$ in the statement of the theorem correspond to $C = C_6, C' = C_4, C'' = C_7$.
\end{proof}

\section{Learning a PSD model from examples}
\label{app:learning}

In this section we provide a proof for \cref{thm:learning}, which characterizes the learning capabilities of PSD models. We first provide intermediate results that will be useful for the proof. 

Let $\X$ be a compact space and let $p:\X \to \R$ be a probability density which we assume to belong to $p \in L^2(\X)$. Let $x_1,\dots, x_n$ sampled i.i.d. according to $p$. We will study an estimator for $p$ in terms of the squared $L2$ norm $\|\cdot\|_{L^2(\X)}$.
Let $\eta = \eta_0 1_d$ with $\eta_0 > 0$ and $\tilde{X} \in \R^{m \times d}$ the base point matrix whose rows are some points $\tilde{x}_1,\dots, \tilde{x}_m$. We will consider the following estimator $\hat{p}$ for $p$
\eqal{
\hat{p}(x) = \pp{x}{\hat{A}, \tilde{X}, \eta}, \qquad \hat{A} = \min_{A \in \psd^m} \hat{L}_\la(A),
}
and, denoting by $\tilde{R}$ the Cholesky decomposition of $K_{\tilde{X},\tilde{X},\eta}$, i.e. the upper triangular matrix such that $K_{\tilde{X},\tilde{X},\eta} = \tilde{R}^\top \tilde{R}$, we define
\eqal{\label{eq:estimator-L2}
\hat{L}_\la(A) = \int_{\X} \pp{x}{A, \tilde{X},\eta}^2 dx - \frac{2}{n} \sum_{i=1}^n {\bf 1}_\X(x_i)\pp{x_i}{A, \tilde{X},\eta} + \la \|\tilde{R} A \tilde{R}^\top\|^2_F.
}
Denote by $L_\la(A)$ the following functional
\eqals{
L_\la(A) = \int_\X \pp{x}{A, \tilde{X},\eta}^2 dx - \int_\X \pp{x}{A, \tilde{X},\eta} p(x) dx + \la \|\tilde{R} A \tilde{R}^\top\|^2_F.
}
and by $\bar{A}_{\eta,\la} \in \psd^m$ the matrix $\bar{A}_{\eta,\la} = \min_{A \in \psd} L_\la(A)$.

\subsection{Operatorial characterization of $\hat{L}_\la, L_\la$}

We can now rewrite the loss functions as follows
\begin{lemma}[Characterization of $\hat{L}_\la, L_\la$ in terms of $\hat{v}, v$]\label{lm:char-Lla}
For any $\la \geq 0$ the following holds
\eqals{
\hat{L}_\la(A) & = \|S \vec(\tilde{Z}^*A\tilde{Z})\|^2_{L^2(\X)} + \la \| \vec(\tilde{Z}^*A\tilde{Z})\|^2_{{\cal G}_\eta} -  2\scal{\hat{v}}{\vec(\tilde{Z}^*A\tilde{Z})}_{{\cal G}_\eta}\\
L_\la(A) & = \|S \vec(\tilde{Z}^*A\tilde{Z})\|^2_{L^2(\X)} + \la \| \vec(\tilde{Z}^*A\tilde{Z})\|^2_{{\cal G}_\eta} - 2\scal{v}{\vec(\tilde{Z}^*A\tilde{Z})}_{{\cal G}_\eta}.
}
\end{lemma}
\begin{proof}
With the notation \cref{sect:notation} and by using the operators defined in \cref{sec:operators-definition} for any $\mm \in \psd(\hh_\eta)$ we have
\eqal{
\pp{x}{\mm, \phi_\eta} = \scal{\psi_\eta(x)}{\vec(\mm)}_{{\cal G}_\eta}, \quad \forall x \in \R^d
}
and in particular for any $A \in \psd^m$, we have
\eqals{\label{eq:equivalence-pp}
\pp{x}{A,\tilde{X},\eta} = \pp{x}{\tilde{Z}^*A\tilde{Z}, \phi_\eta} = \scal{\psi_\eta(x)}{\vec(\tilde{Z}^*A\tilde{Z})}_{{\cal G}_\eta}, \quad \forall x \in \R^d
}
Now note that, by cyclicity of the trace, for any matrix $A,B \in \R^{m\times m}$ we have
\eqals{
\|B^{1/2}AB^{1/2}\|^2_F = \tr(B^{1/2}AB^{1/2}B^{1/2}AB^{1/2}) = \tr(ABAB).
}
This implies that $\|K_{\tilde{X},\tilde{X},\eta}^{1/2}AK_{\tilde{X},\tilde{X},\eta}^{1/2}\|_F^2 = \tr(AK_{\tilde{X},\tilde{X},\eta}AK_{\tilde{X},\tilde{X},\eta})$. Moreover, by cyclicity of the trace, definition of Frobenius norm and since $\tilde{Z}\tilde{Z}^* = K_{\tilde{X},\tilde{X},\eta}$ we have
\eqals{
\tr(AK_{\tilde{X},\tilde{X},\eta}AK_{\tilde{X},\tilde{X},\eta}) &= \tr(A\tilde{Z}\tilde{Z}^*A\tilde{Z}\tilde{Z}^*) = \tr(\tilde{Z}^*A\tilde{Z}\tilde{Z}^*A\tilde{Z}) \\
& = \scal{\vec(\tilde{Z}^*A\tilde{Z})}{\vec(\tilde{Z}^*A\tilde{Z})}_{{\cal G}_\eta} = \|\vec(\tilde{Z}^*A\tilde{Z})\|^2_{{\cal G}_\eta}.
}

By linearity of the integral and the inner product and since $\phi_\eta$ and so $\psi_\eta$ are uniformly bounded,
\eqals{
\int_\X \pp{x}{\mm, \phi_\eta}^2 dx &= \int_\X \scal{\vec(\mm)}{(\psi_\eta(x)\psi_\eta(x)^\top)\vec(\mm)}_{{\cal G}_\eta} dx \\
& =  \scal{\vec(\mm)}{\left(\int_\X \psi_\eta(x)\psi_\eta(x)^\top dx\right)\vec(\mm)}_{{\cal G}_\eta} \\
& = \scal{\vec(\mm)}{Q \vec(\mm)}_{{\cal G}_\eta} =  \scal{\vec(\mm)}{S^*S \vec(\mm)}_{{\cal G}_\eta} \\
& =\scal{S\vec(\mm)}{S \vec(\mm)}_{{\cal G}_\eta} =  \|S \vec(\mm)\|^2_{L^2(\X)}
}
Then, we have
\eqals{
\hat{L}_\la(A) &= \|S \vec(\tilde{Z}^*A\tilde{Z})\|^2_{L^2(\X)} - 2\scal{\hat{v}}{\vec(\tilde{Z}^*A\tilde{Z})}_{{\cal G}_\eta}  + \la \|\vec(\tilde{Z}^*A\tilde{Z})\|^2_{{\cal G}_\eta}.
}
The identical reasoning holds for $L_\la(A)$, with respect to $v$.
\end{proof}

\begin{theorem}[Error decomposition]\label{thm:error-dec}
Let $\hat{A}$ be a minimizer of $\hat{L}_\la$ over a set $S \subseteq \R^{m\times m}$ (non-necessarily convex). Denote by $\mu(A)$ the vector $\mu(A) = \vec(\tilde{Z}^*A\tilde{Z}) \in {\cal G}_\eta$ for any $A \in S$. Then for any $A \in S$ the following holds
\eqal{\label{eq:error-dec-final}
\begin{split}
\big(\|S \mu(\hat{A}) - p\|^2_{L^2(\X)} + \la\|\mu(\hat{A})\|^2_{{\cal G}_\eta}\big)^{1/2} &~~\leq~~ \big(\|S \mu(A) - p\|^2_{L^2(\X)} + \la\|\mu(A)\|^2_{{\cal G}_\eta}\big)^{1/2}\\
&\qquad\quad~+~ 3\sqrt{2}\|(Q+\la I)^{-1/2}(\hat{v} - v)\|_{{\cal G}_\eta}.
\end{split}
}
\end{theorem}
\begin{proof}
We start noting that since $\hat{A}$ is the minimizer over $S$ of $\hat{L}_\la$, then $\hat{L}_\la(\hat{A}) \leq \hat{L}_\la(A)$ for any $A \in S$. In particular, since $\bar{A} \in S$ this means that $\hat{L}_\la(\hat{A}) \leq \hat{L}_\la(\bar{A})$, then $\hat{L}_\la(\hat{A}) - \hat{L}_\la(\bar{A}) \leq 0$, this implies that
\eqals{\label{eq:bound-losses-hatA-A}
L_\la(\hat{A}) - L_\la(A) &= L_\la(\hat{A}) - \hat{L}_\la(\hat{A}) + \hat{L}_\la(\hat{A}) - \hat{L}_\la(A) + \hat{L}_\la(A) - L_\la(A) \\
& \leq L_\la(\hat{A}) - \hat{L}_\la(\hat{A}) + \hat{L}_\la(A) - L_\la(A).
}
Denote $\hat{\mu} = \vec(\tilde{Z}^*\hat{A}\tilde{Z})$ and $\bar{\mu} = \vec(\tilde{Z}^*\bar{A}\tilde{Z})$.
Now note that, by the characterization of $L_\la, \hat{L}_\la$ in \cref{lm:char-Lla}, we have 
\eqals{\label{eq:char-bound-losses}
L_\la(\hat{A}) - \hat{L}_\la(\hat{A}) + \hat{L}_\la(A) - L_\la(A) = 2\scal{\hat{v}-v}{\hat{\mu} - \bar{\mu}}_{{\cal G}_\eta}.
}

\paragraph{Step 1. Decomposing the error}\\
Note that, since $v = S^*p$, then $\scal{v}{w}_{{\cal G}_\eta} = \scal{S^*p}{w}_{{\cal G}_\eta} = \scal{p}{Sw}_{L^2(\X)}$ for any $w \in {\cal G}_\eta$. Then, for any $A \in \psd^m$, denoting by $\mu = \vec(\tilde{Z}^*A\tilde{Z})$, and substituting $\scal{v}{\mu}_{{\cal G}_\eta}$ with $\scal{p}{S\mu}_{L^2(\X)}$ in the definition of $L_\la(A)$, we have
\eqal{\label{eq:Lla+p}
L_\la(A) + \|p\|^2 &  = \|S\mu\|^2_{L^2(\X)} - 2\scal{p}{S\mu}_{L^2(\X)} +  \|p\|^2_{L^2(\X)} ~+~ \la\|\mu\|^2_{{\cal G}_\eta}\\
& = \|S\mu - p\|^2_{L^2(\X)} + \la\|\mu\|^2_{{\cal G}_\eta}.
}
From \cref{eq:bound-losses-hatA-A,eq:char-bound-losses} and the equation above, we obtain
\eqal{\label{eq:error-dec-0}
\begin{split}
\|S\hat{\mu} - p\|^2_{L^2(\X)} + \la\|\hat{\mu}\|^2_{{\cal G}_\eta} &= L_\la(\hat{A}) + \|p\|^2_{L^2(\X)} \\
&\leq L_\la(\bar{A}) + \|p\|^2_{L^2(\X)} + 2\scal{\hat{v}-v}{\hat{\mu} - \bar{\mu}}_{{\cal G}_\eta}\\
 &= \|S\bar{\mu} - p\|^2_{L^2(\X)} + \la\|\bar{\mu}\|^2_{{\cal G}_\eta} ~+~ 2\scal{\hat{v}-v}{\hat{\mu} - \bar{\mu}}_{{\cal G}_\eta} 
\end{split}
}
Now, note that
\eqals{
\|S\hat{\mu} - &p\|^2_{L^2(\X)} = \|S(\hat{\mu} - \bar{\mu}) + (S\bar{\mu} - p)\|^2_{L^2(\X)} \\
& = \|S(\hat{\mu} - \bar{\mu})\|^2_{L^2(\X)} + 2\scal{S(\hat{\mu} - \bar{\mu})}{(S\bar{\mu} - p)}_{L^2(\X)} + \|S\bar{\mu} - p\|^2_{L^2(\X)}\\
\|\hat{\mu}\|^2_{{\cal G}_\eta} &= \|(\hat{\mu} - \bar{\mu})+\bar{\mu}\|^2_{{\cal G}_\eta} = \|\hat{\mu} - \bar{\mu}\|^2_{{\cal G}_\eta} + 2\scal{\bar{\mu}}{\hat{\mu} - \bar{\mu}}_{{\cal G}_\eta}+ \|\bar{\mu}\|^2_{{\cal G}_\eta}.
}
Expanding $\|S\hat{\mu} - p\|^2_{L^2(\X)}$ and $\|\hat{\mu}\|^2_{{\cal G}_\eta}$ in \cref{eq:error-dec-0} and reorganizing the terms, we obtain
\eqal{\label{eq:error-dec-1}
\begin{split}
\|S(\hat{\mu} - \bar{\mu})\|^2_{L^2(\X)} + \la \|\hat{\mu} - \bar{\mu}\|^2_{{\cal G}_\eta} &\leq -2\scal{S(\hat{\mu} - \bar{\mu})}{(S\bar{\mu} - p)}_{L^2(\X)} \\
& \qquad - 2\la\scal{\bar{\mu}}{\hat{\mu} - \bar{\mu}}_{{\cal G}_\eta}  \\
& \qquad + 2\scal{\hat{v}-v}{\hat{\mu}-\bar{\mu}}_{{\cal G}_\eta}.
\end{split}
}

\paragraph{Step 2. Bounding the three terms of the decomposition}\\
The proof is concluded by bounding the three terms of the right hand side of the equation above, indeed
\eqal{
\label{eq:error-dec-term1} -2\scal{S(\hat{\mu} - \bar{\mu})}{(S\bar{\mu} - p)}_{L^2(\X)} &\leq 2\|S(\hat{\mu} - \bar{\mu})\|_{L^2(\X)}\|S\bar{\mu} - p\|_{L^2(\X)}\\
\label{eq:error-dec-term2} -2\la\scal{\bar{\mu}}{\hat{\mu} - \bar{\mu}}_{{\cal G}_\eta} & \leq 2\la\|\bar{\mu}\|_{{\cal G}_\eta}\|\hat{\mu} - \bar{\mu}\|_{{\cal G}_\eta}.
}
For the third term we multiply and divide by $Q + \la$ (it is invertible since $Q \in \psd({\cal G}_\eta)$ and $\la > 0$), so
\eqal{\label{eq:error-dec-vvmm}
\begin{split}
2\scal{\hat{v}-v}{\hat{\mu}-\bar{\mu}}_{{\cal G}_\eta} &= 2\scal{(Q+\la)^{-1/2}(\hat{v}-v)}{(Q+\la)^{1/2}(\hat{\mu}-\bar{\mu})}_{{\cal G}_\eta}\\
& \leq 2\|(Q+\la)^{-1/2}(\hat{v}-v)\|_{{\cal G}_\eta}\|(Q+\la)^{1/2}(\hat{\mu}-\bar{\mu})\|_{{\cal G}_\eta}.
\end{split}
}
Note that for any $w \in {\cal G}_\eta$, since $Q$ is characterized as $Q=S^*S$, we have
\eqal{
\begin{split}
\|(Q+\la)^{1/2} w\|^2_{{\cal G}_\eta} &= \scal{w}{(Q+\la) w}_{{\cal G}_\eta} = \scal{w}{Q w}_{{\cal G}_\eta} + \la \scal{w}{w}_{{\cal G}_\eta}\\
& = \scal{Sw}{Sw}_{{\cal G}_\eta} + \la \scal{w}{w}_{{\cal G}_\eta} = \|Sw\|^2_{{\cal G}_\eta} + \la\|w\|^2_{{\cal G}_\eta}.
\end{split}
} 
By applying the equation above to $w = \hat{\mu}-\bar{\mu}$, since $\sqrt{a^2 + b^2} \leq a + b, \forall a, b \geq 0$, we have
\eqal{\label{error-dec-2}
\|(Q+\la)^{1/2}(\hat{\mu}-\bar{\mu})\|_{{\cal G}_\eta} &\leq  \|S (\hat{\mu}-\bar{\mu})\|_{L^2(\X)}+\sqrt{\la}\|(\hat{\mu}-\bar{\mu})\|_{{\cal G}_\eta}.
}
Combining \cref{eq:error-dec-1} with the bounds in \cref{eq:error-dec-term1,eq:error-dec-term2} for the first two terms of its right hand side and with the bounds in \cref{eq:error-dec-vvmm,error-dec-2} for the third term, we have
\eqals{
\|S(\hat{\mu} - \bar{\mu})\|^2_{L^2(\X)} + \la \|\hat{\mu} - \bar{\mu}\|^2_{{\cal G}_\eta} & ~~\leq~~ 2\alpha \|S(\hat{\mu} - m)\|_{L^2(\X)} ~+~ 2\beta\,\sqrt{\la}  \|\hat{\mu} - \bar{\mu}\|_{{\cal G}_\eta},
}
with $\alpha = \|S\bar{\mu} - p\|_{L^2(\X)} + \|(Q+\la)^{-1/2}(\hat{v}-v)\|_{{\cal G}_\eta}$, $\beta = \sqrt{\la}\|\bar{\mu}\|_{{\cal G}_\eta} + \|(Q+\la)^{-1/2}(\hat{v}-v)\|_{{\cal G}_\eta}$. \\

\paragraph{Step 3. Solving the inequality associated to the bound of the three terms}\\
Now denoting by $x = \|S(\hat{\mu} - \bar{\mu})\|_{L^2(\X)}$ and $y = \sqrt{\la} \|\hat{\mu} - \bar{\mu}\|_{{\cal G}_\eta}$, the inequality above becomes
\eqals{
x^2 + y^2 \leq 2\alpha x + 2 \beta y.
}
By completing the squares it is equivalent to
$(x-\alpha)^2 + (y - \beta)^2 \leq \alpha^2 + \beta^2$, 
from which we derive that $(x-\alpha)^2 \leq (x-\alpha)^2 + (y - \beta)^2 \leq \alpha^2 + \beta^2$. This implies $x \leq \alpha + \sqrt{\alpha^2 + \beta^2} \leq 2\alpha + \beta$. With the same reasoning we derive $y \leq \alpha + 2\beta$, that corresponds to
\eqals{\label{eq:error-dec-3}
\|S(\hat{\mu} - \bar{\mu})\|_{L^2(\X)} &\leq 2\|S\bar{\mu} - p\|_{L^2(\X)} + 3\|(Q+\la)^{-1/2}(\hat{v}-v)\|_{{\cal G}_\eta} + \sqrt{\la}\|\bar{\mu}\|_{{\cal G}_\eta}\\
\|\hat{\mu} - \bar{\mu}\|_{{\cal G}_\eta} & \leq \tfrac{1}{\sqrt{\la}}\|S\bar{\mu} - p\|_{L^2(\X)} + \tfrac{3}{\sqrt{\la}}\|(Q+\la)^{-1/2}(\hat{v}-v)\|_{{\cal G}_\eta} + 2\|\bar{\mu}\|_{{\cal G}_\eta}
}

\paragraph{Step 4. The final bound}\\
The final result is obtained by bounding the term $\scal{\hat{v}-v}{\hat{\mu}-\bar{\mu}}_{{\cal G}_\eta}$ in \cref{eq:error-dec-0}. In particular, we will bound it by using \cref{eq:error-dec-vvmm}, and by bounding the resulting term $\|(Q+\la)^{1/2}(\hat{\mu} - \bar{\mu})\|_{{\cal G}_\eta}$ with \cref{error-dec-2} and the resulting terms $\|S(\hat{\mu} - \bar{\mu})\|_{L^2(\X)}, \|\hat{\mu} - \bar{\mu}\|_{{\cal G}_\eta}$ via \cref{eq:error-dec-3}. This leads to
\eqal{\label{eq:error-dec-final0}
\begin{split}
\|S\hat{\mu} - p\|^2_{L^2(\X)} + \la\|\hat{\mu}\|^2_{{\cal G}_\eta} &\leq \|S\bar{\mu} - p\|^2_{L^2(\X)} + \la\|\bar{\mu}\|^2_{{\cal G}_\eta} \\
&\qquad ~+~ 6 \tau (\|S\bar{\mu} - p\|_{L^2(\X)} + \sqrt{\la}\|\bar{\mu}\|_{{\cal G}_\eta} ~+~ 2\tau).
\end{split}
}
with $\tau = \|(Q+\la)^{-1/2}(\hat{v}-v)\|_{{\cal G}_\eta}$. We can optimize the writing of the theorem by noting that since $2ab\leq a^2 + b^2$ for any $a,b \in \R$, we have $a+b \leq \sqrt{2(a^2+b^2)}$ and so 
$\|S\bar{\mu} - p\|_{L^2(\X)} + \sqrt{\la}\|\bar{\mu}\|_{{\cal G}_\eta} \leq \sqrt{2 (\|S\bar{\mu} - p\|^2_{L^2(\X)} + \la\|\bar{\mu}\|^2_{{\cal G}_\eta})}$ and the bound in \cref{eq:error-dec-final0} becomes
\eqal{\label{eq:error-dec-final1}
\begin{split}
\|S\hat{\mu} - p\|^2_{L^2(\X)} + \la\|\hat{\mu}\|^2_{{\cal G}_\eta} &\leq z^2 + 6 \tau (\sqrt{2}z + 2\tau) \leq (z + 3\sqrt{2} \tau)^2
\end{split}
}
where $z^2 = \|S\bar{\mu} - p\|^2_{L^2(\X)} + \la\|\bar{\mu}\|^2_{{\cal G}_\eta}$. The final result is obtained  by noting that, since $\bar{A}$ is the minimizer of $L_\la$, then for any $A \in \psd^m$ the following holds
\eqals{
z^2 &= \|p\|^2 + L_\la(\bar{A}) = \|p\|^2 + \min_{A' \in \psd^m} L_\la(A') \\
& \leq \|p\|^2 + L_\la(A) = \|S\mu(A) - p\|^2_{L^2(\X)} + \la\|\mu(A)\|^2_{{\cal G}_\eta}.
}

\end{proof}

\begin{lemma}\label{lm:StoL2}
Let $\X \subseteq \R^d$ and let $\eta \in \R^d_{++}$. Let $\tilde{x}_1, \dots, \tilde{x}_m$ be a set of points in $\R^d$ and let $\tilde{X} \in \R^{m \times d}$ be the associated base point matrix, i.e., the $j$-th row of $\tilde{X}$ corresponds to $\tilde{x}_j$.
With the notation and the definitions of \cref{sec:operators-definition} denote by $\mu(A) = \vec(\tilde{Z}^* A \tilde{Z})$. Then, for any $A \in \psd^m$ and any $p \in L^2(\X)$
\eqals{
\|S\mu(A) - p\|^2_{L^2(\X)}  = \|\pp{\cdot}{A, \tilde{X}, \eta} - p\|^2_{L^2(\X)},
}
and moreover
\eqals{
\|\mu(A)\|^2_{{\cal G}_\eta} = \|\tilde{Z}^*A\tilde{Z}\|^2_{F}.
}
\end{lemma}
\begin{proof}
Denote by $\mu(A) = \vec(\tilde{Z}^* A \tilde{Z})$. We recall that the operator used in the rest of the section are defined in \cref{sec:operators-definition}.
First we recall from \cref{eq:Lla+p} that, by definition of $S:{\cal G}_\eta \to L^2(\X)$, we have $S \mu(A) = \scal{\psi_\eta(\cdot)}{\mu(A)}_{{\cal G}_\eta} \in L^2(\X)$. In particular, by definition of $\vec$, for any $x \in \X$
\eqals{
(S \mu(A))(x) = \scal{\mu(A)}{\psi_\eta(x)}_{{\cal G}_\eta} = \scal{\vec(\tilde{Z}^* A \tilde{Z})}{\phi_\eta(x) \otimes \phi_\eta(x)}_{{\cal G}_\eta} = \phi_\eta(x)^\top \tilde{Z}^* A \tilde{Z} \phi_\eta(x).
}
Now, since $\tilde{Z}^* A \tilde{Z} = \sum_{i,j=1}^m A_{i,j} \phi_\eta(x_i) \phi_\eta(x_j)^\top$, and, by the representer property of the kernel $k_\eta$ we have $k_\eta(x,x') = \phi_\eta(x)^\top \phi_\eta(x')$ (see \cref{eq:reproducing-property,ex:gaussian-rkhs}), then
\eqals{
\phi_\eta(x)^\top \tilde{Z}^* A \tilde{Z} \phi_\eta(x) &= \sum_{i,j=1}^m A_{i,j} (\phi_\eta(x)^\top\phi_\eta(x_i)) (\phi_\eta(x_j)^\top\phi_\eta(x)) \\
&= \sum_{i,j=1}^m A_{i,j} k_\eta(x,x_i) k_\eta(x,x_j) = \pp{x}{A, \tilde{X}, \eta}.
}
Then $S \mu(A) = \pp{\cdot}{A, \tilde{X}, \eta} \in L^2(\X)$. So
\eqals{
\|S\mu(A) - p\|^2_{L^2(\X)} = \|\pp{\cdot}{A, \tilde{X}, \eta} - p\|^2_{L^2(\X)}.
}
To conclude note that, by the properties of $\vec$ (see \cref{sec:operators-definition}), we have
\eqals{
\|\mu(A)\|^2_{{\cal G}_\eta} = \scal{\vec(\tilde{Z}^*A\tilde{Z})}{\vec(\tilde{Z}^*A\tilde{Z})}_{{\cal G}_\eta} = \|\tilde{Z}^*A\tilde{Z}\|^2_{F}
}
\end{proof}

\begin{lemma}\label{lm:Qla-v}
Let $s \in \N$ and $\delta \in (0,1]$, $\la > 0$ let $\X \subset \R^d$ be an open set with Lipschitz boundary and let $p \in L^1(\X) \cap L^\infty(\X)$. Then the following holds with probability $1-\delta$
\eqals{
\|(Q+\la I)^{-1/2}(\hat{v} - v)\|_{{\cal G}_\eta} \leq \frac{C \tau^{d/4}\log \frac{2}{\delta}}{n (\la \tau^d)^{d/4s}} ~+~ \sqrt{\frac{2 \tr(Q_\la^{-1}Q) \log\frac{2}{\delta}}{n}},
}
where $C$ is a constant that depends only on $\X, s, d$.
\end{lemma}
\begin{proof}
We are going to use here a Bernstein inequality for random vectors in separable Hilbert spaces (see, e.g., Thm 3.3.4 of \cite{yurinsky1995sums}). 
Define the random variable $\zeta = Q_\la^{-1/2} \psi_\eta(x){\bf 1}_{\X}(x)$ with $x$ distributed according to $\rho$ and $Q_\la = Q + \la I$. To apply such inequality, we need to control the second moment and the norm of $\zeta$. First note that $\zeta $

\paragraph{Step 1. Bounding the variance of $\zeta$}
 We have
\eqal{
\mathbb{E} \|\zeta\|^2 &= \tr\left(\int \zeta \zeta^\top p(x) dx \right) = \tr\left(\int_\X  Q_\la^{-1/2}(\psi_\eta(x) \psi_\eta(x)^\top)Q_\la^{-1/2} p(x) dx \right) \\
& \leq \|p\|_{L^\infty(\X)}\tr\left(Q_\la^{-1/2} \left(\int_\X \psi_\eta(x) \psi_\eta(x)^\top dx \right) Q_\la^{-1/2} \right)\\
&=  \|p\|_{L^\infty(\X)} \tr(Q_\la^{-1/2} Q Q_\la^{-1/2}) = \tr(Q Q_\la^{-1}).
}

\paragraph{Bounding the norm of $\zeta$}
 Applying \cref{lm:sup-phi-to-Linfty} the the operator $Q_\la^{-1/2}$ we have that
\eqals{
\textrm{ess-sup} \|\zeta\|_{{\cal G}_\eta} \leq \sup_{x \in \X} \|Q_\la^{-1/2} \phi_\eta(x)\| \leq \sup_{\|f\|_{{\cal G}_\eta} \leq 1} \|Q_\la^{-1/2}f\|_{L^\infty(\X)}.
}
Now, according to the definitions in \cref{sec:operators-definition}, note that the reproducing kernel Hilbert space ${\cal G}_\eta$ is associated to the kernel $h(x,x') = \psi_\eta(x)^\top \psi_\eta(x)^\top$, that corresponds to
\eqals{\label{eq:kernel-Geta}
h(x,x') &= \psi_\eta(x)^\top \psi_\eta(x)^\top = \scal{ \phi_\eta(x) \otimes \phi_\eta(x)}{\phi_\eta(x') \otimes \phi_\eta(x')}_{\hh_\eta \otimes \hh_\eta} \\
& = (\phi_\eta(x)^\top \phi_\eta(x'))^2 = k_\eta(x,x')^2 = k_{2\eta}(x,x').
}
Then ${\cal G}_\eta$ is still a RKHS associated to a Gaussian kernel, in particular $k_{2\eta}$, so ${\cal G}_\eta \subset W^s_2(\R^d)$ for any $s \geq 0$. In particular, by \cref{eq:Wk-norm-Heta} and the fact that $\X \subset \R^d$, we have
\eqals{
\|f\|_{W^s_2(\X)} \leq \|g\|_{W^s_2(\R^d)} \leq C \|f\|_{{\cal G}_\eta} \tau^{(s-d)/2}, \quad \forall f \in {\cal G}_{\eta},
}
where $C$ depends only on $s, d$.
Now by the interpolation inequality for Sobolev spaces (see, e.g. Thm 5.9, page 139 of \cite{adams2003sobolev}), we have that for any $g \in W^s_2(\X)$ the following holds
\eqals{
\|g\|_{L^\infty(\X)} \leq C' \|g\|^{d/(2s)}_{W^s_2(\X)}\|g\|^{1-d/(2s)}_{L^2(\X)}.
}
Applying the inequality above to the function $g = Q_\la^{-1/2} f$, with $f \in {\cal G}_{\eta}$ and $\|f\|_{{\cal G}_{\eta}} \leq 1$, we have 
\eqal{
\|Q_\la^{-1/2} f\|_{L^\infty(\X)} &\leq C' \|Q_\la^{-1/2} f\|^{d/(2s)}_{W^s_2(\X)}\|Q_\la^{-1/2} f\|^{1-d/(2s)}_{L^2(\X)} \\
& \leq C' C^{d/2s} \tau^{\frac{(s-d)d}{4s}}\|Q_\la^{-1/2} f\|^{d/2s}_{{\cal G}_\eta} \|Q_\la^{-1/2} f\|^{1-d/(2s)}_{L^2(\X)},\\
& \leq C' C^{d/2s} \tau^{\frac{(s-d)d}{4s}}\la^{-d/4s} \|Q_\la^{-1/2} f\|^{1-d/(2s)}_{L^2(\X)}.
}
Finally note that, since by the reproducing property $g(x) = \scal{\psi_\eta(x)}{g}_{{\cal G}_\eta}$ for any $g \in {\cal G}_\eta$, we have that for any $f \in {{\cal G}_\eta}$ such that $\|f\|_{{\cal G}_\eta} \leq 1$, we have
\eqals{
\|Q_\la^{1/2}f\|^2_{L^2(\X)} &= \int_\X (Q_\la^{1/2}f)(x)^2 dx \\
&= \int_\X \scal{Q_\la^{1/2}f}{(\psi_\eta(x)\psi_\eta(x)^\top Q_\la^{1/2}f}_{{\cal G}_\eta} \\
& = \int \scal{f}{Q_\la^{1/2} \left(\int_\X \psi_\eta(x)\psi_\eta(x)^\top dx\right) Q_\la^{1/2}f}_{{\cal G}_\eta} \\
&\leq \|f\|^2 \|Q_{\la}^{-1/2} Q Q_\la^{-1/2}\| \leq 1.
}
Then, to recap
\eqals{
\textrm{ess-sup} \|\zeta\|_{{\cal G}_\eta} &\leq \sup_{\|f\|_{{\cal G}_\eta} \leq 1} \|Q_\la^{-1/2}f\|_{L^\infty(\X)} \leq \sup_{\|f\|_{{\cal G}_\eta} \leq 1} \|Q_\la^{-1/2} f\|^{d/(2s)}_{W^s_2(\X)}\|Q_\la^{-1/2} f\|^{1-d/(2s)}_{L^2(\X)} \\
& \leq C' C^{d/2s} \tau^{\frac{(s-d)d}{4s}}\la^{-d/4s}.
}

\paragraph{Step 3. Bernstein inequality for random vectors in separable Hilbert spaces}
The points $x_1, \dots, x_n$ are independently and identically distributed according to $p$. Define the random variables $\zeta_i = Q_\la^{-1/2} \phi_\eta(x_i) {\bf 1}_\X(x_i)$ for $i = 1,\dots, n$. Note that $\zeta_i$ are independent and identically distributed with the same distribution as $\zeta$. Now note that 
\eqals{
\mathbb{E} \zeta_i = \mathbb{E} \zeta = Q_\la^{-1/2} \int_\X p(x) \psi_\eta(x) dx =Q_\la^{-1/2} v,
}
moreover $\frac{1}{n}\sum_{i=1}^n \zeta_i = Q_\la^{-1/2} \hat{v}$. Then, given the bounds on the variance and on the norm for $\zeta$, by applying a Bernstein inequality for random vectors in separable Hilbert spaces, as, e.g., Thm 3.3.4 of \cite{yurinsky1995sums} (we will use the notation of Prop.~2 of \cite{rudi2016rf}), the following holds with probability $1-\delta$
\eqals{
\|Q_\la^{-1/2}(\hat{v} - v)\|_{{\cal G}_\eta} = \|\frac{1}{n} \sum_{i=1}^n \zeta_i - \mu\|_{{\cal G}_\eta} \leq \frac{M \log \frac{2}{\delta}}{n} ~+~ \sqrt{\frac{S\log\frac{2}{\delta}}{n}},
}
with $M = C' C^{d/2s} \tau^{\frac{(s-d)d}{4s}}\la^{-d/4s}$ and $S = 2 \tr(Q_\la^{-1}Q)$.
\end{proof}

\begin{lemma}\label{lm:eff-dim-gaussian}
Let $\eta = \tau {\bf 1}_d$ with $\tau \geq 1$ and $\la \leq 1/2$. Let $\X \subset [-1,1]^d$. Then
\eqals{
\tr((Q + \la I)^{-1} Q) \leq C \tau^{d/2} \left(\log\tfrac{1}{\la}\right)^{d},
}
\end{lemma}
\begin{proof}
Let $\sigma_j(Q)$ with $j \in \N$ be the sequence of singular values of $Q$ in non-increasing order.  First note that for any $k \in \N$,
\eqals{
\sigma_k(Q) \leq \sum_{j \in \N} \sigma_j(Q) \leq \tr(Q) = \tr\left(\int_\X \psi_\eta(x)\psi_\eta(x)^\top dx\right) = \int_\X \|\psi_\eta(x)\|^2_{{\cal G}_\eta} dx \leq 2^d,
}
since $\|\psi_\eta(x)\|^2_{{\cal G}_\eta} = 1$ and $\X \subseteq [-1,1]^d$. Moreover,
for $k \in \N$, let $\bar{x}_{1,k},\dots, \bar{x}_{k,k}$ be a minimal covering of $[-1,1]^d$. Let $\tilde{P}_k$ be the projection operator whose range is $\textrm{span}\{\bar{x}_{1,k}, \dots, \bar{x}_{k,k}\}$. Note that $\tilde{P}_k$ has rank $k$. By the Eckart-Young theorem, we have that $\sigma_{k+1}(Q) = \inf_{\operatorname{rank}(A) = k} \|Q - A\|^2_{F}$, then
\eqals{
\sigma_{k+1}(Q) &= \inf_{\operatorname{rank}(A) = k} \|Q - A\|^2_{F} \leq \|Q - \tilde{P}_k Q\|_{F} \\
& \leq \Big\|\int_\X (I - \tilde{P}_k)\psi_\eta(x)\psi_\eta(x) dx \Big\|_F \leq \int_\X \|(I - \tilde{P}_k)\psi_\eta(x)\psi_\eta(x)\|_F dx  \\
&  = \int_\X \|(I - \tilde{P}_k)\psi_\eta(x)\|_{{\cal G}_\eta} \|\psi_\eta(x)\|_{{\cal G}_\eta} dx  \leq 2^d \sup_{x \in \X } \|(I - \tilde{P}_k)\psi_\eta(x)\|_{{\cal G}_\eta}.
}
Since $\bar{x}_{1,k},\dots, \bar{x}_{k,k}$ is a minimal covering of $[-1,1]^d$ and since the $\ell_2$-ball of diameter $1$ contains a cube of side $1/\sqrt{d}$, it is possible to cover a cube of side $2$ with $k$ balls of diameter $2h$ (and so of radius $h$) with $h \leq 2\sqrt{d}/(k^{1/d}-1)$, see, e.g., Thm. 5.3, page 76 of \cite{cucker2007learning}. Since ${\cal G}_\eta$ is a reproducing kernel Hilbert space with Gaussian kernel $k_{2\eta}$ as discussed in \cref{eq:kernel-Geta}, by applying \cref{thm:(I-P)phi}, we have that when $h \leq 1/(C'\sqrt{\tau})$ and 
\eqals{
\sup_{x \in \X } \|(I - \tilde{P}_k)\psi_\eta(x)\|_{{\cal G}_\eta} \leq C e^{-\frac{c}{\sqrt{\tau} h} \log \frac{c }{\sqrt{\tau} h}} \leq C e^{-\frac{c(k^{1/d} - 1)}{2 \sqrt{d}\sqrt{\tau} } \log \frac{c (k^{1/d} - 1)}{2 \sqrt{d}\sqrt{\tau} }},
}
with $c,C,C'$ depending only on $d$.
Take $k_\tau \geq (1 + 2 \max(1,e/c,C')\sqrt{d \tau})^d$. When $k \geq k_\tau$, we have $h \leq 1/(C'\sqrt{\tau})$, $c/(h\sqrt{t}) \leq e$, and $k \geq 2^d$, so $(k^{1/d}-1) \geq k^{1/d}/2$, then
\eqals{
\sup_{x \in \X } \|(I - \tilde{P}_k)\psi_\eta(x)\|_{{\cal G}_\eta} \leq C e^{-\frac{c_1 k^{1/d}}{\sqrt{\tau} }}, \qquad \forall k \geq k_\tau.
}
with $c_1 = c/(4\sqrt{d})$. Let $g(x) = Ce^{-\frac{c_1}{\sqrt{\tau}}x^{1/d}}$ for $x \geq 0$. Since $g$ is non-increasing then $g(n+1) \geq \int_{n}^{n+1} g(x) dx$. Since $\int_{k}^\infty e^{-\frac{c_1}{\sqrt{\tau}} x^{1/d}} \leq d (\frac{c_1}{\sqrt{\tau}})^{-d} \Gamma(d, k^{1/d}\frac{c_1}{\sqrt{\tau}})$ where $\Gamma(a, x)$ is the {\em incomplete Gamma function} and is bounded by $\Gamma(a,x) \leq 2 x^a e^{-x}$ for any $x, a \geq 1$ and $x \geq 2a$ (see Lemma P, page 31 of \cite{altschuler2018massively}). The condition to apply the bound on the incomplete Gamma function in our case corresponds to require $k$ to satisfy $k \geq (2d/c_1)^d \tau^{d/2}$.
Then, for any $k \in \N$, we have 
\eqals{
\sum_{t \geq k+2} \sigma_{t}(Q) \leq  2^d 
\begin{cases}
t &  t \leq k'_\tau \\
d k e^{-\frac{c_1}{\sqrt{\tau}}k^{1/d}}  & t > k'_\tau,
\end{cases}
}
where $k'_\tau = \max(k_\tau, 2 + (2d/c_1)^d \tau^{d/2})$. The bound on $\tr((Q+\la I)^{-1} Q)$ is obtained by considering the characterization of the trace of an operator in terms of its singular values and the fact that $\frac{z}{z+\la} \leq \min(1, \frac{z}{\la})$. For any $k \geq k'_\tau$, we have
\eqals{
\tr((Q+\la)^{-1} Q) &= \sum_{t \in \N} \frac{\sigma_t(Q)}{\sigma_t(Q) + \la} = \sum_{t \leq (k+1)} \frac{\sigma_t(Q)}{\sigma_t(Q) + \la} + \sum_{t \geq k+2} \frac{\sigma_t(Q)}{\sigma_t(Q) + \la} \\
& \leq 1 + k ~+~ \tfrac{1}{\la} \sum_{t \geq k+2}\sigma_t(Q) \leq 1 + k +  \tfrac{dk}{\la} 2^d   e^{-\frac{c_1}{\sqrt{\tau}}k^{1/d}}.
}
In particular, choosing $k = (\sqrt{\tau} C_2 \log(1+\frac{1}{\la}))^d$ with $C_2 = 2d/c_1 + \max(1,e/c,C')\sqrt{d}$, then $k \geq k'_\tau$ and so
\eqals{
\tr((Q+\la)^{-1} Q) \leq 1 + (\sqrt{\tau} C_2 \log(1+\tfrac{1}{\la}))^d \leq (\sqrt{\tau} 4C_2 \log(\tfrac{1}{\la}))^d.
}
\end{proof}

\subsection{Proof of \cref{thm:learning}}\label{proof:thm:learning}

We are finally ready to prove \cref{thm:learning}. We restate here the theorem for convenience. 

\TLearning*

\begin{proof}
Let $\eps \in (0,1], h > 0$ and $\eta = \tau {\bf 1}_d$ with $\tau \in [1,\infty)$, to be fixed later.
Denote by $\hat{p}$ the model associated to matrix $\hat{A} \in \psd^m$ that minimizes $\hat{L}_\la$ over the set $\psd^m$, i.e., $\hat{p} = \pp{\cdot}{\hat{A}, \tilde{X},\eta}$. The goal is then to bound $\|\hat{p} - p\|_{L^2(\R^d)}$. First we introduce the probabilities $p_\eps, \tilde{p}_\eps$, useful to perfom the analysis. Let $\mm_\eps \in \psd(\hh_\eta)$ be the operator such that $p_\eps = \pp{\cdot}{\mm_\eps, \phi_\eta}$ approximates $p$ with error $\eps$ as defined in \cref{thm:Meps}, on the functions $\tilde{f}_1,\dots, \tilde{f}_q \in W^\beta_2(\R^d) \cap L^\infty(\R^d)$ that are the extension to $\R^d$ of the functions $f_1,\dots,f_q$ characterizing $p$ via \cref{asm:psd-model}. The details of the extension are in \cref{cor:extension-intersection}. Now, consider the model the compressed model $\tilde{p}_\eps = \pp{\cdot}{\tilde{A}_\eps, \tilde{X}_m, \eta}$ with
\eqals{
\tilde{A}_\eps ~=~ K_{\tilde{X},\tilde{X},\eta}^{-1} \,\tilde{Z} \mm_\eps \tilde{Z}^*\, K_{\tilde{X},\tilde{X},\eta}^{-1},
}
where $\tilde{Z}:\hh_\eta \to \R^m$ is defined in \cref{eq:tildeZ-definition} in terms of $\tilde{X}_m$.

\paragraph{Step 1. Decomposition of the error}\\
By applying \cref{thm:error-dec} with $A = \tilde{A}_\eps$ and \cref{lm:StoL2} to simplify the notation, we derive
\eqals{
\|\hat{p} - p\|_{L^2(\X)} & \leq \|\tilde{p}_\eps - p\|_{L^2(\X)} + \sqrt{\la}\|\tilde{M}_\eps\|_{F} +  5\|Q_\la^{-\frac{1}{2}}(\hat{v} - v)\|_{{\cal G}_\eta},
}
where $Q_\la = Q + \la I$ and $\tilde{M}_\eps = \tilde{Z}^* \tilde{A}_\eps \tilde{Z}_\eps$. Note that
\eqals{
\tilde{M}_\eps = \tilde{Z}^* \tilde{A}_\eps \tilde{Z}_\eps = \tilde{Z}^* K_{\tilde{X},\tilde{X},\eta}^{-1} \tilde{Z} \mm_\eps \tilde{Z}^* K_{\tilde{X},\tilde{X},\eta}^{-1} \tilde{Z}_\eps = \tilde{P} M_\eps \tilde{P},
}
where $\tilde{P}:\hh_\eta \to \hh_\eta$ is defined in \cref{sec:operators-definition} and is the projection operator on the range of $\tilde{Z}^*$, so
\eqals{
\|\tilde{M}_\eps\|_{F} = \|\tilde{P} M_\eps \tilde{P}\|_{F} \leq \|\tilde{P}\|^2 \|M_\eps\|_F \leq \|M_\eps\|_F.
}
Bounding $\|\tilde{p}_\eps - p\|_{L^2(\X)}$ with $\|\tilde{p}_\eps - p\|_{L^2(\X)} \leq \|\tilde{p}_\eps - p_\eps\|_{L^2(\X)} + \|p_\eps - p\|_{L^2(\X)}$, we obtain
\eqal{\label{eq:dec-learning-0}
\|\hat{p} - p\|_{L^2(\X)} & \leq \|\tilde{p}_\eps - p_\eps\|_{L^2(\X)} + \|p_\eps - p\|_{L^2(\X)} + \sqrt{\la}\|M_\eps\|_{F} +  5\|Q_\la^{-\frac{1}{2}}(\hat{v} - v)\|_{{\cal G}_\eta},
}

\paragraph{Step 2. Bounding the terms of the decomposition}

Let $h$ be the fill distance (defined in \cref{eq:fill-distance}) associated to the points $\tilde{x}_1,\dots, \tilde{x}_m$. By \cref{thm:compr-fill-dist}, there exist three constants $c, C, C'$ depending only on $d$ such that, when $h \leq \sigma/C'$, with $\sigma = \min(1, 1/\sqrt{\tau})$, then for any $x \in \X$
\eqals{
|\tilde{p}_\eps(x) - p_\eps(x)| ~\leq~ 2C\sqrt{\|\mm_\eps\| p_\eps(x)}\, e^{-\frac{c\,\sigma}{h}\log\frac{c\,\sigma}{h}}  \,+\, C^2 \|\mm_\eps\|\,e^{-\frac{2c\,\sigma}{h}\log\frac{c\,\sigma}{h}}.
}
Since $q_\eta = 1$, $p_\eps(x) = \phi_\eta(x)^\top M_\eps \phi_\eta(x) \leq \|M_\eps\| \|\phi_\eta(x)\|^2_{\hh_\eta} = \|M_\eps\|$ for any $x \in \R^d$, since $\|\phi_\eta(x)\|^2_{\hh_\eta} = \phi_\eta(x)^\top\phi_\eta(x) = k_\eta(x,x) = 1$, then 
\eqals{
\|\tilde{p}_\eps - p_\eps\|_{L^2(\X)} \leq \operatorname{vol}(\X) (2C + C^2) \|\mm_\eps\|e^{-\frac{c\,\sigma}{h}\log\frac{c\,\sigma}{h}}
}
where $\vol(\X)$ is the volume of $\X$ and is $\vol(\X) = 2^d$. By \cref{thm:Meps} we know also that
$\|p_\eps - p\|_{L^2(\X)} \leq \|p_\eps - p\|_{L^2(\R^d)} \leq \eps$. Moreover, we also know that there exists two constants $C_1, C_2$ depending on $\X, \beta, d$ and the norms of $\tilde{f}_1,\dots, \tilde{f}_q$ (and so on the norms of $f_1,\dots, f_q$ via \cref{cor:extension-intersection}) such that
\eqals{
\|\mm_\eps\| \leq \|\mm_\eps\|_F \leq \tr(\mm_\eps) \leq C_1 \tau^{d/2}(1 + \eps^2\exp(\tfrac{C_2}{\tau} \eps^{-\frac{2}{\beta}})) \leq 2C_1 \tau^{d/2} \exp(\tfrac{C_2}{\tau} \eps^{-\frac{2}{\beta}})
}
By bounding $\|Q_\la^{-\frac{1}{2}}(\hat{v} - v)\|_{{\cal G}_\eta}$ via \cref{lm:Qla-v}, with $s = d$ and \cref{lm:eff-dim-gaussian}, then
\cref{eq:dec-learning-0} is bounded by 
\eqal{
\|\hat{p} - p\|_{L^2(\X)} & \leq \eps + C_3\tau^{d/2} e^{\frac{C_2}{\tau} \eps^{-\frac{2}{\beta}}}\left(\sqrt{\la} + e^{-\frac{c\,\sigma}{h}\log\frac{c\,\sigma}{h}}\right)  \\
& \qquad +  \frac{C_4 \log \frac{2}{\delta}}{n \la^{1/4}} ~+~ \left(\frac{C_5\tau^{d/2} (\log\frac{1}{\la})^{d} \log\frac{2}{\delta}}{n}\right)^{1/2}
}
$C_3 = 2 C_1\operatorname{vol}(\X) (2C + C^2)$,$C_4 = 5 C_4'$ where $C_4'$ is from \cref{lm:Qla-v} and depends only on $d$, while $C_5 = 50 C_5'$, where $C_5'$ is from \cref{lm:eff-dim-gaussian} and depends only on $d$. 
Setting $\eps = n^{-\frac{\beta}{2\beta + d}}$, $\tau = \eps^{-2/\beta}$ and $\la = \eps^{2 + 2d/\beta} = n^{-(2\beta+2d)/(2\beta + d)}$, since $1/(n \la^{1/4}) = \eps$ and $\eps^{d/\beta}/n = \eps^2$, then
\eqal{
\|\hat{p} - p\|_{L^2(\X)} & \leq (1 + C_3e^{C_2})\eps + C_3e^{C_2}\eps^{d/\beta}e^{-\frac{c\,\sigma}{h}\log\frac{c\,\sigma}{h}}\\
& \qquad + C_4 \log\tfrac{2}{\delta}\eps + C_5^{1/2} (\tfrac{2\beta+2d}{2\beta+d})^{d/2}\, \eps (\log n)^{d/2} (\log \tfrac{2}{\delta})^{1/2}.
}

\paragraph{Step 3. Controlling the number $m$ in terms of $h$}

The final result is obtained by controlling the number of points $m$ such that $h \leq \frac{1}{C'\sqrt{\tau}}$ (in order to be able to apply \cref{thm:compr-fill-dist}) and such that $h \leq  \frac{c}{\sqrt{\tau}}\max(e, (1+\frac{d}{\beta}) \log\frac{1}{\eps})^{-1}$, so $\log\frac{c\,\sigma}{h} \geq 1$ and 
\eqals{
\eps^{d/\beta}e^{-\frac{c\,\sigma}{h}\log\frac{c\,\sigma}{h}} \leq \eps^{d/\beta}e^{-\frac{c\,\sigma}{h}} \leq \eps.
}
By, e.g. Lemma~12, page 19 of \cite{vacher2021dimension} and the fact that $[-1,1]^d$ is a convex set, we have that there exists two constants $C_6, C_7$ depending only on $d$ such that $h \leq C_6 m^{-1/d} (\log(C_7 m / \delta))^{1/d}$, with probability at least $1-\delta$. In particular, we want to find $m$ that satisfy
\eqals{
C_6 C_8^d \tau^{d/2} (\log 1/\eps)^d \, m \geq \log \tfrac{C_7 m}{\delta},
}
with $C_8 = \max(\frac{1}{c},C', e, 1+ d/\beta)$. With the same reasoning as in \cref{eq:mlogm}, we see that any $m$ satisfying $m \geq 2B \log(2A B)$ with $A = C_6 C_8^d \tau^{d/2} (\log 1/\eps)^d$ and $B = C_7/\delta$ suffices to guarantee the inequality above. In particular, since $\eps \leq n$ and $\eps^{d/\beta} = n^\frac{d}{2\beta+d} \leq n^d$ we choose
\eqals{
m = C_9 n^\frac{d}{2\beta+d} (\log n)^{d} ~ d \log(C_{10}^{1/d} n (\log n)).
}
with $C_9 = C_6 C_8^d$, $C_{10} = 2 C_9 C_7$. With this choice, we have
\eqals{
\|\hat{p} - p\|_{L^2(\X)} & \leq (1 + 2C_3e^{C_2})\eps + C_4\eps \log\tfrac{2}{\delta} + C_5^{\frac{1}{2}}(\tfrac{2\beta+2d}{2\beta+d})^{\frac{d}{2}}\, \eps (\log n)^{\frac{d}{2}} (\log \tfrac{2}{\delta})^{\frac{1}{2}}\\
& \leq (1 + 2C_3e^{C_2} + C_4 + C_5^{\frac{1}{2}} (\tfrac{2\beta+2d}{2\beta+d})^{\frac{d}{2}}) \, \eps \, (\log n)^{\frac{d}{2}} (\log \tfrac{2}{\delta})\\
& = C_{11} ~ n^{-\frac{\beta}{2\beta + d}} ~ (\log n)^{\frac{d}{2}} (\log \tfrac{2}{\delta}).
}
with $C_{11} = 1 + 2C_3e^{C_2} + C_4 + C_5^{\frac{1}{2}}  (\tfrac{2\beta+2d}{2\beta+d})^{\frac{d}{2}}$.
\end{proof}

\section{Operations}\label{app:operations}

We report here the derivation of the operations discussed in \cref{sec:operations} for PSD models. For simplicity in the following, given a vector $x\in\R^d$ and a positive vector $\eta\in\R_{++}^n$, with some abuse of notation, in the following we will denote $\eta\nor{x}^2$ for $x^\top\diag(\eta) x$ when clear from the context.

\subsection{Properties of the Gaussian Function}

We recall the following classical properties of Gaussian functions, which are key to derive the results in the following. For any two points $x_1,x_2\in\R^d$ and $\eta_1,\eta_2\in\R^{d}_{++}$, with the notation introduced above, let $c_\eta = \pi^{d/2} \det(\diag(\eta))^{-1/2}$ and $x_3 = \frac{\eta_1 x_1 + \eta_2 x_2}{\eta_1 +\eta_2}$ and $\eta_3 = \frac{\eta_1\eta_2}{\eta_1 + \eta_2}$. We have
\eqal{\label{eq:properties-gaussian-function}
    k_{\eta_1}(x,x_1) k_{\eta_2}(x,x_2)  ~=~ k_{\eta_1 + \eta_2}(x,x_3)k_{\eta_3}(x_1,x_2) ~~\quad\textrm{and}~~\quad \int k_\eta(x,x_1) dx  ~=~ c_\eta.
}
Additionally, we recall that the joint Gaussian kernel corresponds to the product kernels in the two variables, namely $k_{(\eta_1,\eta_2)}((x,y),(x',y')) = k_{\eta_1}(x,x') k_{\eta_2}(y,y')$.

We begin by recalling how the equalities in \cref{eq:properties-gaussian-function} can be derived. First, we recall that for any positive definite matrix $A\in\mathbb{S}_{++}^n$ (namely an invertible positive semi-definite matrix), the integral of the Gaussian function $e^{-\scal{x}{Ax}}$ is
\eqal{
    \int_{\R^d} e^{-\eta\nor{x}^2}~dx = \pi^{d/2} \det{A}^{-1/2},
}
which yields the required equality in \cref{eq:properties-gaussian-function} for $A = \diag(\eta)$ and $\eta\in\R_{++}^d$. 

For the second property in \cref{eq:properties-gaussian-function}, let $x_1,x_2\in\R^d$ and $\eta_1,\eta_2\in\R_{++}^d$. For any $x\in\R^d$ we have
\eqal{
    k_{\eta_1}(x,x_1)k_{\eta_2}(x,x_2) & = e^{-\eta_1\nor{x-x_1}^2 -\eta_2\nor{x-x_2}^2}.
}
By expanding the argument in the exponent we have
\eqal{
    \eta_1\nor{x-x_1}^2 & +\eta_2\nor{x-x_2}^2 \\
    & = (\eta_1+\eta_2)\nor{x}^2 - 2 \scal{x}{\eta_1x_2+\eta_2x_2} + \eta_1\nor{x_1}^2+\eta_2\nor{x_2}^2\\
    & = (\eta_1+\eta_2)\Big(\nor{x}^2 - 2 \scal{x}{\tfrac{\eta_1 x_1 + \eta_2 x_2}{\eta_1+\eta_2}}\pm\nor{\tfrac{\eta_1 x_1 + \eta_2 x_2}{\eta_1+\eta_2}}^2\Big)\\
    & \quad+ \eta_1\nor{x_1}^2+\eta_2\nor{x_2}^2\\
    & = (\eta_1+\eta_2)\nor{x - x_3}^2 +\eta_1\nor{x_1}^2 + \eta_2\nor{x_2}^2  -(\eta_1+\eta_2)\nor{x_3}^2,
}
where $x_3 = \tfrac{\eta_1 x_1 + \eta_2 x_2}{\eta_1+\eta_2}$. Now, 
\eqal{
\eta_1\nor{x_1}^2 + \eta_2\nor{x_2}^2 & -(\eta_1+\eta_2)\nor{x_3}^2 
\\& = \eta_1\nor{x_1}^2 + \eta_2\nor{x_2}^2  -\frac{\nor{\eta_1 x_1 + \eta_2 x_2}^2}{\eta_1+\eta_2}\\
& = \eta_1\nor{x_1}^2 + \eta_2\nor{x_2}^2  -\frac{\eta_1^2\nor{x_1}^2 +2\eta_1\eta_2\scal{x_1}{x_2} + \eta_2^2\nor{x_2}^2 }{\eta_1+\eta_2}\\
& = \frac{\eta_1\eta_2\nor{x_1}^2 - 2\eta_1\eta_2\scal{x_1}{x_2} + \eta_1\eta_2\nor{x_2}^2}{\eta_1+\eta_2}\\
& = \frac{\eta_1\eta_2}{\eta_1+\eta_2}\nor{x_1-x_2}.
}
We conclude that, for $\eta_3 = \tfrac{\eta_1\eta_2}{\eta_1+\eta_2}$, we have
\eqal{
    k_{\eta_1}(x,x_1)k_{\eta_2}(x,x_2) & = e^{-\eta_1\nor{x-x_1}^2 -\eta_2\nor{x-x_2}^2}\\
    & = e^{-(\eta_1+\eta_2)\nor{x-x_3}^2 - \frac{\eta_1\eta_2}{\eta_1+\eta_2}\nor{x_1-x_2}^2 }\\
    & = k_{(\eta_1+\eta_2)}(x,x_3)k_{\eta_3}(x_1,x_2),
}
as required.

Finally, we recall that for any vectors $x,x'\in\R^{d_1}$ and $y,y'\in\R^{d_2}$ and positive weights $\eta_1\in\R_{++}^{d_1},\eta_2\in\R_{++}$, we will make use of the fact that 
\eqal{
    k_{(\eta_1,\eta_2)}((x,y),(x',y')) & = e^{-(x-x',y-y')^\top \diag(\eta_1,\eta_2)(x-x',y-y')}\\
    & = e^{-(x-x')^\top \diag(\eta_1) x - y^\top \diag(\eta_2) y'}\\
    & = k_{\eta_1}(x,x')k_{\eta_2}(y,y')
}

\subsection{Evaluation}
We recall that, given a PSD model $\pp{x}{A,X,\eta}$ evaluating it in a point $x_0$ writes as
\eqal{
    \pp{x=x_0}{A,X,\eta} = K_{X,x_0,\eta}^\top A K_{X,x_0,\eta}.
}
Given a PSD model of the form $\pp{x,y}{A,[X,Y],(\eta_1,\eta_2)}$ we denote partial evaluation in a vector $y_0\in\Y$ as
\eqal{
    \pp{x,y=y_0}{A,[X,Y],(\eta_1,\eta_2)} & = \sum_{i,j=1}^n A_{i,j}k_{\eta_1}(x_i,x)k_{\eta_1}(x_j,x)k_{\eta_2}(y_i,y_0)k_{\eta_2}(y_j,y_0)\\
    & = \sum_{i,j=1}^n \bigg[k_{\eta_2}(y_i,y_0)A_{ij}k_{\eta_2}(y_j,y_0)\bigg]k_{\eta_1}(x_i,x)k_{\eta_1}(x_j,x)\\
    & = \pp{x}{A\circ(K_{Y,y_0,\eta_2}K_{Y,y_0,\eta_2}^\top), X,\eta_1}
}

\subsection{Integration and Marginalization}

We begin by showing the result characterizing the marginalization of a PSD model with respect to a number of random variables.

\PMarginalization*

\begin{proof}
The result is obtained as follows \eqal{
    \int \pp{x,y}{A,[X,Y],(\eta,\eta')} ~dy & = \sum_{i,j=1}^n A_{ij}k_{\eta'}(y_i,y)k_{\eta'}(y_j,y)\int k_{\eta}(x_i,y)k_{\eta}(x_j,y)~dx\\
    & = \sum_{i,j=1}^n A_{ij}k_{\frac{\eta}{2}}(x_i,x_j)k_{\eta'}(y_i,y)k_{\eta'}(y_j,y)\int k_{2\eta}(\tfrac{x_i+x_j}{2},x)~dx\\
    & = c_{2\eta} \sum_{i,j=1}^n \Big[A_{ij} k_{\frac{\eta_2}{2}}(x_i,x_j)\Big]k_{\eta'}(y_i,y)k_{\eta'}(t_j,y)\\
    & = \pp{y}{c_{2\eta}(A \circ K_{X,X,\frac{\eta}{2}}), Y, \eta' }
}
\end{proof}

\paragraph{Integration}
Analogously, we can write in matrix form the full integral of a PSD model (with respect to all its variables): let $\pp{x}{A,X,\eta}$, we have
\eqal{
    \int \pp{x}{A,X,\eta}~dx & = \sum_{i,j=1}^n A_{ij}\int k_\eta(x_i,x)k_\eta(x_j,x)~dx\\
    & = \sum_{i,j=1}^n A_{ij} k_{\frac{\eta}{2}}(x_i,x_j)\int k_{2\eta}(\tfrac{x_i+x_j}{2},x)~dx\\
    & = c_{2\eta}\sum_{i,j=1}A_{ij}k_{\frac{\eta}{2}}(x_i,x_j)\\
    & = c_{2\eta}\tr(A K_{X,X,\frac{\eta}{2}}),
}
which yields \cref{eq:integration}.

\paragraph{Integration on the Hypercube}
In \cref{rem:integration-on-hypercube} we commented upon restricting integration and marginalization on the hypercube $H = \prod_{t=1}^d [a_t,b_t]$. Both operations can be performed by slightly changing the integrals above. In particular, let $G\in\R^{n \times n}$ be the matrix with $i,j$-th entry equal to
\eqal{
    G_{ij} = c_{2\eta} \prod_{t=1}^d\operatorname{erf}\left(\sqrt{2\eta_t}\bigg(b_t - \frac{x_{i,t}+x_{j,t}}{2}\bigg)\right )-\operatorname{erf}\left(\sqrt{2\eta_t}\bigg(a_t - \frac{x_{i,t}+x_{j,t}}{2}\bigg)\right)
}
Then, integration becomes
\eqal{
    \int_H \pp{X;A,X,\eta} = \tr((A\circ K_{X,X,\frac{\eta}{2}}) G)
}
and marginalization
\eqal{
    \int \pp{x,y}{A,[X,Y],(\eta,\eta')} ~dy & = \pp{y}{c_{2\eta}A\circ K_{X,X,\frac{\eta}{2}} \circ G. }
}
This result is a corollary of \cref{prop:mean-variance-characteristic} that we prove below. 
\PMeanVariance*

\begin{proof}
Following the same proof for \cref{prop:marginalization}, we have
\eqal{
    \int g(x)\pp{x}{A,[X,Y],(\eta)} ~dy & = \sum_{i,j=1}^n A_{ij}\int g(x)k_{\eta}(x_i,y)k_{\eta}(x_j,y)~dx\\
    & = \sum_{i,j=1}^n A_{ij}k_{\frac{\eta}{2}}(x_i,x_j)\int g(x) k_{2\eta}(\tfrac{x_i+x_j}{2},x)~dx\\
    & = \sum_{i,j=1}^n \Big[A_{ij} k_{\frac{\eta_2}{2}}(x_i,x_j)\Big] c_{g,2\eta}\big(\frac{x_i+x_j}{2}\big)\\
    & = \tr((A\circ K_{X,X,\eta/2}) G).
}

\end{proof}

\subsection{Multiplication}

\PMultiplication*

\begin{proof}
We begin by explicitly writing the product between the two PSD models
\eqal{
& \pp{x,y}{A,[X,Y],(\eta_1,\eta_2)}
     \pp{y,z}{B,[Y',Z],(\eta_2',\eta_3)} \\
    & \quad~~ = \Big(\sum_{i,j=1}^{n} A_{ij} k_{\eta_1}(x_i,x)k_{\eta_1}(x_j,x)k_{\eta_2}(y_i,x)k_{\eta_2}(y_j,y)\Big) ~ \times\\
    &\quad~~\qquad \Big(\sum_{\ell,h=1}^m B_{\ell h} k_{\eta_2'}(y_\ell',y)k_{\eta_2'}(y_h',y)k_{\eta_3}(z_\ell,z)k_{\eta_3}(z_h,z)\Big)\\
    & \quad~~= \sum_{i,j,\ell,h=1}^{n,m} \big[A_{ij} B_{\ell h} k_{\tilde\eta_2}(y_i,y_\ell')k_{\tilde\eta_2}(y_j,y_h')\big]k_{\eta_2+\eta_2'}(\tfrac{\eta_2y_i+\eta_2'y_\ell'}{\eta_2+\eta_2'},y)k_{\eta_2+\eta_2'}(\tfrac{\eta_2y_j+\eta_2'y_h'}{\eta_2+\eta_2'},y)\\
    & \quad~~\qquad\qquad\quad k_{\eta_1}(x_i,x)k_{\eta_1}(x_j,x)k_{\eta_3}(z_\ell,z)k_{\eta_3}(z_h,z),
}
where we have coupled together the pairs $(i,\ell)$ and $(j,h)$ using the product rule between Gaussian functions. Let
\eqal{
C = (A\otimes B)\circ(\vec(K_{Y',Y,\tilde\eta_2})\vec(K_{Y',Y,\tilde\eta_2})^\top),
}
and denote by  $\mathfrak{i}:\N\times\N\to\N$ now the indexing function such that $\mathfrak{i}(i,\ell) = (i-1)m + \ell$. It follows that the term 
\eqal{
A_{ij}B_{\ell h}k_{\tilde\eta_2}(y_i,y_\ell')k_{\tilde\eta_2}(y_j,y_h') = C_{\mathfrak{i}(i,\ell)\mathfrak{i}(j,h)},
}
corresponds to the $(\mathfrak{i}(i,\ell)\mathfrak{i}(j,h))$-th entry of the matrix $C$. Analogously, let: 
\begin{itemize}
    \item $\tilde y_{\mathfrak{i}(i,\ell)} = \frac{\eta_2 y_i + \eta_2' y_\ell'}{\eta_2+\eta_2'}$ is the $\mathfrak{i}(i,\ell)$-th row of the matrix $\widetilde{Y}= (Y\frac{\eta_2}{\eta_2+\eta_2'}) \otimes \mathbf{1}_m +  \mathbf{1}_n\otimes (Y'\frac{\eta_2'}{\eta_2+\eta_2'})$, which is the $nm \times d_2$ matrix whose rows correspond to all possible pairs from $Y$ and $Y'$ respectively. 
     
    \item $\tilde x_{\mathfrak{i}(i,\ell)} = x_i$ is the $\mathfrak{i}(i,\ell)$-th row of the matrix $\widetilde{X} = X\otimes\mathbf{1}_m$, namely the $nm\times d_1$ matrix containing $m$ copies of each row of $X$. 
    
    \item $\tilde z_{\mathfrak{i}(i,\ell)} = z_\ell$ is the $\mathfrak{i}(i,\ell)$-th row of the matrix $\widetilde{Z}= \mathbf{1}_n\otimes Z$, namely the $nm\times d_3$ matrix containing $n$ copies of $Z$. 
\end{itemize}
Then we have that 
\eqal{
& \sum_{i,j,\ell,h=1}^{n,m}  \big[A_{ij} B_{\ell h} k_{\tilde\eta_2}(y_i,y_\ell') k_{\tilde\eta_2}(y_j,y_h')\big] k_{\eta_2+\eta_2'}(\tfrac{\eta_2y_i+\eta_2'y_\ell'}{\eta_2+\eta_2'},y)k_{\eta_2+\eta_2'}(\tfrac{\eta_2y_j+\eta_2'y_h'}{\eta_2+\eta_2'},y)\\
    & \quad\qquad\qquad\quad k_{\eta_1}(x_i,x)k_{\eta_1}(x_j,x)k_{\eta_3}(z_\ell,z)k_{\eta_3}(z_h,z)\\
    & \quad= \sum_{i,j,\ell,h=1}^{n,m} C_{\mathfrak{i}(i,\ell)\mathfrak{i}(j,h)} k_{\eta_2+\eta_2'}(\tilde y_{\mathfrak{i}(i,\ell)},y)k_{\eta_2+\eta_2'}(\tilde y_{\mathfrak{i}(j,h)},y)k_{\eta_1}(\tilde x_{\mathfrak{i}(i,\ell)},x)k_{\eta_1}(\tilde x_{\mathfrak{i}(j,h)},x)\\
    & \quad\qquad\qquad\qquad k_{\eta_3}(\tilde z_{\mathfrak{i}(i,\ell)},x)k_{\eta_3}(\tilde z_{\mathfrak{i}(j,h)},x)\\
    & \quad= \sum_{s,t}^{nm} C_{st} k_{\eta_2+\eta_2'}(\tilde y_{s},y)k_{\eta_2+\eta_2'}(\tilde y_{t},y)k_{\eta_1}(\tilde x_{s},x)k_{\eta_1}(\tilde x_{t},x)k_{\eta_3}(\tilde z_{s},x)k_{\eta_3}(\tilde z_{t},x)\\
    & \quad= \pp{x,y,z}{C,[\widetilde{X},\widetilde{Y},\widetilde{Z}],(\eta_1,\eta_2+\eta_2',\eta_3)},
}
as desired. 
\end{proof}

\subsection{Reduction}

The reduction operation leverages the structure of the base matrix $X \otimes \mathbf{1}_m$ to simplify the PSD model. To this end, fenote by $\widetilde X = X\otimes \mathbf{1}_m$ and consider again the indexing function $\mathfrak{i}(i,\ell) = (i-1)m + \ell$. Then (see also the proof of \cref{prop:multiplication} we have that the $\mathfrak{i}(i,\ell)$-th row of $\widetilde X$ is $\tilde x_{\mathfrak{i}(i,\ell)} = x_i$, the $i$-th row of $X$. Therefore we have
\eqal{
    \pp{x}{A,\widetilde{X},\eta} & = \sum_{s,t=1}^{nm} A_{st} k_\eta(\tilde x_s,x)k_\eta(\tilde x_t,x)\\
    & = \sum_{i,j,\ell,h=1}^{n,m} A_{\mathfrak{i}(i,\ell)\mathfrak{i}(j,h)} k_\eta(\tilde x_{\mathfrak{i}(i,\ell)},x)k_\eta(\tilde x_{\mathfrak{i}(j,h)},x)\\
    & = \sum_{i,j,\ell,h=1}^{n,m} A_{\mathfrak{i}(i,\ell)\mathfrak{i}(j,h)} k_\eta(x_i,x)k_\eta(x_j,x)\\
    & = \sum_{i,j=1}^n \Big[\sum_{\ell,h=1}^m A_{\mathfrak{i}(i,\ell)\mathfrak{i}(j,h)}\Big] k_\eta(x_i,x)k_\eta(x_j,x)\\
    & = \sum_{i,j=1}^n B_{ij} k_\eta(x_i,x)k_\eta(x_j,x),
}
where $B$ is a $n\times n$ PSD matrix and each of its entries is the sum of the entries of $A$ corresponding to the repeated rows in $X \otimes \mathbf{1}_m$. Therefore $B = (I_m \otimes \mathbf{1}_n^\top) A (I_m \otimes \mathbf{1}_n)$ as required. 

We give here the explicit form for the Markov transition in \cref{cor:markov}

\CMarkovTransition*

\begin{proof}
We first multiply the two PSD model to obtain, by \cref{prop:multiplication}
\eqal{
\pp{x,y}{A,[X,Y],(\eta_1,\eta_2)}\pp{y}{B,Y',\eta_2'} = \pp{x,y}{A_1,[\widetilde{X}, \widetilde{Y}],(\eta,\eta_2 + \eta_2')},
}
with 
\eqal{
    A_1 = (A \otimes B) \circ (\vec(K_{Y,Y',\tilde{\eta_2}})\vec(K_{Y,Y',\tilde{\eta_2}})^\top) \qquad \tilde\eta_2 = \frac{\eta_2\eta_2'}{\eta_2+\eta_2},
}
and $\tilde X = X \otimes\mathbf{1}_m$ and $\widetilde{Y} = (Y\frac{\eta_2}{\eta_2+\eta_2'}) \otimes \mathbf{1}_m +  \mathbf{1}_n\otimes (Y'\frac{\eta_2'}{\eta_2+\eta_2'})$. Then, we proceed with marginalization. By \cref{prop:marginalization}, we have
\eqal{
    \int \pp{x,y}{A_1,[\widetilde{X}, \widetilde{Y}],(\eta,\eta_2 + \eta_2')}~dy = \pp{x}{A_2,\widetilde{X},\eta},
}
where $A_2 = c_{2(\eta_2+\eta_2')}A_1 \circ K_{\widetilde Y,\widetilde Y, (\eta_2+\eta_2')/2}$. Finally, since $\widetilde X = X \otimes \mathbf{1}_m$, by reduction \cref{eq:reduction}, we have
\eqal{
    \pp{x}{A_2,\widetilde{X},\eta} = \pp{x}{C,X,\eta}
}
with $C = (I_n \otimes \mathbf{1}_m^\top) A_2 (I_n \otimes \mathbf{1}_m)$, which concludes the derivation. 
\end{proof}

\subsection{Hidden Markov Models}

We conclude this section by providing the derivation of the HMM inference in \cref{sec:hmm}. 

\PHMM*

\begin{proof}
Let $\pp{x_t}{A_{t-1},X_{t-1},\eta_{t-1}}$ be the estimate $\hat p(x_{t-1}|y_{1:t-1})$ obtained at the previous step, with $A_{t-1}\in\psd^{n_{t-1}}$ and $X_{t-1}\in\R^{n_{t-1}}$ We then proceed by performing the operations in \cref{eq:hmm-approximate}. 

{\itshape Observation $\hat\omega_t(x)$}. A new observation $y_t$ is received. By \cref{eq:partial-evaluation} we have
\eqal{
    \hat\omega_t(x) = \hat\omega(y=y_t,x) = \pp{y=y_t,x_+}{C,[Y,X'],(\eta_{obs},\eta')} = \pp{x}{C_t,X',\eta'},
}
with 
\eqal{
C_t = C \circ (K_{Y,y_t,\eta_obs}K_{Y,y_t,\eta_obs}^\top).
}

{\itshape Product $\hat\beta_t(x_+,x
) = \hat\tau(x_+,x)\hat p(x|y_{1:t-1})$.} We perform the product between the transition function and the previous state estimation
\eqal{
    \hat\beta(x_+,x) & = \pp{x_+,x}{B,[X_+,X],(\eta_+,\eta)}\pp{x}{A_{t-1},X_{t-1},\eta_{t-1}}\\
    & = \pp{x_+,x}{B_t,[X_+\otimes \mathbf{1}_{n_{t-1}}, \widetilde{X}_t], (\eta_+,\eta + \eta_t)},
}
with 
\eqal{
    B_t = (B \otimes A_{t-1})\circ(\vec(K_{X_t,X,\tilde\eta_t})\vec(K_{X_t,X,\tilde\eta_t})^\top) \qquad \tilde\eta_t = \frac{\eta \eta_t}{\eta +\eta_t},
}
and
\eqal{
\widetilde{X}_t \tfrac{\eta}{\eta+\eta_t} X\otimes \mathbf{1}_{n_{t-1}} + \tfrac{\eta_t}{\eta+\eta_t}\mathbf{1}_n\otimes X_t.
}

{\itshape Marginalization (+ Reduction) $\hat\beta_t(x_+) = \int \hat\beta_t(x_++,x)~dx$.} We perform marginalization by \cref{prop:marginalization} to obtain
\eqal{
    \hat\beta_t(x_+) & = \int \hat\beta_t(x_+,x)~dx\\
    & = \int \pp{x_+,x}{B_t,[X_+\otimes \mathbf{1}_{n_{t-1}}, \widetilde{X}_t], (\eta_+,\eta + \eta_t)}~dx\\
    & = \pp{x_+}{D_t',X_+\otimes \mathbf{1}_{n_{t-1}}, \eta_+}
}
with
\eqal{
    D_t' = c_{2(\eta + \eta_t)} B_t \circ K_{\widetilde X_t, \widetilde X_t, \tilde\eta_t/2}.
}
Since the PSD model has a redundant base point matrix, we can apply reduction from \cref{eq:reduction}, to obtain
\eqal{
    \hat\beta_t(x_+) & = \pp{x_+}{D_t',X_+\otimes \mathbf{1}_{n_{t-1}}, \eta_+}\\
    & = \pp{x_+}{D_t,X_+, \eta_+},
}
where
\eqal{
    D_t = (I_n \otimes \mathbf{1}_{n_{t-1}}^\top) D_t' (I_n \otimes \mathbf{1}_{n_{t-1}}).
}

{\itshape Multiplication $\hat\pi_t(x_+) = \hat\omega_t(x_+)\hat\beta_t(x_+)$.} We now multiply the observation term with the state estimation to obtain
\eqal{
    \hat\pi_t(x_+) & = \hat\omega_t(x_+)\hat\beta_t(x_+)\\
    & = \pp{x_+}{C_t,X',\eta'}\pp{x_+}{D_t,X_+, \eta_+}\\
    & = \pp{x_+}{E_t,\widetilde{X}, \eta'+\eta_+},
}
with 
\eqal{
    E_t = (C_t \otimes D_t) \circ (\vec(K_{X,X',\tilde\eta'})\vec(K_{X,X',\tilde\eta'})^\top) \qquad \tilde\eta' = \frac{\eta'\eta_+}{\eta'+\eta_+}
}
and $\tilde X = (X'\tfrac{\eta'}{\eta'+\eta_+})\otimes \mathbf{1}_n + \mathbf{1}_m\otimes (X_+ \tfrac{\eta_+}{\eta'+\eta_+})$. 

{\itshape Normalization $\hat p(x_t|y_{1:t}) = \hat\pi_t(x_+)/\int \hat\pi_t(x_+)~dx_+$.} We finally integrate $\hat\pi_t(x_+)$ in order to normalize it. By \cref{eq:integration} we have
\eqal{
    c_t = \int \hat\pi_t(x_+)~dx_+ & = \int \pp{x_+}{E_t,\widetilde{X}, \eta'+\eta_+}~dx_+ = c_{2(\eta'+\eta_+)} \tr(E_t K_{\tilde X,\tilde X,(\eta'+\eta_+)/2}),
}
and therefore
\eqal{
    \hat p(x_t|y_{1:t}) = \frac{\hat\pi_t(x_+)}{\int \hat\pi_t(x_+)~dx_+} = \pp{x_t}{A_t,\widetilde{X},\eta'+\eta_+},
}
with 
\eqal{
    A_t = E_t/c_t.
}
This concludes the proof showing that, at every step, $\hat p(x_t|y_{1:t})$ has always same base point matrix $\tilde X$ and parameters $\eta'+\eta_+$. Note that the proof above also recovers explicitly the steps in \cref{alg:psd-hmm}.
\end{proof}

\end{document}